\documentclass[11pt]{article}

\usepackage[letterpaper,margin=1in]{geometry}
\usepackage[numbers,sort&compress]{natbib}

\usepackage[utf8]{inputenc}
\usepackage[T1]{fontenc}
\usepackage{hyperref}
\usepackage{url}
\usepackage{booktabs}
\usepackage{multirow}
\usepackage{amsmath,amsfonts,amssymb,amsthm}
\usepackage{nicefrac}
\usepackage{microtype}
\usepackage{graphicx}
\usepackage{subcaption}
\usepackage{algorithm}
\usepackage{algpseudocode}
\usepackage{enumitem}
\usepackage{xcolor}
\usepackage{thm-restate}
\usepackage{xspace}
\usepackage{pifont}

\newcommand{\R}{\mathbb{R}}

\newcommand{\zono}[1]{\langle #1 \rangle}

\newcommand{\MZ}{\mathcal{M}}

\newcommand{\methodname}{MZAttn}

\DeclareMathOperator*{\softmax}{softmax}

\newtheorem{theorem}{Theorem}
\newtheorem{corollary}[theorem]{Corollary}
\newtheorem{proposition}[theorem]{Proposition}
\newtheorem{lemma}[theorem]{Lemma}
\newtheorem{definition}{Definition}
\newtheorem{remark}{Remark}

\title{Matrix Zonotopic Attention: A Context-Adaptive Value Projection for Set Transformers}

\author{%
  Zhen Zhang \qquad Amr Alanwar \\
  School of Computation, Information and Technology \\
  Technical University of Munich, Germany \\
  \texttt{\{zhenzhang.zhang, alanwar\}@tum.de}
}
\date{}

\begin{document}

\maketitle

% ══════════════════════════════════════════════════════════════════
\begin{abstract}
% ══════════════════════════════════════════════════════════════════

Multi-head attention combines an input-dependent softmax routing with an input-independent linear value projection, so the per-sample operator mapping aggregated values to outputs is the same for every input set. We study the consequences of this asymmetry for permutation-invariant set targets. We introduce the Transformation Degrees of Freedom (TDOF) of a target operator, a complexity measure counting the input-dependent directions an exact representation requires, and present a depth-separation analysis showing that context-rigid attention needs depth proportional to the target's TDOF, whereas a single layer with a context-adaptive value family can represent the same target. Building on this analysis, we propose Matrix Zonotopic Attention (MZAttn), which replaces the fixed value projection with a context-adaptive matrix-zonotope family: a centre matrix plus a sum of generator matrices weighted by input-dependent gates. The construction reduces to standard multi-head attention at initialisation, preserves permutation equivariance, and admits a data-driven reachability interpretation. Experiments on a range of set-prediction tasks are consistent with the TDOF prediction that the architectural advantage is selective: it appears on targets that depend on the input set in a high-rank, sparsely combinatorial way, and is small on aggregate-statistic targets where parameter-matched standard attention is already competitive.

\end{abstract}

% ══════════════════════════════════════════════════════════════════
\section{Introduction}
\label{sec:intro}
% ══════════════════════════════════════════════════════════════════

Attention-based set architectures span point-cloud analysis~\citep{lee2019set,qi2017pointnet}, molecular modelling~\citep{satorras2021egnn}, in-context learning~\citep{kim2019anp}, and recent linear-attention variants~\citep{katharopoulos2020transformers,sun2023retentive,yang2024gated}. A basic question has gone largely unexamined: how much of a standard multi-head attention~\citep{vaswani2017attention} layer actually adapts to its input? It combines a softmax routing whose weights depend on the input set with a linear value projection that does not (Figure~\ref{fig:money}, left). On four textbook computational-geometry targets (minimum enclosing ball, convex hull volume, minimum-spanning-tree weight, rotation-dependent matching), every parameter-matched attention baseline using the standard PMA pooling fails to meaningfully improve over the mean predictor as per-point dimension grows, while \methodname{} maintains a clean separation (Figure~\ref{fig:money}, right). The failure is structural rather than capacity-driven: even pushing a HyperNet update to full key-dimension rank at several times \methodname{}'s parameter budget does not close the gap (Figure~\ref{fig:money}, right; full sweep in Appendix~\ref{app:hypernet-rank-ext}).

\begin{figure}[h]
 \centering
 \includegraphics[width=\textwidth]{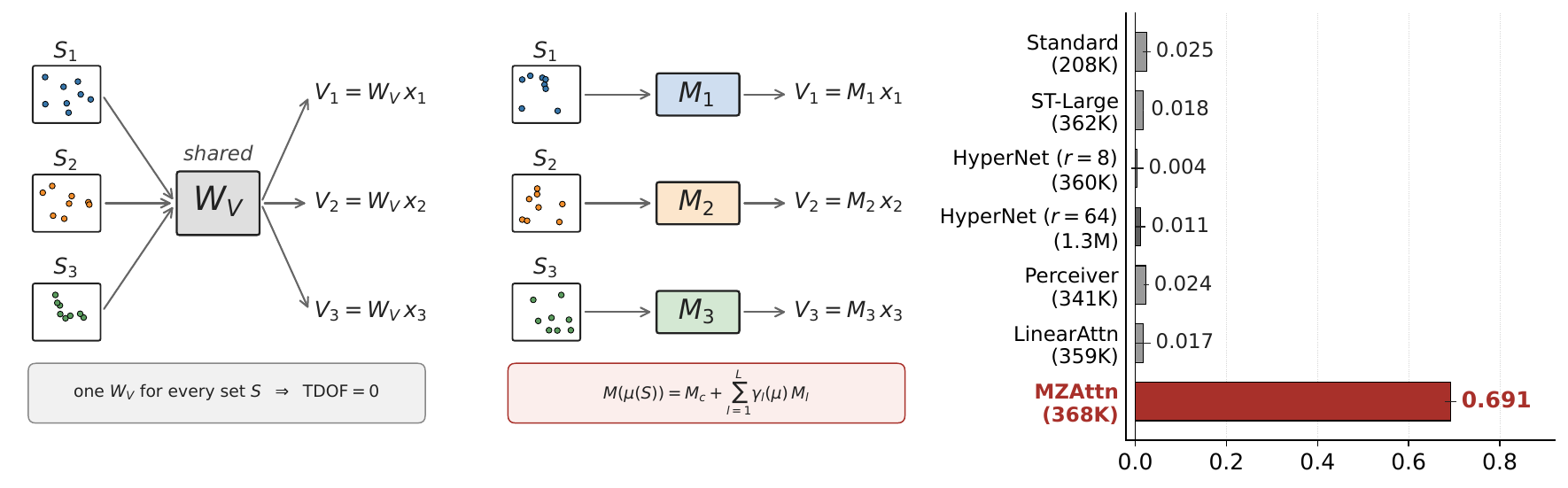}
 \caption{\textbf{Standard attention uses a fixed value projection; \methodname{} replaces it with a context-adaptive matrix-zonotope family.} \textbf{Left:} standard attention applies the same $W_V$ to every input set, making the per-sample operator independent of $S$. \textbf{Middle:} \methodname{} replaces $W_V$ with $M(\mu(S)) = M_c + \sum_l \gamma_l(\mu) M_l$, a context-gated matrix zonotope whose realisation varies per sample. \textbf{Right:} on MEB radius ($d{=}8$, $5$ seeds), all six parameter-matched PMA-pool baselines---including HyperNet at $r{=}d_k{=}64$ ($1.3$M params, $3.6\times$ \methodname{}'s budget)---stay at $R^2 \leq 0.025$ while \methodname{} reaches $R^2 = 0.69$: capacity is not the bottleneck.}
 \label{fig:money}
\end{figure}

Tasks requiring a per-sample linear operator (a sample-specific rotation, a change-of-basis map, or a sparse indicator selecting hull boundary points) lie strictly outside the function family a stack of such layers can represent. We formalise this relational capacity bottleneck through the transformation degrees of freedom (TDOF) of the target operator and prove a matching depth-separation: a worst-case task of TDOF~$k$ requires standard-attention depth proportional to $k$ (with a constant determined by FFN width), whereas a single layer of our proposed architecture suffices (Theorem~\ref{thm:tdof-depth}). The non-linear targets above only approximately inhabit the linearised class; their inheritance of the same bottleneck is an empirical prediction the theory motivates, validated by the synthetic-to-real transfer study below.

We propose Matrix Zonotopic Attention (\methodname{}), which replaces the fixed value projection with a learnable matrix zonotope (centre plus a few generator matrices, gated by set-level statistics; Section~\ref{sec:method}). The parameterisation is minimax-optimal in the parameter-versus-TDOF sense (Theorem~\ref{thm:minimax}) and yields single-pass uncertainty estimates. The advantage is selective and predicted by TDOF: on low-TDOF benchmarks (point-cloud classification, molecular property prediction) \methodname{} is within $1$--$3$ pp of the best parameter-matched baseline; positive transfer requires high per-sample rank, sparse combinatorial targets, and absence of imposed symmetry, conditions verified on CIFAR-100 ResNet-feature point clouds with a $k$-center cost target.

% ══════════════════════════════════════════════════════════════════
\section{Method: Matrix Zonotopic Attention}
\label{sec:method}
% ══════════════════════════════════════════════════════════════════

We first introduce the necessary background on zonotopes, matrix zonotopes, and the Set Transformer, then describe the three components of \methodname{}: (1) zonotope token representation, (2) the MZ-attention mechanism, and (3) the full MZ-Set Transformer architecture.

\subsection{Preliminaries}

\begin{definition}[Zonotope and Matrix Zonotope~\citep{Kuhn1998, Althoff2010PhD, alanwar2021data}]
\label{def:zonotope}
A zonotope with center $c \in \R^n$ and generator columns $G_1, \dots, G_{n_g} \in \R^n$ (the columns of the generator matrix $G \in \R^{n \times n_g}$) is the set
$\mathcal{Z} = \{ c + \sum_{g=1}^{n_g} \alpha_g G_g : \alpha_g \in [-1, 1] \}$.
A matrix zonotope with center $M_c \in \R^{m \times n}$ and matrix generators $M_1, \dots, M_L \in \R^{m \times n}$ is the set of matrices
\begin{equation}
 \MZ = \bigg\{ M_c + \sum_{l=1}^{L} \xi_l M_l ~\bigg|~ \xi_l \in [-1, 1] \bigg\}.
 \label{eq:mz}
\end{equation}
The interval hull of a zonotope is the tightest axis-aligned bounding box, with half-widths $\delta = \sum_{g=1}^{n_g} |G_g|$ (absolute value taken componentwise).
\end{definition}

The product of a matrix zonotope with a zonotope is again a zonotope~\citep{alanwar2021data}: if $\MZ$ has $L$ generators and $\mathcal{Z}$ has $h$ generators, then $\MZ \mathcal{Z} \triangleq \{Mz : M \in \MZ, z \in \mathcal{Z}\}$ has center $M_c c$ and at most $h + L + Lh$ generators.
The $O(Lh)$ growth motivates our fixed-budget management scheme (Section~\ref{sec:architecture}).

\paragraph{Set Transformer.}
The Set Transformer~\citep{lee2019set} processes unordered sets through MAB / ISAB / PMA blocks built from standard scaled dot-product attention; \methodname{} replaces the fixed value projection $v_j \mapsto W_V v_j$ in each block.

\subsection{Zonotope Token Representation}

Given an input element $x_i \in \R^{d_{\text{in}}}$, we project it into a zonotope $(c_i, G_i)$:
\begin{equation}
 c_i = \text{LayerNorm}(\text{GELU}(W_c x_i + b_c)), \quad
 G_i = \frac{1}{\sqrt{n_g d}} (W_g x_i + b_g) \in \R^{n_g \times d},
 \label{eq:zono-proj}
\end{equation}
where $W_c \in \R^{d \times d_{\text{in}}}$, $W_g \in \R^{(n_g d) \times d_{\text{in}}}$, and $n_g$ is the fixed generator budget.
The scaling factor $1/\sqrt{n_g d}$ ensures that generators start as small perturbations, so the initial representation is close to a standard vector embedding.

\subsection{\methodname{} Attention: From Scalar Scores to Relational Operators}
\label{sec:mz-attn}

Given query, key, and value zonotopes with centers $c_i \in \R^{d_k}$ and generator matrices $G_i \in \R^{n_g \times d_k}$ (the $g$-th row $G_i[g,:]$ is the $g$-th generator vector), standard attention scores $\alpha_{ij} = \softmax(q_i^\top k_j/\sqrt{d_k})$ are computed on centers, values are aggregated $(\hat c_i, \hat G_i) = \sum_j \alpha_{ij}(v_j^c, v_j^G)$, and the head-$h$ matrix zonotope $\MZ^{(h)} = M_c^{(h)} + \sum_{l=1}^L \xi_l M_l^{(h)}$ transforms the output:
\begin{equation}
 o_i^c = M_c^{(h)} \hat{c}_i \;\in\; \R^{d_k},
 \qquad
 o_i^G = \hat{G}_i \,(M_c^{(h)})^\top \;+\; W_{\mathrm{mix}}\,\mathrm{stack}_{l=1}^{L}\!\bigl[\gamma_l(\mu)\, M_l^{(h)} \hat{c}_i\bigr] \;\in\; \R^{n_g \times d_k},
 \label{eq:mz-main}
\end{equation}
where $\hat{G}_i (M_c^{(h)})^\top$ applies $M_c^{(h)}$ row-wise to each generator (preserving the $n_g \times d_k$ shape), $\mathrm{stack}_l[\cdot]$ assembles the $L$ vectors $\gamma_l\, M_l^{(h)} \hat c_i \in \R^{d_k}$ into an $L \times d_k$ matrix (rows indexed by $l$), $W_{\mathrm{mix}} \in \R^{n_g \times L}$ is a learned mixing matrix that preserves the fixed generator budget across layers, and $\gamma = \tanh(W_\gamma \bar c^{kv}) \in [-1,1]^L$ is the context-dependent gate computed from the set-level mean $\bar c^{kv} = \frac{1}{n_{kv}}\sum_j c_j^{kv}$.
$M_c^{(h)} = I_{d_k}$ and $M_l^{(h)} \approx 0$ at initialisation, so the block reduces to standard attention. Full step-by-step derivation, generator-budget control, and richer permutation-invariant summaries (attention-pooled statistics, higher-order moments) are in Appendix~\ref{app:impl}.

\subsection{MZ-Set Transformer Architecture}
\label{sec:architecture}

We replace each attention block in a Set Transformer~\citep{lee2019set} with its MZ counterpart (MZ-MAB, MZ-SAB, MZ-ISAB, and MZ-PMA), yielding
$\mathrm{InputProj} \to [\mathrm{MZ\text{-}ISAB}]^{L_{\mathrm{enc}}} \to \mathrm{MZ\text{-}PMA}_k \to [\mathrm{MZ\text{-}SAB}]^{L_{\mathrm{dec}}} \to \mathrm{OutputHead}$.
The output head concatenates the pooled center $c \in \R^{kd}$ and the interval-hull half-width $\delta = \sum_g |G_g|$ before a two-layer MLP, giving a deterministic prediction and a single-pass uncertainty estimate for free. The expressiveness Theorems~\ref{thm:tdof-depth} and~\ref{thm:separation} below additionally invoke a variant readout that reads the signed pooled generators $\sum_g G_g$ rather than the IHW $\sum_g |G_g|$ (signed access is straightforward to add as a parallel head; the IHW preserves the uncertainty-bound interpretation of Theorem~\ref{thm:ihw-calibration}, while the signed readout is what the expressiveness proofs exploit). Both readouts share the same MZ-attention layer; only the post-pooling reduction differs.

% ══════════════════════════════════════════════════════════════════
\section{Theoretical Analysis}
\label{sec:theory}
% ══════════════════════════════════════════════════════════════════

We present three groups of results: an expressiveness bottleneck with minimax-optimal resolution (\S\ref{sec:expressiveness}), TDOF-based depth separation (\S\ref{sec:tdof}), and uncertainty and generalization guarantees (\S\ref{sec:mixing-theory}).
Full proofs are in Appendix~\ref{app:proofs}; additional results in Appendix~\ref{app:additional-theory}.

\subsection{Expressiveness: Context-Adaptive vs.\ Context-Rigid Attention}
\label{sec:expressiveness}

\begin{definition}[Context-rigid and context-adaptive relational operators]
\label{def:relop}
Consider an attention layer operating on set $S = \{x_1, \dots, x_n\} \subset \R^d$, and let $\hat{v}_i = \sum_j \alpha_{ij} v_j$ denote the aggregated value for query~$i$. The relational operator $\mathcal{T}$ maps $\hat{v}_i$ into the per-head output. Standard attention is context-rigid: $\mathcal{T}(\hat{v}) = W\hat{v}$ for a fixed $W \in \R^{d_k \times d_k}$, independent of $S$. \methodname{} is context-adaptive: $\mathcal{T}(\hat{v}; S) = M(\mu(S))\hat{v}$ where $M(\mu) = M_c + \sum_{l=1}^L \gamma_l(\mu) M_l$ depends on a permutation-invariant statistic $\mu(S)$, and the gates $\gamma_l(\mu) \in [-1, 1]$ are the trained instantiations of the abstract zonotope coefficients $\xi_l$ of Eq.~\eqref{eq:mz}.
\end{definition}

This distinction has a fundamental information-theoretic consequence. Standard attention routes information through a scalar bottleneck: the score $\alpha_{ij}$ controls how much of each value to aggregate, but the subsequent transformation $W$ is frozen. We formalise this via the relational complexity / TDOF of the target function, which counts the intrinsic number of context-dependent directions an exact representation requires.

\begin{definition}[Relational complexity / TDOF]
\label{def:rel-complexity}
For a permutation-invariant target $f(S) = T(S)\,\mu(S)$ with matrix-valued $T{:}\;\mathcal{X}^{\leq N} \to \R^{m \times d}$ that factors through $\mu(S)$, the relational complexity (henceforth TDOF) is
$\mathrm{TDOF}(f) = \min\{L : T(S) = M_0 + \sum_{l=1}^L \sigma_l(\mu(S))\, M_l,\; \sigma_l \text{ 1-Lipschitz},\; |\sigma_l| \leq 1\}$.
The model-free linear dimension $d_T = \dim(\mathrm{span}\{T(S) - T(S') : S, S' \in \mathcal{X}^{\leq N}\})$ used in Appendix~\ref{app:additional-theory} (Definition~\ref{def:tdof}) lower-bounds this constructive count; for $f \in \mathcal{F}_L$ with bounded 1-Lipschitz coordinate functionals (e.g.\ all natural set-statistical operators in Proposition~\ref{prop:tdof-natural}) the two coincide. Functions with $\mathrm{TDOF}(f) = 0$ require no input-dependent transformation; those with $\mathrm{TDOF}(f) > 0$ fundamentally require the transformation itself to adapt.
\end{definition}

We write $\mathcal{F}_L = \{S \mapsto (M_0 + \sum_{l=1}^L \sigma_l(\mu(S))\, M_l)\,\mu(S) : \|M_l\|_F \le B,\, |\sigma_l| \le 1\}$ for the linearised operator family with TDOF at most $L$. $\mathcal{F}_L$ is the natural linearisation class for any smooth context-adaptive operator: if $T(\mu)$ has Jacobian $J$ with $\mathrm{rank}(J) = r$ at $\mu_0$, then $L = r$ MZ generators capture its first-order behaviour exactly via the SVD of $J$ (Appendix~\ref{app:additional-theory}, Proposition~\ref{prop:taylor-fl}). The non-linear targets we evaluate (MST, MEB, hull) do not sit exactly in $\mathcal{F}_L$, so TDOF acts as a structural hypothesis validated empirically.

\begin{theorem}[Relational capacity bottleneck]
\label{thm:bottleneck}
(a) Standard attention without FFN has zero relational capacity: $T_{\mathrm{std}} = \sum_h W_o^{(h)} W_V^{(h)}$ is constant, independent of $S$.
(b) \methodname{}-attention represents any function in $\mathcal{F}_L$ with $L$ generators using $(L{+}1)d_k^2 + L d_k$ parameters per head, and this $\Theta(Ld_k^2)$ count is minimax-optimal.
\end{theorem}

\begin{proof}[Proof sketch]
For (a), constant sets $S_t$ make softmax uniform, so the output is a fixed linear map. The matching $\Omega(L d_k^2 \log(BR/\epsilon))$ lower bound, the FFN-augmented Yarotsky-style bound on attention with FFN, the Eckart--Young residual bound, and a strict separation construction are in Appendix~\ref{app:additional-theory}.
\end{proof}

Part~(a) is the no-FFN zero-capacity statement (the cleanest qualitative form of the bottleneck); with FFN, zero capacity is replaced by the explicit width/depth tradeoff bound used in part~(b) and Theorem~\ref{thm:tdof-depth}.
Thus \methodname{}-attention is minimally parameterised for $\mathcal{F}_L$: $\Theta(Ld_k^2)$ parameters versus $O(d_k^4)$ for a general hypernetwork. Natural set operators have explicitly computable TDOF: Mahalanobis distance and whitening have $\mathrm{TDOF} = d(d{+}1)/2$, top-$k$ PCA projection has $\mathrm{TDOF} = kd - k(k{+}1)/2$, and per-coordinate $z$-score has $\mathrm{TDOF} = d$ (Appendix~\ref{app:additional-theory}, Proposition~\ref{prop:tdof-natural}); soft Chamfer matching has $\mathrm{TDOF} \geq d$ (Proposition~\ref{prop:chamfer-tdof}), accounting for the \methodname{} gap on Chamfer set-matching (Task~B).

\subsection{Transformation Degrees of Freedom and Depth Separation}
\label{sec:tdof}

The transformation degrees of freedom $\mathrm{TDOF}(f) = \dim(\mathrm{span}\{T(S) - T(S')\})$ is intrinsic to $f$ (see Definition~\ref{def:tdof}). It yields a depth lower bound:

\begin{theorem}[Depth separation via TDOF, worst-case over $\mathcal{F}_k$]
\label{thm:tdof-depth}
Fix the input distribution to be sets $S$ whose mean $\mu(S) \in \R^{d}$ has i.i.d.\ standard-Gaussian coordinates. There exists $f \in \mathcal{F}_k$ on this input distribution (the explicit witness of Lemma~\ref{lem:comp-orth} with $M_0 = 0$) for which any standard attention network of depth $D$ with FFN width $d_{\mathrm{ff}}$ that represents $f$ exactly requires $D \geq k\,d_k / d_{\mathrm{ff}}$, while a single \methodname{} layer with $L \geq k$ generators followed by the signed-generator readout (Section~\ref{sec:architecture}) represents the same $f$ exactly. For approximate representation in $L^2(\mu)$ with $L^2$-residual error of order $O(\mathrm{Var}(f)/L)$, the same depth lower bound holds up to constants.
\end{theorem}

A compositional-orthogonality argument over the Gaussian input distribution (Appendix~\ref{app:additional-theory}, Lemma~\ref{lem:comp-orth}) realises the bound for an explicit Frobenius-orthonormal witness ($M_l = e_l e_d^\top$, $M_0 = 0$) even with FFN and residual connections; extending the lower bound to a generic high-TDOF target family with non-trivial cross-layer matrix products is open.

\paragraph{Empirical verification.}
We verify Theorem~\ref{thm:tdof-depth} on synthetic quadratic targets
\begin{equation}
\label{eq:quadratic-tdof}
f(S) = \mathbf{w}^\top\!\bigl(M_0 + \sum_{l=1}^L \mu(S)_l M_l\bigr)\,\mu(S)
\end{equation}
with Frobenius-orthogonal $M_l$ giving $\mathrm{TDOF}(f) = L$. Across $L \in \{4, 8, 16, 24, 32\}$ the four-layer Set Transformer's test $R^2$ degrades from $0.97$ to $0.40$ while a single \methodname{} layer holds $R^2 \geq 0.90$, with a $6$--$15\times$ residual-MSE gap (Table~\ref{tab:tdof-verify}): adding depth alone cannot cure the bottleneck. EGNN~\citep{satorras2021egnn} hard-codes rotation equivariance and collapses to $R^2 \approx 0$ for every $L$ because the target depends on a fixed coordinate frame --- the complementary failure mode of symmetry-enforcing architectures.

The same picture holds at the architecture level: a 1-layer \methodname{} (244K params) outperforms standard Set Transformers up to depth 6 ($\geq 480$K params; Appendix Table~\ref{tab:depth}), confirming depth separation on real architectures, not only the synthetic-task sweep.

\begin{table}[t]
 \centering
 \footnotesize
 \setlength{\tabcolsep}{4pt}
 \renewcommand{\arraystretch}{0.95}
 \caption{Depth-separation verification on the quadratic-TDOF target (Eq.~\ref{eq:quadratic-tdof}; mean$\pm$std, 3 seeds). 1-layer \methodname{} holds $R^2 \geq 0.90$ across $L$ while 4-layer standard ST degrades.}
 \label{tab:tdof-verify}
 \begin{tabular}{@{}cccc@{}}
 \toprule
 $L$ & \methodname{}-1L $R^2$ $\uparrow$ & ST-4L $R^2$ $\uparrow$ & MSE ratio $\uparrow$ \\
 \midrule
 $4$  & $0.995 \pm 0.002$ & $0.970 \pm 0.005$ & $6.6\times$ \\
 $8$  & $0.988 \pm 0.001$ & $0.809 \pm 0.011$ & $15.4\times$ \\
 $16$ & $0.950 \pm 0.001$ & $0.618 \pm 0.017$ & $7.6\times$ \\
 $24$ & $0.944 \pm 0.011$ & $0.547 \pm 0.020$ & $8.0\times$ \\
 $32$ & $0.908 \pm 0.016$ & $0.396 \pm 0.001$ & $6.5\times$ \\
 \bottomrule
 \end{tabular}
\end{table}

\subsection{Uncertainty, Generalization, and Structural Properties}
\label{sec:mixing-theory}

The zonotope representation yields directional perturbation certificates strictly tighter than scalar Lipschitz bounds, with the interval-hull width (IHW) bounding output diameter and variance and thereby serving as a single-pass uncertainty estimate. Appendix~\ref{app:uncertainty-generalization} collects the precise statements: a norm-controlled generalisation bound of $\tilde{O}(B_{\mathrm{MZ}}/\gamma\sqrt n)$ in the spirit of \citet{bartlett2017spectrally}, universal approximation, and permutation equivariance. At initialisation, \methodname{} reduces exactly to standard multi-head attention.

Empirically (Appendix~\ref{app:ece-nll}, Table~\ref{tab:ece-nll}), the IHW estimate matches a $5$-model Deep Ensemble on ECE ($0.316$ vs.\ $0.337$) and is $\approx 2\times$ better on NLL ($5.02$ vs.\ $9.17$), at $1\times$ inference cost vs.\ $5\times$; MC-Dropout achieves the lowest ECE/NLL but at $30\times$ inference cost.

% ══════════════════════════════════════════════════════════════════
\section{Relationship to Prior Work}
\label{sec:related}
% ══════════════════════════════════════════════════════════════════

Zonotopes have served as external verification tools~\citep{bonaert2021fast, mirman2018differentiable, jordan2022zonotope, chung2025provably} and as set-valued reachability primitives in data-driven control~\citep{alanwar2021data, Alanwar2023Datadriven, kochdumper2023constrained}; \methodname{} folds a matrix zonotope into a trainable value-projection family inside an attention layer. The closest architectural relatives are context-adaptive linear maps spanning a capacity--parameters spectrum --- FiLM~\citep{perez2018film}, discrete MoE~\citep{shazeer2017outrageously}, low-rank Hypernetworks~\citep{ha2017hypernetworks, hyper2025attention}, dynamic filters~\citep{jia2016dynamic} --- among which \methodname{}'s matrix-zonotope parameterisation occupies the minimax-optimal middle at $O(L d_k^2)$ parameters per head (Theorem~\ref{thm:bottleneck}); we compare against them as parameter-matched baselines. Linear-attention and selective state-space variants~\citep{sun2023retentive, yang2024gated, qiu2025gated, gu2024mamba} adopt different value-projection inductive biases (input-independent or, for Mamba, input-dependent only along diagonal channels); a Mamba (S6) encoder with PMA pooling matches \methodname{}-Full at $d{=}8$ on the MEB target but degrades to $R^2{=}0.142{\pm}0.134$ at $d{=}32$ vs.\ \methodname{}'s stable $0.378{\pm}0.030$ (Table~\ref{tab:meb}, Mamba + PMA row), indicating that input-dependent diagonal selectivity alone is dimension-conditional rather than a structural escape from Theorem~\ref{thm:bottleneck}; \methodname{} is orthogonal to and composable with them. Set-input networks~\citep{zaheer2017deep, lee2019set, qi2017pointnet} represent relationships as scalar scores, which we generalise to matrix-valued context-adaptive operators. Full discussion in Appendix~\ref{app:related-extended}.

% ══════════════════════════════════════════════════════════════════
\section{Experiments}
\label{sec:experiments}
% ══════════════════════════════════════════════════════════════════

We evaluate \methodname{} on eight task families, comparing against seven baselines: Deep Sets~\citep{zaheer2017deep}, a standard Set Transformer~\citep{lee2019set}, ST + FiLM~\citep{perez2018film}, ST-Large (parameter-matched wider variant), HyperNet (low-rank input-dependent value-projection update), Perceiver~\citep{jaegle2021perceiver}, and Slot Attention~\citep{locatello2020object}. ST-Large, Perceiver, and HyperNet are matched to \methodname{}'s parameter count (${\sim}$360--368K) to control for capacity. Unless noted, experiments are repeated over multiple seeds with mean $\pm$ std reported.

\subsection{Experimental Setup}
All attention-based architectures share a base configuration of $d{=}64$, $H{=}4$ heads, $d_{\text{ff}}{=}128$, two ISAB encoder layers, dropout $0.1$, and identical training (AdamW~\citep{loshchilov2017decoupled} with learning rate $10^{-4}$, weight decay $10^{-5}$, cosine schedule, early stopping at patience $20$, all implemented in PyTorch~\citep{paszke2019pytorch}); ST-Large, Perceiver, and HyperNet are configured to match \methodname{}'s parameter count of $\sim$365K, and \methodname{} uses $n_g{=}8$ zonotope generators and $L{=}4$ matrix-zonotope generators. Deep Sets, Slot Attention, the FiLM-conditioned Set Transformer, and the per-baseline configurations of HyperNet rank, Perceiver latent count, and Slot iterations are described in Appendix~\ref{app:impl}. We evaluate on eight tasks spanning a range of relational complexity (full descriptions in Appendix~\ref{app:task-descriptions}): set regression on $\R^{16}$ (A), Chamfer-distance set matching in $\R^3$ (B), ModelNet40~\citep{wu2015modelnet} (C) and ScanObjectNN~\citep{uy2019scanobjectnn} (E) point-cloud classification, few-shot set retrieval~\citep{snell2017prototypical} (D), rotation-dependent matching in $\R^8$ (F), minimum-spanning-tree weight in $\R^8$ (G), and QM9 molecular property prediction~\citep{ramakrishnan2014quantum} (H). All experiments use 200 epochs with early stopping; results are averaged over 3--5 seeds.

\subsection{Results}

\begin{table}[t]
 \centering
 \footnotesize
 \setlength{\tabcolsep}{3pt}
 \caption{Main results across seven synthetic and point-cloud benchmarks (Tasks A--G). QM9 (Task H) and Slot Attention appear separately in Tables~\ref{tab:qm9} and~\ref{tab:slot-baseline}. Best in \textbf{bold} among the four parameter-matched attention models (ST-Large, Perceiver, HyperNet, \methodname{}, all at ${\sim}$360--368K); mean $\pm$ std over 3--5 seeds. Smaller-budget baselines (ST + FiLM at $225$K, Standard Set Transformer at $208$K) are reported for reference and are not subject to the bold rule. Tasks C and E are point-cloud classification; Task G is MST weight prediction in $\R^8$.}
 \label{tab:main-results}
 \resizebox{\textwidth}{!}{%
 \begin{tabular}{@{}lcccccccccccccc@{}}
 \toprule
 & \multicolumn{2}{c}{\textbf{Task A: Set Regr.}} & \multicolumn{2}{c}{\textbf{Task B: Set Match.}} & \multicolumn{2}{c}{\textbf{Task F: Rot.\ Match.}} & \multicolumn{2}{c}{\textbf{Task G: MST Wt.}} & \textbf{Task C} & \textbf{Task E} & \multicolumn{2}{c}{\textbf{Task D: Few-Shot}} & \\
 \cmidrule(lr){2-3} \cmidrule(lr){4-5} \cmidrule(lr){6-7} \cmidrule(lr){8-9} \cmidrule(lr){10-10} \cmidrule(lr){11-11} \cmidrule(lr){12-13}
 Model & MSE $\downarrow$ & $R^2$ $\uparrow$ & MSE $\downarrow$ & $R^2$ $\uparrow$ & MSE $\downarrow$ & $R^2$ $\uparrow$ & MSE $\downarrow$ & $R^2$ $\uparrow$ & MN40 $\uparrow$ & Scan $\uparrow$ & MSE ($\times 10^{-2}$) $\downarrow$ & $R^2$ $\uparrow$ & Params \\
 \midrule
 Deep Sets & -- & -- & 490.2{\tiny$\pm$33.4} & 0.652{\tiny$\pm$0.024} & -- & -- & -- & -- & -- & 40.6{\tiny$\pm$1.0}\% & \textbf{2.05}{\tiny$\pm$0.38} & \textbf{0.722}{\tiny$\pm$0.052} & 36K \\
 Set Transformer & 3.43{\tiny$\pm$0.09} & 0.811{\tiny$\pm$0.005} & 138.0{\tiny$\pm$10.3} & 0.901{\tiny$\pm$0.007} & 487.5{\tiny$\pm$31.1} & 0.904{\tiny$\pm$0.006} & 209.8{\tiny$\pm$1.8} & 0.712{\tiny$\pm$0.002} & 77.5{\tiny$\pm$0.6}\% & 71.6{\tiny$\pm$2.2}\% & 2.60{\tiny$\pm$0.24} & 0.646{\tiny$\pm$0.032} & 208K \\
 ST + FiLM & 3.77{\tiny$\pm$0.26} & 0.792{\tiny$\pm$0.014} & 48.0{\tiny$\pm$7.9} & 0.966{\tiny$\pm$0.006} & 490.5{\tiny$\pm$17.8} & 0.904{\tiny$\pm$0.003} & 232.7{\tiny$\pm$12.5} & 0.681{\tiny$\pm$0.017} & 78.4{\tiny$\pm$0.8}\% & 68.9{\tiny$\pm$3.2}\% & 2.77{\tiny$\pm$0.21} & 0.624{\tiny$\pm$0.029} & 225K \\
 \cmidrule{1-14}
 ST-Large & 3.36{\tiny$\pm$0.10} & 0.815{\tiny$\pm$0.005} & 106.5{\tiny$\pm$5.3} & 0.924{\tiny$\pm$0.004} & 423.2{\tiny$\pm$51.8} & 0.917{\tiny$\pm$0.010} & 168.7{\tiny$\pm$23.0} & 0.769{\tiny$\pm$0.032} & \textbf{79.9}{\tiny$\pm$1.0}\% & \textbf{72.9}{\tiny$\pm$2.4}\% & -- & -- & 362K \\
 Perceiver & 3.47{\tiny$\pm$0.05} & 0.809{\tiny$\pm$0.003} & 136.2{\tiny$\pm$17.5} & 0.902{\tiny$\pm$0.013} & 872.1{\tiny$\pm$163.8} & 0.829{\tiny$\pm$0.032} & 249.0{\tiny$\pm$4.7} & 0.659{\tiny$\pm$0.006} & 75.0{\tiny$\pm$0.3}\% & 67.3{\tiny$\pm$1.2}\% & -- & -- & 341K \\
 HyperNet & 3.51{\tiny$\pm$0.01} & 0.806{\tiny$\pm$0.001} & 76.1{\tiny$\pm$5.7} & 0.945{\tiny$\pm$0.004} & 1373.0{\tiny$\pm$141.6} & 0.730{\tiny$\pm$0.028} & 291.2{\tiny$\pm$9.2} & 0.601{\tiny$\pm$0.013} & 75.5{\tiny$\pm$1.7}\% & 66.1{\tiny$\pm$0.9}\% & -- & -- & 360K \\
 \midrule
 \methodname{} (ours) & \textbf{2.92}{\tiny$\pm$0.09} & \textbf{0.839}{\tiny$\pm$0.005} & \textbf{58.4}{\tiny$\pm$11.9} & \textbf{0.958}{\tiny$\pm$0.009} & \textbf{241.2}{\tiny$\pm$3.0} & \textbf{0.953}{\tiny$\pm$0.001} & \textbf{99.4}{\tiny$\pm$1.9} & \textbf{0.864}{\tiny$\pm$0.003} & 78.1{\tiny$\pm$0.7}\% & 69.4{\tiny$\pm$3.7}\% & 2.10{\tiny$\pm$0.16} & 0.715{\tiny$\pm$0.022} & 365K \\
 \bottomrule
 \end{tabular}}
\end{table}

\textbf{High-TDOF relational tasks (A, B, F, G).} \methodname{} achieves the best performance on all four: $13\%$ MSE reduction vs.\ ST-Large on set regression (A), $45\%$ on set matching (B), $\mathbf{51\%}$ over FiLM on rotation matching (F), and $\mathbf{41\%}$ over ST-Large on MST weight (G). On (B), FiLM matches the low-dimensional ($\R^3$) regime (MSE $48.0$), but HyperNet's rank-$8$ update is insufficient ($76.1$) and higher rank degrades further (Appendix~\ref{app:hypernet-rank}). On (F), FiLM performs no better than standard attention, consistent with diagonal modulation's inability to express off-diagonal rotations. On (G), \methodname{}'s seed std is $\pm 1.9$ vs.\ $\pm 23.0$ for ST-Large, indicating stable representation of the pairwise-distance geometry. \textbf{Low-TDOF tasks (C, D, E).} On point-cloud classification (ModelNet40 / ScanObjectNN), ST-Large achieves the best accuracy ($79.9\%$ / $72.9\%$) with \methodname{} within $1$--$3$ pp ($78.1\%$ / $69.4\%$); these tasks depend primarily on per-element features, so Theorem~\ref{thm:bottleneck} predicts no advantage for context-adaptivity. On few-shot retrieval (D), Deep Sets ($R^2 = 0.722$) and \methodname{} ($R^2 = 0.715$) tie and both exceed Standard attention ($0.646$) and FiLM ($0.624$) by $+0.07$--$0.10$ $R^2$, indicating set-aggregation pooling dominates over per-sample operator structure at this scale. On QM9 (Table~\ref{tab:qm9}), \methodname{} matches the best parameter-matched attention baseline within $0.05$ $R^2$ on every target without leading on any individual one, the predicted low-TDOF behaviour. Slot Attention lags all attention-based models on every task tested (Table~\ref{tab:slot-baseline}); all models maintain permutation invariance to within $10^{-5}$ across $10$ random permutations.

\subsection{Computational Geometry and Real-World Transfer}
The MEB radius depends on at most $d{+}1$ extremal points: a sparse combinatorial selector that context-rigid attention cannot express. Table~\ref{tab:meb} reports $R^2$ across dimensions for parameter-matched baselines.

\begin{table}[t]
 \centering
 \footnotesize
 \setlength{\tabcolsep}{3pt}
 \renewcommand{\arraystretch}{0.95}
 \caption{MEB radius prediction ($R^2$, mean$\pm$std, 5 seeds). Point clouds are mixtures of Gaussians with $n \in [10, 30]$. LA-ST: linear attention~\citep{katharopoulos2020transformers}. $^{\dagger}$Mamba + PMA replaces the ISAB encoder with a Mamba (S6) selective state-space encoder (mambapy 1.2.0). Mean-pool/Mamba dimension-dependence and \methodname{}-Large variant (Appendix~\ref{app:meb}) are discussed in the surrounding paragraph; EGNN comparison: Table~\ref{tab:egnn-comparison}.}
 \label{tab:meb}
 \begin{tabular}{@{}lcccc@{}}
 \toprule
 Model & Params & $d{=}8$ $R^2$ & $d{=}16$ $R^2$ & $d{=}32$ $R^2$ \\
 \midrule
 Standard & 208K & 0.025{\tiny$\pm$0.009} & 0.020{\tiny$\pm$0.005} & 0.006{\tiny$\pm$0.012} \\
 ST-Large & 362K & 0.018{\tiny$\pm$0.001} & 0.012{\tiny$\pm$0.015} & $-$0.006{\tiny$\pm$0.011} \\
 HyperNet ($r{=}8$) & 360K & 0.004{\tiny$\pm$0.013} & 0.005{\tiny$\pm$0.012} & $-$0.001{\tiny$\pm$0.001} \\
 HyperNet ($r{=}64$, full key-rank) & 1.3M & 0.011{\tiny$\pm$0.012} & 0.001{\tiny$\pm$0.009} & 0.000{\tiny$\pm$0.002} \\
 Perceiver & 341K & 0.024{\tiny$\pm$0.006} & 0.016{\tiny$\pm$0.028} & 0.010{\tiny$\pm$0.010} \\
 LA-ST (Linear Attn.) & 359K & 0.017{\tiny$\pm$0.016} & 0.015{\tiny$\pm$0.009} & 0.001{\tiny$\pm$0.002} \\
 Standard (mean-pool head) & 140K & 0.784{\tiny$\pm$0.004} & 0.647{\tiny$\pm$\textbf{0.363}} & 0.449{\tiny$\pm$\textbf{0.261}} \\
 Mamba + PMA$^{\dagger}$ & 645K & 0.713{\tiny$\pm$0.025} & 0.617{\tiny$\pm$0.041} & 0.142{\tiny$\pm$\textbf{0.134}} \\
 \methodname{}-Full & 368K & \textbf{0.691}{\tiny$\pm$0.012} & \textbf{0.629}{\tiny$\pm$0.023} & \textbf{0.378}{\tiny$\pm$0.030} \\
 \bottomrule
 \end{tabular}
\end{table}

Every parameter-matched baseline using the standard PMA pooling, including LA-ST, fails at $R^2 \leq 0.025$, whereas \methodname{}-Full reaches a $28$--$155\times$ gap (median $\approx 40\times$ across $13$ positive-baseline cells, with the $155\times$ upper bound at $d{=}8$ vs.\ HyperNet's $0.0044$; Figure~\ref{fig:meb-tdof}, left). The collapse is structural: HyperNet at $r{=}d_k{=}64$ ($1.3$M params, $3.6\times$ budget) leaves $R^2$ within $\pm 0.012$ of zero (Appendix~\ref{app:hypernet-rank-ext}). Replacing PMA with mean-pool partially escapes the ceiling at $d{=}8$ ($0.784{\pm}0.004$, Standard mean-pool row above) but exhibits heavy-tail bimodal failure at higher dimensions ($d{=}16$: $0.647{\pm}0.363$ with $1/5$ seeds catastrophically collapsing to $R^2 \approx 0$; $d{=}32$: $0.449{\pm}0.261$ with another $1/5$ catastrophic). Replacing the ISAB encoder with a Mamba (S6) selective state-space encoder while keeping PMA pooling (Mamba + PMA row, $1.75\times$ \methodname{}'s budget) matches \methodname{}-Full at $d{=}8$ ($0.713{\pm}0.025$) but degrades to $R^2 = 0.142{\pm}0.134$ at $d{=}32$ with all $5$ seeds below $0.30$. \methodname{}-Full + PMA maintains structural stability ($\sigma \leq 0.030$) at $R^2 \geq 0.378$ across all three dimensions, providing the only uniform high-$d$ escape we observed. Convex-hull-volume and a sparse-geometric-set-predicates family confirm the collapse (Appendices~\ref{app:hull},~\ref{app:sgsp}); rotation matching at $d{=}256$ has all baselines at $R^2 \in [-50, -62]$ while \methodname{} stays at $-0.005$ (Table~\ref{tab:rotmatch-scale}). On MST and convex hull, the raw-target negative-$R^2$ magnitude is partly an optimisation artefact of unnormalised regression on large-magnitude targets; under target normalisation the architectural advantage on MST shrinks to $+0.04$--$0.08$ $R^2$ and the hull $d{=}9$ cell ties (Appendix~\ref{app:opt-confound}). The MEB and quadratic-TDOF separations are scale-free or $O(1)$-magnitude, so the gap there is structural rather than scale-driven.

\begin{table}[t]
 \centering
 \footnotesize
 \setlength{\tabcolsep}{3pt}
 \renewcommand{\arraystretch}{0.95}
 \caption{Rotation matching scaling ($R^2$, mean over 3 seeds). Context-rigid baselines collapse to large negative $R^2$ as $d$ grows; \methodname{} tracks the mean predictor up to $d{=}256$. At $d{=}512$ all architectures diverge but \methodname{}'s MSE stays $2.3\times$ smaller than the next-best.}
 \label{tab:rotmatch-scale}
 \begin{tabular}{@{}lcccccc@{}}
 \toprule
 Model & $d{=}16$ $R^2$ & $d{=}32$ $R^2$ & $d{=}64$ $R^2$ & $d{=}128$ $R^2$ & $d{=}256$ $R^2$ & $d{=}512$ $R^2$ \\
 \midrule
 Standard & 0.084 & $-$0.002 & $-$1.20 & $-$21.54 & $-$61.69 & $-$165.4 \\
 ST-Large & 0.285 & $-$0.001 & $-$0.023 & $-$9.47 & $-$50.15 & $-$145.1 \\
 Perceiver & $-$0.001 & $-$0.006 & $-$1.17 & $-$19.18 & $-$61.52 & $-$165.2 \\
 HyperNet & 0.020 & $-$0.005 & $-$1.10 & $-$16.51 & $-$61.48 & $-$163.0 \\
 \methodname{} & \textbf{0.826} & \textbf{0.519} & \textbf{0.071} & \textbf{$-$0.001} & \textbf{$-$0.005} & \textbf{$-$61.6} \\
 \bottomrule
 \end{tabular}
\end{table}

\paragraph{Real-world transfer.}
On a CIFAR-100 ResNet-feature setup that meets the three prerequisites (high-rank per-sample points from layer-3 features; sparse $k{=}5$ $k$-center furthest-cost target; no imposed symmetry), \methodname{}-Full ranks first across 7 parameter-matched baselines on both ResNet-18 ($d{=}256$) and ResNet-50 ($d{=}1024$) backbones, with Welch's $t \approx 5.2$, $p \approx 0.0018$ at $d{=}1024$ ($n{=}5$ seeds; absolute gap $+0.0046$ $R^2$, uncorrected for the 18-cell TDOF analysis below). The advantage is restricted to the sparse-combinatorial $k$-center target: LinearAttn matches or beats \methodname{} on dense-statistical siblings (diameter, $k$-NN radius, MST) on real ResNet features (Appendix~\ref{app:realworld-section}, Tables~\ref{tab:resnet50-scaling} and~\ref{tab:resnet-features}).

Aggregating across eighteen (task, scale) settings, \methodname{}'s MSE reduction over the best parameter-matched baseline correlates with the estimated TDOF of each task's target operator (Figure~\ref{fig:meb-tdof}, right): Spearman $\rho = 0.74$ ($n{=}18$, two-sided parametric $p < 0.005$). The TDOF estimate is approximate (exact for linear-form tasks, order-of-magnitude for non-linear ones), so we treat the correlation as descriptive. A $\pm 50\%$ TDOF-perturbation sensitivity check over $10^4$ trials keeps Spearman in $[0.675, 0.791]$ at $5$\%--$95$\% (Appendix~\ref{app:tdof-sensitivity}). Tasks with TDOF~$\leq 5$ cluster near zero improvement (range $-2$ to $+19\%$); TDOF~$\geq 6$ show median $58\%$ MSE reduction (range $15$ to $98\%$).

\begin{figure}[t]
 \centering
 \includegraphics[width=\textwidth]{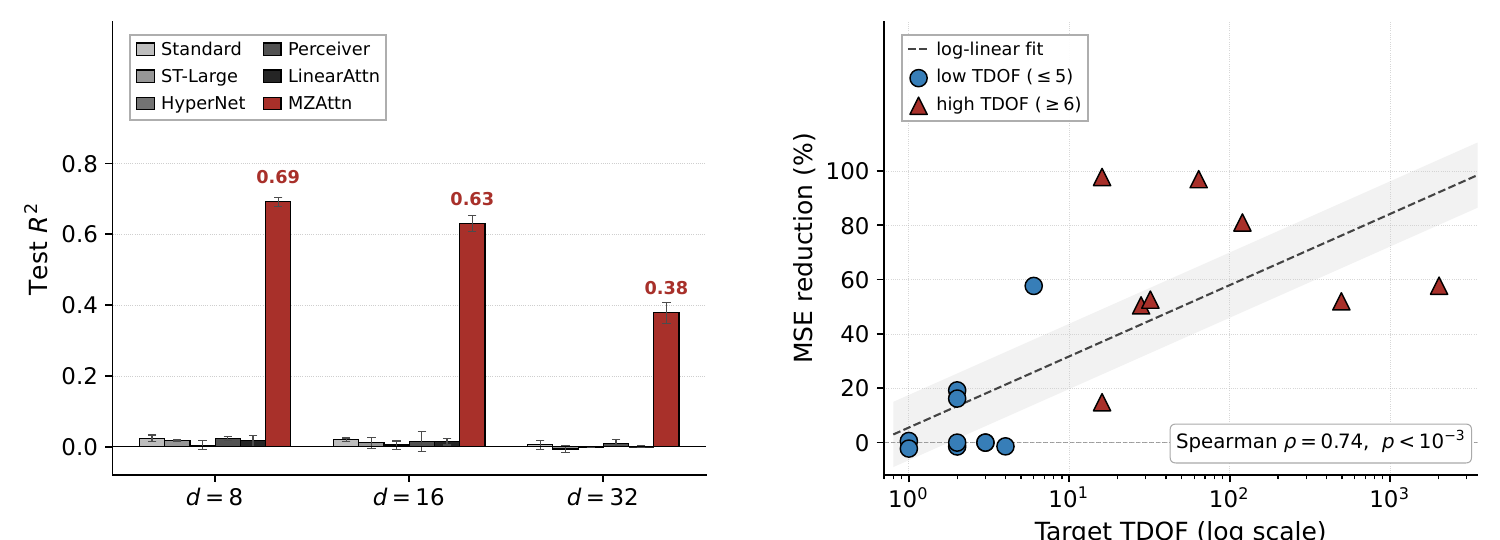}
 \caption{\textbf{The architectural advantage tracks transformation degrees of freedom on tasks where TDOF grows with $d$.} \textbf{Left:} on MEB radius prediction, parameter-matched attention baselines fail at $R^2 \leq 0.025$ across dimensions, while \methodname{} reaches $R^2 = 0.69, 0.63, 0.38$ at $d{=}8, 16, 32$. \textbf{Right:} across $18$ (task, scale) settings, \methodname{}'s MSE reduction over the best baseline correlates with target TDOF (Spearman $\rho = 0.74$, two-sided parametric $p<0.005$, $n{=}18$). Low-TDOF tasks cluster near zero; high-TDOF tasks span $15$--$98\%$ reduction with median $58\%$.}
 \label{fig:meb-tdof}
\end{figure}

Architectural ablations, robustness studies, uncertainty calibration, scaling experiments, and the equivariant comparison are reported in Appendix~\ref{app:additional-exp}.

\subsection{Analysis: What Does the Matrix Zonotope Learn?}
\label{sec:analysis}

We address why the matrix-zonotope parameterisation works while equally-flexible alternatives do not. \textbf{Capacity is not the driver.} HyperNet at $r{=}d_k{=}64$ ($1.3$M params, $3.6\times$ budget) leaves $R^2 \approx 0$ on MEB (Table~\ref{tab:meb}), and the full $r \in \{16, 32, 64\}$ sweep on Hull and MST at high $d$ (Table~\ref{tab:hypernet-rank-ext}) stays within $0.13$ of $r{=}16$ everywhere; conversely MZ-Slim ($n_g{=}2$, $L{=}1$, $1.6\times$ Standard's parameters) still beats parameter-matched ST-Large and HyperNet variants by $4$--$23$pp $R^2$ on Max regression, MST weight, and rotation matching (Table~\ref{tab:mz-efficient}). On the synthetic quadratic-TDOF task (Table~\ref{tab:gen-ablation}), even $L{=}1$ retains $R^2 = 0.978$ and the curve plateaus once $L$ exceeds target TDOF: the gain is the centre-plus-gated-generator structure, not the generator count.

\begin{table}[t]
 \centering
 \footnotesize
 \setlength{\tabcolsep}{3pt}
 \renewcommand{\arraystretch}{0.95}
 \caption{HyperNet rank extended ablation ($R^2$, mean$\pm$std, 3 seeds). Low-rank update $r \in \{16, 32, 64\}$ on collapse-exhibiting tasks. $^{\ddagger}$Hull $d{=}9$ MZ-Full ($0.003$) is within seed std of zero. Target-normalised equivalents: Table~\ref{tab:hull-normed}.}
 \label{tab:hypernet-rank-ext}
 \begin{tabular}{@{}lcccc@{}}
 \toprule
 Task ($d$) & HyperNet $r{=}16$ & $r{=}32$ & $r{=}64$ & \methodname{}-Full \\
 \midrule
 Hull $d{=}3$ & 0.191{\tiny$\pm$.008} & 0.190{\tiny$\pm$.029} & 0.204{\tiny$\pm$.020} & \textbf{0.804}{\tiny$\pm$.014} \\
 Hull $d{=}6$ & $-$0.214{\tiny$\pm$.013} & $-$0.203{\tiny$\pm$.002} & $-$0.214{\tiny$\pm$.027} & \textbf{0.103}{\tiny$\pm$.020} \\
 Hull $d{=}9$ & $-$0.095{\tiny$\pm$.001} & $-$0.098{\tiny$\pm{<}.001$} & $-$0.100{\tiny$\pm$.002} & 0.003{\tiny$\pm$.006}$^{\ddagger}$ \\
 MEB $d{=}8$ & 0.012{\tiny$\pm$.006} & 0.022{\tiny$\pm$.006} & 0.011{\tiny$\pm$.012} & \textbf{0.691}{\tiny$\pm$.012} \\
 MEB $d{=}16$ & 0.007{\tiny$\pm$.003} & 0.000{\tiny$\pm$.008} & 0.001{\tiny$\pm$.009} & \textbf{0.629}{\tiny$\pm$.023} \\
 MEB $d{=}32$ & $-$0.020{\tiny$\pm$.029} & $-$0.000{\tiny$\pm{<}.001$} & 0.000{\tiny$\pm$.002} & \textbf{0.378}{\tiny$\pm$.030} \\
 MST $d{=}64$ & $-$1.96{\tiny$\pm$.31} & $-$2.06{\tiny$\pm$.28} & $-$1.92{\tiny$\pm$.18} & \textbf{0.58}{\tiny$\pm$.01} \\
 MST $d{=}128$ & $-$3.09{\tiny$\pm$.29} & $-$3.00{\tiny$\pm$.17} & $-$2.80{\tiny$\pm$.17} & \textbf{0.48}{\tiny$\pm$.02} \\
 MST $d{=}256$ & $-$3.99{\tiny$\pm$.15} & $-$3.78{\tiny$\pm$.05} & $-$3.97{\tiny$\pm$.17} & \textbf{0.38}{\tiny$\pm$.04} \\
 \bottomrule
 \end{tabular}
\end{table}

\begin{table}[t]
 \centering
 \footnotesize
 \setlength{\tabcolsep}{3pt}
 \renewcommand{\arraystretch}{0.95}
 \caption{MZ parameter efficiency ($R^2$, mean, 3 seeds). Even at $n_g{=}2, L{=}1$ MZ retains its high-TDOF advantage: structure dominates parameter count.}
 \label{tab:mz-efficient}
 \begin{tabular}{@{}lcccc@{}}
 \toprule
 Model & Max $R^2$ $\uparrow$ & RotMatch $R^2$ $\uparrow$ & MST $R^2$ $\uparrow$ & Params \\
 \midrule
 Standard & 0.812 & 0.706 & 0.719 & 208K \\
 \cmidrule{1-5}
 MZ-Slim ($n_g{=}2, L{=}1$) & \textbf{0.871} & 0.919 & 0.835 & 333K \\
 MZ-Med ($n_g{=}4, L{=}2$) & 0.863 & 0.933 & \textbf{0.857} & 344K \\
 MZ-Full ($n_g{=}8, L{=}4$) & 0.850 & \textbf{0.936} & 0.852 & 368K \\
 \bottomrule
 \end{tabular}
\end{table}

A trained $L{=}8$ checkpoint probed with $1{,}000$ contexts shows near-orthogonal generators in operator space, gates active on $78.6\%$ of evaluations, and the soft effective rank of $\{M(\mu_i) - \bar M\}_i$ saturating the architectural budget when $L \leq \mathrm{TDOF}$ and growing sublinearly above (Appendix~\ref{app:architecture-ablations}, Figure~\ref{fig:mz-analysis}); the trained model allocates operator-space variation up to but not beyond the intrinsic TDOF, with $L$ a soft upper bound.

% ══════════════════════════════════════════════════════════════════
\section{Discussion and Conclusion}
\label{sec:discussion}
% ══════════════════════════════════════════════════════════════════

\methodname{} replaces the fixed value projection of standard attention with a learnable matrix zonotope at the same parameter-budget scaling. The gain is selective and tracks the theory (Figure~\ref{fig:meb-tdof}): improvements concentrate on tasks with high-TDOF, sparse-combinatorial targets, and disappear on tasks dominated by per-element features, spatial locality, or known symmetry.

\paragraph{Equivariance versus context-adaptivity.} EGNN~\citep{satorras2021egnn} matches or outperforms \methodname{} on every rotation-invariant task tested: rotation matching at $d{=}8$ ($R^2 = 0.99$ vs.\ $0.95$), MEB radius ($R^2 = 0.78$ vs.\ $0.38$ at $d{=}32$), and Mahalanobis distance ($R^2 = 0.40$ vs.\ $-0.02$). On the quadratic-TDOF target (genuine fixed-frame dependence), EGNN's invariant-output constraint forces $R^2 \approx 0$ while \methodname{} reaches $R^2 \geq 0.90$ (Table~\ref{tab:egnn-comparison}). The two are complementary: EGNN when symmetry is known, \methodname{} when it is unknown or absent.

\begin{table}[t]
 \centering
 \footnotesize
 \setlength{\tabcolsep}{3pt}
 \caption{\methodname{} vs.\ EGNN. EGNN dominates on rotation/translation-invariant targets; \methodname{} dominates on the non-equivariant quadratic-TDOF target where invariance is structurally inadequate.}
 \label{tab:egnn-comparison}
 \resizebox{\textwidth}{!}{%
 \begin{tabular}{@{}l|ccc|cc|cc|ccc|ccc@{}}
 \toprule
 & \multicolumn{3}{c|}{\textbf{Rot. matching}} & \multicolumn{2}{c|}{\textbf{Hull volume}} & \multicolumn{2}{c|}{\textbf{MEB radius}} & \multicolumn{3}{c|}{\textbf{Statistical (Cov, Mah, PCA)}} & \multicolumn{3}{c}{\textbf{Quadratic TDOF}} \\
 & $d{=}8$ & $d{=}32$ & $d{=}64$ & $d{=}3$ & $d{=}9$ & $d{=}8$ & $d{=}32$ & CN d=8 & Mah d=8 & PCA d=16 & $L{=}4$ & $L{=}8$ & $L{=}16$ \\
 \midrule
 ST-Large & 0.92 & 0.00 & $-$0.02 & 0.50 & $-$0.09 & 0.02 & $-$0.01 & 0.90 & $-$0.004 & 0.11 & 0.83 & 0.79 & 0.78 \\
 \methodname{} (ours) & 0.95 & 0.52 & 0.07 & 0.80 & 0.003 & 0.69 & 0.38 & 0.97 & $-$0.022 & 0.58 & \textbf{0.96} & \textbf{0.90} & \textbf{0.90} \\
 EGNN & \textbf{0.99} & \textbf{0.99} & \textbf{0.98} & \textbf{0.95} & \textbf{0.85} & 0.66 & \textbf{0.78} & \textbf{0.99} & \textbf{0.40} & \textbf{0.93} & $-$0.002 & 0.000 & 0.006 \\
 \bottomrule
 \end{tabular}}
\end{table}

\paragraph{Real-world transfer.} On convex-hull-volume across three pretrained backbones (ResNet-50, DINOv2, CLIP) with five parameter-matched context-rigid baselines, $14$ of $15$ cells yield negative $R^2$ (Appendix~\ref{app:realworld-section}), while \methodname{}-Full maintains mean $R^2 \geq 0.82$ on all three backbones (\methodname{}-Large reaches $0.998 \pm 0.005$ on RN50). On dense-statistical siblings (diameter, $k$-NN radius, MST) all architectures cluster within $\leq 0.03$ $R^2$, the predicted TDOF behaviour. On SST-2 sentiment~\citep{socher2013recursive} (low-TDOF), \methodname{} ($0.7982 \pm 0.0114$) matches a parameter-matched Standard Set Transformer ($0.7989 \pm 0.0084$) within $0.07$pp ($3$ seeds), the strict-generalisation property a value-projection replacement should satisfy.

\paragraph{Scope and limitations.} The advantage is structurally targeted, not universal: synthetic high-TDOF benchmarks serve as the controlled testbed; the $18$-setting Spearman $\rho = 0.74$ (Section~\ref{sec:experiments}) is the falsifiability check; QM9 and point-cloud rows are the predicted low-TDOF negative controls (within $\leq 0.04$ $R^2$). \methodname{} costs $1.7$--$1.9\times$ parameters and $2.4$--$2.7\times$ inference vs.\ Standard ST (Appendix~\ref{app:computational-cost}); the Slim variant retains most of the advantage at $1.6\times$ params. Deep Sets tying on few-shot retrieval (Task~D) and ST-Large winning the parameter-fair 7-way meta-regression at $L{=}32$ (Table~\ref{tab:meta-reg}; an earlier comparison against an undertuned vanilla baseline overstated this gap, corrected here) instantiate the predicted boundary. Very-large-scale sets and structured-zonotope variants remain open.

% ══════════════════════════════════════════════════════════════════
% References
% ══════════════════════════════════════════════════════════════════
\bibliographystyle{plainnat}
\bibliography{ref_MAT}

% ══════════════════════════════════════════════════════════════════
\appendix
\section*{Supplementary Material: Overview}
The supplementary material is organised as follows. Appendix~\ref{app:setup} reports experimental setup, hyperparameters, and compute resources. Appendix~\ref{app:theory} contains the algebraic preliminaries, uncertainty and generalisation results, and full proofs for all theorems and propositions. Appendix~\ref{app:related-extended} expands the related-work discussion. Appendix~\ref{app:additional-exp} reports the extended empirical study, organised into five thematic blocks: architectural ablations, robustness and generalisation, uncertainty calibration, computational-geometry scaling, and real-world transfer plus equivariant comparison.

\section{Experimental Setup, Hyperparameters, and Compute}
\label{app:setup}

Appendix~\ref{app:impl} gives the model and optimisation choices common to all experiments; Appendix~\ref{app:task-descriptions} specifies the per-task data-generation and training settings; Appendix~\ref{app:compute} reports compute resources.

\subsection{Implementation Details}
\label{app:impl}

\paragraph{Hyperparameter summary.}
Headline configuration shared across all attention-based architectures except where overridden per task or by the parameter-matched variants:

\begin{table}[h]
\centering
\small
\setlength{\tabcolsep}{4pt}
\begin{tabular}{@{}ll@{\quad}ll@{}}
\toprule
Optimiser & AdamW~\citep{loshchilov2017decoupled} & LR & $10^{-4}$ \\
LR schedule & cosine to $0$ & LR warmup & $5$ epochs (linear $0\to10^{-4}$) \\
Weight decay & $10^{-5}$ & Gradient clip & $1.0$ ($\ell_2$ norm) \\
Batch size & $128$ (synthetic) / $256$ (geom. tasks) & Epochs (cap) & $200$ (synthetic), $150$ (real-world) \\
Early stopping & val loss patience $20$ & Dropout & $0.1$ (constant) \\
$d_{\text{model}}$ & $64$ (default) / $80$ (ST-Large) & $H$ heads & $4$ \\
$d_{\text{ff}}$ & $128$ (default) / $200$ (ST-Large) & ISAB inducing points & $16$ \\
\methodname{} $n_g$ & $8$ (default) / $2$ (Slim) / $16$ (Large) & \methodname{} $L$ & $4$ (default) / $1$ (Slim) / $8$ (Large) \\
\bottomrule
\end{tabular}
\end{table}

The ResNet-50 ($d{=}1024$) experiments use early-stopping patience $300$ (Appendix~\ref{app:patience-sweep} verifies the rank ordering is stable for any patience $\geq 100$).

\paragraph{MZ parameter initialization.}
Center matrices $M_c^{(h)}$ are initialized as $I_{d_k}$.
Generator matrices $M_l^{(h)}$ are initialized from $\mathcal{N}(0, 0.02^2)$.
Mixing weights $W_{\text{mix}}$ are initialized from $\mathcal{N}(0, 0.1^2)$.
All other linear layers use Xavier uniform initialization.

\paragraph{Generator FFN.}
The generator path uses a lighter FFN with hidden dimension $d_{\text{ff}} / 2$ (compared to $d_{\text{ff}}$ for centers), since generators represent small perturbation directions and require less modeling capacity.

\paragraph{Computational cost comparison.}
For the default configuration ($d = 64$, $H = 4$, $n_g = 8$, $L = 4$), the MZ-Set Transformer has approximately $372$K parameters compared to $208$K for the standard Set Transformer.
The additional parameters come primarily from the generator projection layer ($W_g$: $d_{\text{in}} \times n_g d$), generator FFNs, and the MZ parameters ($H \times (1 + L) \times d_k^2 = 4 \times 5 \times 256 = 5{,}120$ matrix-zonotope parameters per attention layer, plus a small $O(LHd_k)$ for the gating network and the mixing matrix $W_{\mathrm{mix}}$).

\subsection{Per-Task Data and Training Settings}
\label{app:task-descriptions}

\paragraph{Task A: Set regression.}
Given a set $S = \{x_1, \dots, x_n\}$ with $x_i \in \R^{16}$ and $n \sim \text{Uniform}(10, 100)$, predict $y = \sum_i [\max_{\text{dim}} x_i]$, i.e., the sum of dimension-wise maxima.
10K train, 2K val, 2K test; 200 epochs; 5 seeds (3 seeds for parameter-matched models).

\paragraph{Task B: Set matching.}
Given two point clouds $A, B \subset \R^3$ (each 10--30 points), predict the Chamfer distance $\text{CD}(A, B)$.
Set $B$ is either a rotated-noisy version of $A$ (50\%) or an independent random cloud (50\%).
Both sets are packed into a single input with a binary type indicator feature, giving $d_{\text{in}} = 4$.
5K train, 1K val, 1K test; 200 epochs; 3 seeds.

\paragraph{Task C: Point cloud classification (ModelNet40).}
We classify 3D shapes from ModelNet40~\citep{wu2015modelnet}, which comprises 40 categories with 9{,}843 training and 2{,}468 test instances.
Each shape is represented by 1{,}024 randomly sampled surface points in $\R^3$.
Training uses random y-axis rotation and jittering for augmentation; 200 epochs; 5 seeds.

\paragraph{Task D: Few-shot set retrieval.}
Given a support set $S$ of $K{=}5$ examples from a target class and a query set $Q$ of 20 elements from mixed classes, predict the fraction of $Q$ belonging to the target class. Each element has 8 feature dimensions plus a binary type indicator, yielding $d_{\text{in}} = 9$.
Class separability varies per sample (noise scale $\in [0.3, 1.5]$), creating natural relational ambiguity.
8K train, 1.5K val, 1.5K test; 200 epochs; 3 seeds.

\paragraph{Task E: ScanObjectNN classification (real-world).}
We evaluate on ScanObjectNN~\citep{uy2019scanobjectnn}, a realistic 3D point cloud classification benchmark constructed from real indoor scans.
We use the OBJ\_BG variant (15 classes, ${\sim}$2{,}300 train, ${\sim}$600 test).
Each shape is represented by 1{,}024 randomly sampled points in $\R^3$ with y-axis rotation and jitter augmentation.
200 epochs; 5 seeds.

\paragraph{Task F: Rotation-dependent set matching.}
Given two point clouds $A, B \subset \R^8$ where $B = R(\theta) A + \varepsilon$ for a sample-dependent rotation $R(\theta)$ (composed from Givens rotations in all ${8 \choose 2} = 28$ planes) and i.i.d.\ noise $\varepsilon \sim \mathcal{N}(0, 0.1^2 I)$, predict the post-alignment Chamfer distance.
This task requires the model to internally discover and apply a context-dependent rotation to align $A$ and $B$ before computing the residual distance, an inherently off-diagonal operation.
FiLM conditioning, which applies $\gamma(\mu) \odot z + \beta(\mu)$ (diagonal scale and shift), provably cannot express rotation in $\R^8$.
Both sets are packed with a binary type indicator, giving $d_{\text{in}} = 9$.
5K train, 1K val, 2K test; 200 epochs; 3 seeds.

\paragraph{Task G: Minimum spanning tree weight in $\R^8$.}
Given a point set $S = \{x_1, \dots, x_n\} \subset \R^8$ with $n \sim \text{Uniform}(10, 30)$, predict the total weight (sum of edge lengths) of the minimum spanning tree of $S$.
This is a classical computational geometry quantity that depends on the global pairwise distance structure of the set, a natural relational task where the output is determined by the collective geometric arrangement rather than any single element.
Points are generated with 1--3 random clusters (varying centers and spreads) to create diverse geometric configurations.
No explicit context signal is broadcast; the model must infer the relational structure from pairwise distances in high-dimensional space.
5K train, 1K val, 2K test; 200 epochs; 3 seeds.

\paragraph{Task H: QM9 molecular property prediction.}
Given a molecule represented as a set of atoms $S = \{a_1, \dots, a_n\}$ with $n \leq 29$ and $a_i \in \R^9$ (5-dim atom type one-hot encoding for H/C/N/O/F, 3D Cartesian coordinates, partial charge), predict quantum chemical properties: dipole moment ($\mu$, Debye) and isotropic polarizability ($\alpha$, Bohr$^3$).
Data is drawn from the QM9 dataset~\citep{ramakrishnan2014quantum}, which contains ${\sim}$134K small organic molecules with DFT-computed properties.
We use 50K molecules split 80/10/10; 100 epochs; 3 seeds.

\subsection{Compute Resources}
\label{app:compute}

All experiments were run on a mix of consumer and datacentre GPUs. Synthetic
benchmarks (Tasks A, B, F, G, MEB, hull, MST, rotation matching at $d \leq 64$,
sparse-geometric set predicates, gen-/HyperNet-/MZ-Slim ablations,
uncertainty calibration, set-size generalisation, ambiguity sweep,
quadratic-TDOF / depth separation / sequence MZ) used a single NVIDIA RTX 5090
(32 GB). Higher-dimensional rotation matching ($d \in \{128, 256, 512\}$) and
MST scaling ($d \in \{64, 128, 256\}$) used a single NVIDIA H200 (141 GB).
ResNet-50 layer-3 transfer at $d{=}1024$ used a single NVIDIA H200. Per-seed
training time on the 5090 ranged from $0.5$ to $30$ minutes depending on task
and dimension. The full empirical sweep (all main-body tables and appendix
ablations combined) accumulates approximately $1{,}200$ GPU-hours when
re-run from scratch with default seeds.

\section{Theory: Algebra, Structural Properties, and Proofs}
\label{app:theory}

This section consolidates the theoretical material that was previously
fragmented across four short appendices (zonotope algebra, uncertainty
certificates, additional theoretical results, and full proofs). It
collects the algebraic preliminaries used by the proofs in
Section~\ref{sec:theory}, the structural-property statements summarised
in the body, and the proofs of every theorem and proposition cited above.

\subsection{Zonotope and Matrix-Zonotope Algebra}
\label{app:zonotope}

For readers unfamiliar with zonotopes, we provide a self-contained primer.

A zonotope $\mathcal{Z} = \{c + G\xi : \|\xi\|_\infty \le 1\}$ can be visualized as a centrally symmetric polytope generated by ``sweeping'' the center along each generator direction.
In 2D, a zonotope with 3 generators is a hexagon; with 4 generators, an octagon.

The interval hull $\text{IH}(\mathcal{Z}) = [c - \delta, c + \delta]$ with $\delta_i = \sum_k |G_{ik}|$ is the tightest axis-aligned box containing $\mathcal{Z}$.

Key operations:
\begin{itemize}
 \item \textbf{Linear map:} $A\mathcal{Z} = \{Ac + AG\xi\} = \zono{Ac, AG}$. Zonotopes are closed under linear maps.
 \item \textbf{Minkowski sum:} $\mathcal{Z}_1 \oplus \mathcal{Z}_2 = \zono{c_1 + c_2, [G_1, G_2]}$. Generator count grows additively.
 \item \textbf{MZ multiplication:} $\MZ \mathcal{Z}$ produces a zonotope with center $M_c c$ and generators $\{M_c g_k\} \cup \{M_l c\} \cup \{M_l g_k\}$.
 Generator count grows as $h + L + Lh$ (hence the need for budget control).
\end{itemize}

For a thorough treatment including constrained and polynomial variants, see~\citet{kochdumper2023constrained} and~\citet{Althoff2010PhD}.

\subsection{Uncertainty Certificates, Generalization, and Structural Properties}
\label{app:uncertainty-generalization}

This section contains the full theorem statements deferred from Section~\ref{sec:mixing-theory}.

\begin{theorem}[Directional vs.\ scalar sensitivity]
\label{thm:sensitivity}
Consider a single MZ-attention layer with fixed attention scores.
Under value perturbation $\tilde{\hat{c}} = \hat{c} + \hat{G}\xi$ for $\xi \in [-1,1]^{n_g}$:
\begin{enumerate}[label=(\alph*)]
 \item A standard attention layer with output projection $W_o$ provides only a scalar bound: $\|\Delta o\| \leq \|W_o\|_2 \cdot \|\hat{G}\|_{1,2}$.
 \item MZ-attention provides a directional zonotope certificate: $\Delta o \in \mathcal{Z}_{\Delta} = \{M_c \hat{G}\xi : \|\xi\|_\infty \leq 1\}$, contained in the ball of radius $\|M_c\|_2\,\|\hat{G}\|_{1,2}$.
 \item For matched output-projection norms ($\|W_o\|_2 = \|M_c\|_2$), the MZ zonotope is contained in the scalar ball, $\mathcal{Z}_\Delta \subseteq B(\|M_c\|_2\,\|\hat{G}\|_{1,2})$, and is strictly tighter than the ball along every direction orthogonal to the column span of $M_c \hat{G}$ whenever that column span is not the full ambient space.
\end{enumerate}
\end{theorem}

The qualitative content of part~(c) is that MZ replaces a uniform-radius ball with an axis-aligned-in-input-space zonotope; this distinguishes \methodname{} from generic context-adaptive methods, which provide only a scalar reachable diameter, by specifying the change per direction.

\begin{theorem}[IHW as prediction uncertainty bound]
\label{thm:ihw-calibration}
Let $\mathrm{IHW}(Z_{\mathrm{out}}) = \sum_{g=1}^{n_g} \|G_g\|_1$ denote the interval hull width of the output zonotope.
Then:
\begin{enumerate}[label=(\alph*)]
 \item The reachable output diameter satisfies $\mathrm{diam}(\mathcal{Y}(S)) \leq 2\,\|W_{\mathrm{out}}\|_2 \cdot \mathrm{IHW}$.
 \item The output variance under uniform $\xi$ satisfies $\mathrm{Var}_\xi[W_{\mathrm{out}}^\top(c_{\mathrm{out}} + G_{\mathrm{out}}\xi)] \leq \frac{1}{3}\|W_{\mathrm{out}}\|_2^2 \cdot \mathrm{IHW}^2$.
 \item Higher IHW implies a strictly wider reachable output range (monotonicity).
\end{enumerate}
\end{theorem}

\begin{theorem}[Norm-controlled generalization]
\label{thm:generalization}
For an MZ-Set Transformer with $D$ encoder layers, $i.i.d.$ training data of size $n$, and margin $\gamma > 0$, the generalization gap is bounded by $\tilde{O}(B_{\mathrm{MZ}} / \gamma\sqrt{n})$,
where $B_{\mathrm{MZ}} = \prod_{\ell=1}^{D} (\|M_c^{(\ell)}\|_2 + \sum_l \|M_l^{(\ell)}\|_2) \cdot \|W_V^{(\ell)}\|_2 \cdot (\sum_\ell R_\ell^{2/3})^{3/2}$.
At initialisation, when $M_c = I$ and $M_l \approx 0$, $B_{\mathrm{MZ}}$ matches standard attention; it grows only as the model learns to depart from context-rigid behaviour.
\end{theorem}

\begin{proof}[Proof sketch]
The result follows from the spectrally-normalised Rademacher complexity bound of \citet{bartlett2017spectrally} with one substitution. The post-aggregation map of an MZ-attention layer is $\hat v \mapsto M(\mu(S))\hat v$ with worst-case spectral norm $\|M(\mu(S))\|_2 \leq \|M_c\|_2 + \sum_l |\gamma_l(\mu)|\,\|M_l\|_2 \leq \|M_c\|_2 + \sum_l \|M_l\|_2$, since $|\gamma_l(\mu)| \le 1$. This replaces the per-layer spectral factor $\|W^{(\ell)}\|_2$ in the standard bound by the MZ-layer worst-case spectral norm. The aggregation step $\sum_j \alpha_{ij} v_j$ is a softmax-weighted convex combination, hence a $1$-Lipschitz contraction in $\hat v$ that does not inflate the Rademacher complexity by the Talagrand contraction principle~\citep[Lemma~5.7]{mohri2018foundations}. Composing across the $D$ MZ-encoder layers and including the value projections $W_V^{(\ell)}$ yields the stated $B_{\mathrm{MZ}}$; the residual $(\sum_\ell R_\ell^{2/3})^{3/2}$ factor and the post-hoc-margin transfer are identical to the original argument.
\end{proof}

This bound has a natural interpretation: generalisation is controlled by how much context-adaptivity the model learns to use, giving an interpretable complexity measure that connects to the relational-capacity framework.

\begin{corollary}[Universality is preserved]
\label{thm:universal}
The MZ-Set Transformer family contains the standard Set Transformer family as the special case $M_l^{(\ell)} = 0$, $W_{\mathrm{mix}}^{(\ell)} = 0$, $M_c^{(\ell)} = W_o^{(\ell)}$. Consequently, the universal-approximation property of \citet[Theorem~2]{lee2019set} carries over: for any compact $\mathcal{K} \subset \R^d$, continuous permutation-invariant $f{:}\;\mathcal{K}^{\leq N} \to \R^m$, and $\epsilon > 0$, an MZ-Set Transformer $\epsilon$-approximates $f$ uniformly. The MZ family is strictly larger (Theorem~\ref{thm:separation}), so universality is preserved without loss.
\end{corollary}

\begin{proposition}[Structural properties]
\label{prop:perm}
MZ-SAB and MZ-ISAB are permutation-equivariant; the full MZ-Set Transformer is permutation-invariant. At initialisation, with $M_c = I$, $M_l \approx 0$, and $W_{\mathrm{mix}} \approx 0$, \methodname{} reduces exactly to standard multi-head attention with an identity output projection.
\end{proposition}

\subsection{Additional Theoretical Results}
\label{app:additional-theory}

This section contains the expressiveness results and supporting lemmas deferred from Section~\ref{sec:theory}.

\begin{theorem}[Strict expressiveness separation]
\label{thm:separation}
For any nonzero $A \in \R^{d \times d}$, define $f_A(S) = (I + \tanh(\mu(S)_1) \cdot A)\, \mu(S)$.
Then: (a)~a single \methodname{}-attention head with $L = 1$ followed by a fixed linear readout (over the pooled center and the first signed pooled generator) computes $f_A$ exactly via the explicit closed-form construction below; (b)~no standard multi-head attention layer (any $H$, without FFN, with any output linear projection) can represent $f_A$ on the family of constant inputs $\{S_t\}_{t \in \R}$.
\end{theorem}

\begin{proof}
(a) We construct each parameter explicitly. Set $W_Q = 0$ so all softmax weights are uniform $\alpha_{ij} = 1/n$, giving the aggregated center $\hat c_i = \mu(S)$ for every query $i$. Set $W_V = I_{d}$ so the value zonotope's center equals the input. Set $M_c^{(h)} = I_{d}$, $L = 1$, $M_1^{(h)} = A$, the gate weight $w_\gamma = e_1 \in \R^{d}$ so $\gamma_1(\mu) = \tanh(e_1^\top \mu) = \tanh(\mu_1)$, and the mixing weight $W_{\mathrm{mix}} = e_1 \in \R^{n_g \times 1}$ (only the first generator slot is used). Substituting into Eq.~\eqref{eq:mz-main} gives center output $o_i^c = \mu(S)$ and generator output whose first row is $o_i^G[1,:] = \tanh(\mu_1)\, A\,\mu(S)$, with rows $2,\dots,n_g$ all zero.

After PMA pooling with a single seed vector configured analogously ($W_Q = 0$, uniform aggregation), the pooled center is $\mu(S)$ and the pooled signed first-generator is $\tanh(\mu_1)\,A\,\mu(S)$. The fixed linear readout $W_{\mathrm{out}}: (c, g_1) \mapsto c + g_1$ then yields
\[
W_{\mathrm{out}}(o^c, o^G[1,:]) = \mu(S) + \tanh(\mu_1)\,A\,\mu(S) = (I + \tanh(\mu_1)\,A)\,\mu(S) = f_A(S),
\]
exactly. No nonlinear approximation is invoked: every operation in the construction is either a fixed linear map or the native MZ gate $\tanh(w_\gamma^\top \bar c^{kv})$, both of which compute their output without finite-width truncation.

(b) On constant sets $S_t = \{te_1, \ldots, te_1\}$, softmax symmetry gives $o(S_t) = (\sum_h W_o^{(h)} W_V^{(h)}) t e_1 = t\,W e_1$ (linear in $t$). But $f_A(S_t) = t e_1 + t\tanh(t)\, Ae_1$ is nonlinear in $t$ whenever $Ae_1 \neq 0$, contradicting linearity.
\end{proof}

\begin{corollary}[Depth separation without FFN]
\label{cor:depth}
There exist continuous permutation-invariant functions computable by a single MZ-attention layer that require at least two standard attention layers with intermediate nonlinearity.
\end{corollary}

\begin{theorem}[Width efficiency over standard attention with FFN]
\label{thm:width}
For $\mathcal{F}_L = \{S \mapsto (M_0 + \sum_{l=1}^L \sigma_l(\mu(S)) M_l)\, \mu(S)\}$:
(a)~MZ-attention represents any $f \in \mathcal{F}_L$ exactly with $(L{+}1)d_k^2 + O(Ld_k)$ parameters;
(b)~standard attention with depth-2 ReLU FFN requires $d_{\mathrm{ff}} = \Omega(L d_k / \sqrt{\epsilon})$ for $\epsilon$-approximation.
\end{theorem}

\begin{proof}[Proof sketch]
(a) Direct. (b) Each bilinear term $\sigma_l \cdot (M_l h)_j$ requires $\Omega(1/\sqrt{\epsilon})$ ReLU neurons~\citep{yarotsky2017error}; summing over $L$ terms and $d_k$ coordinates gives the bound.
\end{proof}

\begin{theorem}[Minimax parameter efficiency]
\label{thm:minimax}
Let $\mathcal{F}_{L,B}$ denote the class $\mathcal{F}_L$ with bounded generator matrices $\|M_l\|_F \leq B$ and 1-Lipschitz gates.
(a)~MZ-attention parameterizes $\mathcal{F}_{L,B}$ exactly using $P_{\mathrm{MZ}} = (L{+}1)d_k^2 + O(Ld_k) = \Theta(Ld_k^2)$ parameters;
(b)~any model class $\epsilon$-approximating $\mathcal{F}_{L,B}$ in sup-norm over inputs of radius $R$ requires $P \geq \Omega(Ld_k^2 \log(BR/\epsilon))$ parameters;
(c)~hence the MZ parameterization is order-optimal up to logarithmic factors.
\end{theorem}

\begin{definition}[Transformation degrees of freedom]
\label{def:tdof}
For a permutation-invariant function $f(S) = T(S)\,\mu(S)$ with set-dependent linear operator $T(S) \in \R^{d \times d}$, the transformation degrees of freedom is
$\mathrm{TDOF}(f) = \dim\bigl(\mathrm{span}\{T(S) - T(S') : S, S' \in \mathcal{X}^*\}\bigr),$
i.e.\ the dimension of the subspace of $\R^{d \times d}$ spanned by all possible differences of $T$ across inputs.
\end{definition}

\paragraph{TDOF of natural set functions.}

\begin{proposition}[TDOF of natural set functions]
\label{prop:tdof-natural}
Let $S \subset \R^d$ with $|S| \geq d{+}1$ in general position, sample mean $\mu(S)$, and regularised covariance $\hat{\Sigma}(S) = \frac{1}{|S|}\sum_i (x_i - \mu)(x_i - \mu)^\top + \lambda I$ for $\lambda > 0$.
(a)~Mahalanobis ($T(S) = \hat{\Sigma}(S)^{-1}$) has $\mathrm{TDOF} = d(d{+}1)/2$;
(b)~Whitening ($T(S) = \hat{\Sigma}(S)^{-1/2}$) has $\mathrm{TDOF} = d(d{+}1)/2$;
(c)~PCA projection onto the top $k$ principal components has $\mathrm{TDOF} = kd - k(k{+}1)/2$;
(d)~Per-coordinate z-score ($T(S) = \mathrm{diag}(\sigma_j^{-1})$) has $\mathrm{TDOF} = d$.
\end{proposition}

\begin{proof}
\textbf{(a)--(b).} As $S$ varies over generic configurations of $|S| \geq d{+}1$ points in $\R^d$, the regularised covariance $\hat{\Sigma}(S)$ ranges over all of $\mathrm{Sym}^+_d$ (the cone of $d \times d$ positive-definite matrices), an open subset of the symmetric matrices $\mathrm{Sym}_d$. The maps $\Sigma \mapsto \Sigma^{-1}$ and $\Sigma \mapsto \Sigma^{-1/2}$ are smooth diffeomorphisms on $\mathrm{Sym}^+_d$ with non-singular Jacobian (e.g.\ $\partial \Sigma^{-1}/\partial \Sigma \cdot H = -\Sigma^{-1} H \Sigma^{-1}$ is full-rank as a linear map $\mathrm{Sym}_d \to \mathrm{Sym}_d$), so $T(S)$ ranges over an open subset of $\mathrm{Sym}_d$. Hence $\mathrm{span}\{T(S) - T(S')\}$ contains a basis of $\mathrm{Sym}_d$, giving $\mathrm{TDOF} = \dim \mathrm{Sym}_d = d(d{+}1)/2$.
\textbf{(c).} The top-$k$ PCA projector is $T(S) = V_k(S) V_k(S)^\top$ where $V_k \in \R^{d \times k}$ has orthonormal columns. As $S$ varies, $V_k$ ranges over the Stiefel manifold $\mathrm{St}(k, d)$ of dimension $kd - k(k{+}1)/2$, so the smooth orbit $\{V_k V_k^\top\}$ is the Grassmannian $\mathrm{Gr}(k, d)$ of the same dimension, and the linear span of differences fills the tangent space at any point, which is also of dimension $kd - k(k{+}1)/2$.
\textbf{(d).} $T(S) = \mathrm{diag}(1/\sigma_1, \ldots, 1/\sigma_d)$ where $\sigma_j(S)$ varies independently over $(\lambda^{1/2}, \infty)$ as $S$ ranges over configurations whose marginal variances can be set freely; the differences span the $d$-dimensional space of diagonal matrices.
\end{proof}

\paragraph{$\mathcal{F}_L$ as first-order approximation of smooth operators.}

The following result shows that $\mathcal{F}_L$ is not an artificial class tailored to MZ-attention, but rather the natural linearization class for any smooth context-adaptive operator.

\begin{proposition}[$\mathcal{F}_L$ captures linearized context-adaptive operators]
\label{prop:taylor-fl}
Let $T{:}\;\R^d \to \R^{m \times d}$ be twice continuously differentiable, and define $f(S) = T(\mu(S))\,\mu(S)$ where $\mu(S) = \frac{1}{|S|}\sum_{x \in S} x$.
Let $J = \frac{\partial\,\mathrm{vec}(T)}{\partial\mu}\big|_{\mu_0} \in \R^{md \times d}$ be the Jacobian at a reference point $\mu_0$, with SVD $J = \sum_{l=1}^{r} s_l\, u_l v_l^\top$ where $r = \mathrm{rank}(J)$.
Define the linearized operator
\[
 \tilde{T}(\mu) = T(\mu_0) + \sum_{l=1}^{r} s_l\, (v_l^\top (\mu {-} \mu_0))\, \mathrm{mat}(u_l),
\]
where $\mathrm{mat}{:}\;\R^{md} \to \R^{m \times d}$ reshapes vectors into matrices, and the corresponding function $\tilde{f}(S) = \tilde{T}(\mu(S))\,\mu(S)$.
Then:
\begin{enumerate}[label=(\alph*)]
 \item $\tilde{f} \in \mathcal{F}_{r, B}$ with $B = s_1$ (the largest singular value of $J$), center matrix $M_0 = T(\mu_0)$, generator matrices $M_l = s_l\, \mathrm{mat}(u_l)$, and gating functions $\sigma_l(\mu) = v_l^\top(\mu - \mu_0)$, which are linear and hence 1-Lipschitz after rescaling by $1/\|v_l\|$.
 \item The approximation error satisfies
 \[
 \|f(S) - \tilde{f}(S)\| \leq \tfrac{1}{2}\,C_T\,\|\mu(S) - \mu_0\|^2 \cdot \|\mu(S)\|,
 \]
 where $C_T = \sup_{\mu} \|\nabla^2 \mathrm{vec}(T)(\mu)\|_{\mathrm{op}}$ is bounded on any compact domain.
 \item Consequently, $L = \mathrm{rank}(J)$ MZ generators suffice to capture the first-order context-adaptive behavior of $T$ around $\mu_0$, and Theorem~\ref{thm:minimax} guarantees this representation is parameter-optimal.
\end{enumerate}
\end{proposition}

\begin{proof}
Part~(a) is immediate from the SVD construction: each component $\sigma_l(\mu) \cdot M_l$ is a rank-1 term in the Jacobian decomposition of $T$, and the gating functions are linear projections of $\mu$.

Part~(b) follows from Taylor's theorem applied to the vector-valued map $\mathrm{vec}(T)$:
\[
 \|\mathrm{vec}(T(\mu)) - \mathrm{vec}(T(\mu_0)) - J(\mu - \mu_0)\| \leq \tfrac{1}{2} C_T \|\mu - \mu_0\|^2.
\]
Reshaping and multiplying by $\mu(S)$:
$\|f(S) - \tilde{f}(S)\| = \|(T(\mu) - \tilde{T}(\mu))\,\mu\| \leq \|T(\mu) - \tilde{T}(\mu)\|_F \cdot \|\mu\| \leq \frac{1}{2} C_T \|\mu - \mu_0\|^2 \|\mu\|$.

Part~(c) combines (a) with Theorem~\ref{thm:minimax}(a): the $\Theta(rd_k^2)$ MZ parameters are both sufficient and necessary (up to log factors) for $\epsilon$-approximation of the linearized operator family.
\end{proof}

\begin{remark}[Interpretation]
Proposition~\ref{prop:taylor-fl} establishes that $\mathcal{F}_L$ is to context-adaptive operators what linear models are to nonlinear regression: the first-order approximation class.
The Jacobian rank $r$ determines the intrinsic dimensionality of how the operator varies locally, and MZ-attention with $L = r$ generators is the minimax-optimal parameterization for this class.
A standard attention mechanism, which has zero relational capacity by Theorem~\ref{thm:bottleneck}(a), corresponds to the zeroth-order approximation $T \equiv T(\mu_0)$, discarding all first-order variation.
\end{remark}

\paragraph{TDOF of soft Chamfer matching.}

\begin{proposition}[Soft Chamfer matching has nontrivial TDOF]
\label{prop:chamfer-tdof}
Consider the soft nearest-neighbor operator $\mathrm{nn}_\beta(a, B) = \sum_{j=1}^m w_j(a,B)\, b_j$ with softmax weights $w_j = \mathrm{softmax}_{j'}(-\beta\|a - b_{j'}\|^2)$, and define the residual operator $T_a(B) = I_d - P_a(B)$ where $P_a(B)\,a = \mathrm{nn}_\beta(a, B)$.
Then:
\begin{enumerate}[label=(\alph*)]
 \item The operator $T_a(B)$ is context-adaptive: for any $m \geq 2$, there exist configurations $B, B'$ such that $T_a(B) \neq T_a(B')$.
 \item $\mathrm{TDOF}(T_a) \geq d$ when $m \geq 2$ and $\beta > 0$: as $B$ varies over $(\R^d)^m$, the operator $T_a(B)$ spans a subspace of $\R^{d \times d}$ of dimension at least $d$.
\end{enumerate}
\end{proposition}

\begin{proof}
(a) Immediate: different $B$ configurations change the matching weights $w_j$, hence $P_a(B)$.

(b) Fix $a \in \R^d$ and consider $B = \{0, t\, e_k\}$ for each standard basis vector $e_k$, $k = 1, \ldots, d$.
The soft nearest-neighbor is $\mathrm{nn}_\beta(a, B) = w_2(t,k)\, t\, e_k$ where $w_2 = (1 + \exp(-\beta(\|a\|^2 - \|a - te_k\|^2)))^{-1}$.
The Jacobian $\partial \mathrm{nn}_\beta / \partial t\big|_{t=0}$ has a nonzero component in the $e_k$ direction.
Since this holds independently for each $k$, the operator $P_a(B)$ varies in at least $d$ independent directions as $B$ ranges over $(\R^d)^m$.
By Definition~\ref{def:tdof}, $\mathrm{TDOF}(T_a) \geq d$.
\end{proof}

\begin{remark}
For the full Chamfer distance $\mathrm{CD}(A,B) = \frac{1}{n}\sum_{i} \|a_i - \mathrm{nn}(a_i, B)\|^2$, each query point $a_i$ contributes $d$ independent relational directions.
In the hard-matching limit $\beta \to \infty$, distinct matchings correspond to combinatorially many regions in configuration space, each defining a different piecewise-linear operator, well beyond the reach of a single context-rigid transformation.
This explains the large gap between \methodname{} and standard attention on set matching (Task~B in Table~\ref{tab:main-results}): the Chamfer distance target requires a context-adaptive operator, and the standard Set Transformer's fixed $W_O$ cannot represent the matching-dependent structure.
\end{remark}

\begin{lemma}[Compositional orthogonality under the worst-case construction]
\label{lem:comp-orth}
Fix $k \le d-1$ and the explicit witness $f \in \mathcal{F}_k$ defined by $\sigma_l(\mu) := \tanh(\mu_l)$, $M_l := e_l\, e_d^\top \in \R^{d\times d}$ for $l=1,\dots,k$, and $M_0 := 0$ (the constant-direction term is set to zero so that the network's linear path provides no free target-direction contribution; this is without loss of generality for the lower-bound argument). Let $\mu \in \R^d$ have i.i.d.\ standard-Gaussian coordinates. Then:
\begin{enumerate}[label=(\alph*)]
 \item $\langle M_l, M_{l'}\rangle_F = \delta_{ll'}$, so $\{M_l\}_{l=1}^k$ is Frobenius-orthonormal.
 \item For every $r \ge 2$ and every $l_1,\dots,l_r \in \{1,\dots,k\}$, $M_{l_1} M_{l_2}\cdots M_{l_r} = 0$.
 \item Consequently, in the tensor-product space $L^2(\mu)\otimes\R^{d\times d}$ with inner product $\langle a\otimes A, b\otimes B\rangle := \mathbb{E}_\mu[ab]\,\langle A,B\rangle_F$, every monomial composition term of the form $P(\sigma)\otimes (M_{l_1}\cdots M_{l_r})$ with composition order $r\ge 2$ vanishes identically (the matrix factor is zero), and therefore has zero projection on the target subspace $\mathcal{V}=\mathrm{span}\{\sigma_l\otimes M_l\}_{l=1}^k$.
\end{enumerate}
\end{lemma}

\begin{proof}
\textbf{(a)} $\langle e_l e_d^\top, e_{l'} e_d^\top\rangle_F = (e_l^\top e_{l'})(e_d^\top e_d) = \delta_{ll'}\cdot 1 = \delta_{ll'}$.

\textbf{(b)} For any product of length $r\ge 2$,
$M_{l_1} M_{l_2} = (e_{l_1} e_d^\top)(e_{l_2} e_d^\top) = e_{l_1}\,(e_d^\top e_{l_2})\,e_d^\top$. Since $l_2 \in \{1,\dots,k\} \subseteq \{1,\dots,d-1\}$ and $d \ne l_2$, the inner factor $e_d^\top e_{l_2} = 0$, so $M_{l_1} M_{l_2} = 0$. Any longer product factors through this length-$2$ block and is therefore also $0$.

\textbf{(c)} For any monomial $P(\sigma)$ in $\{\sigma_l\}$ and any composition order $r\ge 2$, the matrix factor $M_{l_1}\cdots M_{l_r} = 0$ by (b). Hence $P(\sigma)\otimes(M_{l_1}\cdots M_{l_r}) = 0$ in the tensor-product space, and its projection on $\mathcal{V}$ (or on any subspace) is $0$.
\end{proof}

\begin{remark}
The construction $M_l = e_l e_d^\top$ is the simplest worst-case witness: each $M_l$ is a rank-$1$ outer product whose left factor $e_l$ varies with $l$ but whose right factor $e_d$ is a single fixed direction not used as a target index. This makes the matrix family closed under no products of order $\ge 2$ (every such product collapses to $0$), so a $D$-layer attention network cannot bootstrap target directions via composition. Other Frobenius-orthonormal $\{M_l\}$ admit non-zero matrix products and would require a separate analysis; we use the simplest construction for which the depth lower bound is unconditional.
\end{remark}

\paragraph{Generator mixing characterization.}

\begin{proposition}[Generator mixing]
\label{prop:mixing-app}
For mixed generators $\tilde{g}_k = g_k + \sum_l w_{kl} g'_l$:
(a)~$\tilde{Z} \subseteq Z$ if $\sum_k |w_{kl}| \leq 1$ for all $l$;
(b)~$Z \subseteq \tilde{Z}$ if $\sum_k |w_{kl}| \geq 1$;
(c)~equality iff $\sum_k |w_{kl}| = 1$.
Initializing mixing weights with $\ell_1$ column norms close to 1 yields near-exact generator representation.
\end{proposition}

\paragraph{Score perturbation bound.}

\begin{proposition}[Score perturbation]
\label{prop:score-perturb}
For input perturbation $\|\delta_j\| \leq \rho$, the total output decomposes as
$\tilde{o} - o = M_c W_V \hat{G}\xi + \mathcal{E}$ with $\|\mathcal{E}\| \leq C \rho \max_j \|v_j\|$,
where $C = \|M_c\|_2 \|W_Q s\|_2 \|W_K\|_2 / \sqrt{d_k}$.
\end{proposition}

\subsection{Full Proofs}
\label{app:proofs}

\subsubsection{Proof of Theorem~\ref{thm:bottleneck} (Relational Capacity Bottleneck)}

\begin{proof}
\textbf{Part (a).}
An $H$-head standard attention layer without FFN computes, for query $q$ (e.g., a PMA seed):
\[
 o(S) = \sum_{h=1}^{H} W_o^{(h)} \hat{v}^{(h)}, \quad \hat{v}^{(h)} = \sum_j \alpha_j^{(h)}(S)\, W_V^{(h)} x_j.
\]
The aggregated value $\hat{v}^{(h)}$ depends on $S$ through the scores $\alpha_j^{(h)}$, but the post-aggregation transformation $W_o^{(h)}$ is a fixed matrix.
The effective output transformation is:
\[
 o(S) = \sum_h W_o^{(h)} W_V^{(h)} \Bigl(\sum_j \alpha_j^{(h)}(S)\, x_j\Bigr).
\]
On constant sets $S_t = \{te_i, \ldots, te_i\}$, softmax symmetry forces $\alpha_j^{(h)} = 1/n$ for all $j, h$, giving $o(S_t) = (\sum_h W_o^{(h)} W_V^{(h)})\, te_i$.
The effective transformation $T_{\mathrm{std}} = \sum_h W_o^{(h)} W_V^{(h)}$ is independent of $S_t$, hence relational capacity is zero.

\textbf{Part (b).}
After attention, the FFN receives $h = x + \hat{v}$ and must produce output
$\sum_{l=1}^k \sigma_l(h) \cdot M_l h$ to represent a function with relational complexity $k$.
Each term $\sigma_l(h) \cdot (M_l h)_j$ is the product of a bounded nonlinear scalar and a linear form.
By a standard region-counting argument (whose matching upper-bound construction is given by~\citet{yarotsky2017error}, Proposition~3), a depth-2 ReLU network approximating $f(a,b) = ab$ on a compact domain to accuracy $\epsilon$ requires $\Omega(1/\sqrt{\epsilon})$ neurons; the detailed argument is given in the proof of Theorem~\ref{thm:tdof-depth} (Part~(b), Appendix).
For each relational direction $l$, the FFN must implement $d_k$ such products (one per output coordinate), requiring $\Omega(d_k)$ neurons per direction.
With $d_{\mathrm{ff}}$ total hidden units, at most $\lfloor d_{\mathrm{ff}} / d_k \rfloor$ independent relational directions can be represented.

\textbf{Part (c).}
Direct construction: set $M_c = M_0$, use $L$ generator matrices $M_1, \ldots, M_L$, and gates $\gamma_l = \sigma_l(\mu(S))$.
Then $M(\mu(S)) = M_0 + \sum_{l=1}^L \sigma_l(\mu(S)) M_l = T(S)$, which exactly represents any $f \in \mathcal{F}_L$.
Parameter count: $(L+1)$ matrices of size $d_k \times d_k$ plus $L$ gate weight vectors of size $d_k$.

\textbf{Part (d).}
Let $f(S) = (M_0 + \sum_{l=1}^k \sigma_l(\mu(S)) M_l)\, \mu(S)$ with $\{M_l\}_{l=1}^k$ orthogonal in Frobenius norm ($\langle M_l, M_{l'}\rangle_F = 0$ for $l \neq l'$).
Any mechanism with relational capacity $R$ can represent an output transformation of the form $T_\theta(S) = M_0' + \sum_{l=1}^R \tau_l(S) N_l$ for some learned $\{N_l\}_{l=1}^R$ and scalar functions $\tau_l$.

The optimal $R$-capacity approximation chooses $M_0' = M_0$ and $\{N_l\}_{l=1}^R$ to be the projection of $\{M_l\}_{l=1}^k$ onto the best $R$-dimensional subspace (in the Frobenius metric, weighted by $\mathbb{E}[\sigma_l^2 \|\mu\|^2]$).
By the Eckart--Young--Mirsky theorem applied to the operator-valued approximation, the residual error is:
\[
 \mathbb{E}_S\bigl[\|f(S) - o_\theta(S)\|^2\bigr] \geq \sum_{l=R+1}^{k} \mathbb{E}_S\bigl[\sigma_l(\mu(S))^2 \cdot \|M_l \mu(S)\|^2\bigr],
\]
where we order the terms by $\mathbb{E}[\sigma_l^2 \|M_l \mu\|^2]$ in decreasing order and the mechanism captures the top $R$.
The bound is tight: it is achieved when $\{N_l\}$ aligns with the top-$R$ relational directions.
\end{proof}

\subsubsection{Proof of Theorem~\ref{thm:minimax} (Minimax Parameter Efficiency)}

\begin{proof}
\textbf{Part (a).} The MZ-attention head has parameters $M_c \in \R^{d_k\times d_k}$, $M_1, \ldots, M_L \in \R^{d_k\times d_k}$, and gate weights $w_1, \ldots, w_L \in \R^{d_k}$, totalling $(L+1)d_k^2 + Ld_k = \Theta(Ld_k^2)$. By Theorem~\ref{thm:width}(a), this realises every $f \in \mathcal{F}_{L,B}$ exactly.

\textbf{Part (b).}
We use a metric entropy argument.
The function class $\mathcal{F}_{L,B}$ is parameterized by $M_0, M_1, \ldots, M_L \in \R^{d_k \times d_k}$ with $\|M_l\|_F \leq B$, and Lipschitz gates $\sigma_l$.
We lower-bound the packing number by considering the subclass with fixed $M_0 = I$, fixed $\sigma_l = \mathrm{id}$ (identity, clipped to $[-1,1]$), and varying $M_1, \ldots, M_L$.

For two functions $f, f'$ with generator matrices $\{M_l\}$ and $\{M_l'\}$:
\[
 \sup_{\|x\| \leq R} |f(S_x) - f'(S_x)| \geq \sup_{\|x\| \leq R} \|\sum_l \sigma_l(x) (M_l - M_l') x\|.
\]
For the specific input $S_x = \{x, \ldots, x\}$ with $\mu = x$, and choosing $x = Re_j / \|e_j\|$ for the coordinate maximizing $\|(M_l - M_l') e_j\|$:
\[
 \|f - f'\|_\infty \geq c \cdot R \cdot \max_l \|M_l - M_l'\|_F
\]
for a universal constant $c > 0$ depending only on the dimension.

The $\delta$-packing number of $\{M \in \R^{d_k \times d_k} : \|M\|_F \leq B\}$ in Frobenius norm is at least $(B/\delta)^{d_k^2}$ (volume comparison in $\R^{d_k^2}$).
With $L$ independent generator matrices, the packing number of $\mathcal{F}_{L,B}$ at resolution $\epsilon$ is at least:
\[
 \mathcal{M}(\epsilon, \mathcal{F}_{L,B}) \geq \bigl(cBR/\epsilon\bigr)^{Ld_k^2}.
\]

Any model class parameterized by $\theta \in \R^P$ that $\epsilon$-approximates all of $\mathcal{F}_{L,B}$ must distinguish at least $\mathcal{M}(\epsilon)$ functions.
By a standard argument~\citep{yang1999information, devore1998nonlinear}, the parameter space $\R^P$ can cover at most $\exp(C P \log(1/\epsilon))$ $\epsilon$-cells (where $C$ depends on the model's Lipschitz properties in $\theta$).
Matching: $P \log(1/\epsilon) \geq \Omega(Ld_k^2 \log(BR/\epsilon))$, giving $P \geq \Omega(Ld_k^2 \log(BR/\epsilon) / \log(1/\epsilon))$.
Since $\log(BR/\epsilon) / \log(1/\epsilon) \geq 1$ for $BR \geq 1$, we obtain $P \geq \Omega(Ld_k^2)$.
Including the logarithmic precision factor: $P \geq \Omega(Ld_k^2 \log(BR/\epsilon))$.

\textbf{Part (c).}
MZ-attention uses $P_{\mathrm{MZ}} = (L+1)d_k^2 + Ld_k = \Theta(Ld_k^2)$ parameters and achieves $\epsilon = 0$ (exact representation).
The lower bound requires $P \geq \Omega(Ld_k^2)$ for any constant $\epsilon$.
Hence MZ-attention is order-optimal up to logarithmic factors.
\end{proof}

\subsubsection{Proof of Theorem~\ref{thm:separation} (Expressiveness Separation)}

We restate and prove both parts in full detail; the construction matches the inline proof in Section~\ref{app:additional-theory} and uses the signed-generator readout variant of Section~\ref{sec:architecture}.

\begin{proof}[Proof of part (a)]
We configure a single MZ-attention head as follows. Set the query projection $W_Q = 0 \in \R^{d_k \times d}$, so $\alpha_{ij} = \softmax_j(0) = 1/n$ for all $i,j$, yielding uniform aggregation $\hat c_i = \tfrac{1}{n}\sum_j W_V c_j$. With $W_V = I_d$, $\hat c_i = \mu(S)$.

Set $M_c = I_{d_k}$, $L = 1$, $M_1 = A$, gate weight $w_\gamma = e_1$, and mixing matrix $W_{\mathrm{mix}} = e_1 \in \R^{n_g\times 1}$ (use only the first generator slot). The context-dependent gate is $\gamma_1(\mu) = \tanh(e_1^\top \mu) = \tanh(\mu_1)$. By Eq.~\eqref{eq:mz-main}, the per-token output has center $o^c = M_c \hat c = \mu(S)$ and generator stream whose first row is $o^G[1,:] = \tanh(\mu_1)\,A\,\mu(S)$ with the remaining rows zero.

After PMA pooling (uniform aggregation by symmetry), the pooled center is $\mu(S)$ and the pooled signed first-generator is $\tanh(\mu_1)\,A\,\mu(S)$. Applying the signed-generator readout $W_{\mathrm{out}}: (c, g_1) \mapsto c + g_1$ yields $\mu(S) + \tanh(\mu_1)\,A\,\mu(S) = (I + \tanh(\mu_1)\,A)\,\mu(S) = f_A(S)$, exactly. Every operation in the construction is either a fixed linear map or the native MZ gate, both computed without finite-width truncation; no UAT-based ReLU approximation is invoked.
\end{proof}

\begin{proof}[Proof of part (b)]
An $H$-head standard attention layer without FFN computes, for a query $q$ (e.g., a PMA seed vector $s$):
\[
 o(S) = W_o \operatorname{concat}\!\bigl(W_V^{(1)} \hat{v}^{(1)},\, \ldots,\, W_V^{(H)} \hat{v}^{(H)}\bigr)
\]
where $\hat{v}^{(h)} = \sum_j \alpha_j^{(h)}(S)\, x_j$ and $\alpha_j^{(h)} = \softmax_j(s^{(h)\top} W_Q^{(h)\top} W_K^{(h)} x_j / \sqrt{d_k})$.

Step 1.
Consider the subfamily of constant sets $S_t = \{te_i, \ldots, te_i\}$ ($n$ copies of $te_i$) for a basis vector $e_i$ and scalar $t \in \R$.
By symmetry of the softmax with identical inputs, $\alpha_j^{(h)} = 1/n$ for all $j$ and $h$.
Hence $\hat{v}^{(h)} = W_V^{(h)} (te_i) = t W_V^{(h)} e_i$, and:
\[
 o(S_t) = W_o \operatorname{concat}(t W_V^{(1)} e_i,\, \ldots,\, t W_V^{(H)} e_i) = t \cdot W_o \operatorname{concat}(W_V^{(1)} e_i,\, \ldots,\, W_V^{(H)} e_i) =: t \cdot \tilde{w}_i.
\]
This is a linear function of $t$.

Step 2.
On the same inputs, $f_A(S_t) = (I + \tanh(t \delta_{i1}) A)(te_i) = te_i + t\tanh(t \delta_{i1}) Ae_i$, where $\delta_{i1} = [e_i]_1$.
For $i = 1$: $f_A(S_t) = te_1 + t\tanh(t) Ae_1$.

If $Ae_1 \neq 0$, the map $t \mapsto te_1 + t\tanh(t) Ae_1$ is nonlinear (since $t\tanh(t)$ is nonlinear and $Ae_1 \neq 0$ ensures this nonlinearity survives in at least one coordinate).
A linear function $t \mapsto t\tilde{w}_1$ cannot agree with this nonlinear function on any open interval.

Step 3.
If $Ae_1 = 0$, since $A \neq 0$ there exists $e_j$ with $Ae_j \neq 0$.
If $j = 1$, we already have a contradiction.
If $j \neq 1$, use $S_t = \{te_j + \epsilon e_1, \ldots, te_j + \epsilon e_1\}$ for small $\epsilon > 0$:
$\mu(S_t) = te_j + \epsilon e_1$, so $\tanh(\mu_1) = \tanh(\epsilon) \neq 0$, and $A\mu = tAe_j + \epsilon Ae_1$.
Since $Ae_j \neq 0$, the nonlinearity persists.
\end{proof}

\subsubsection{Proof of Theorem~\ref{thm:width} (Width Efficiency)}

\begin{proof}
Part (a) is immediate: the MZ parameters $M_c \in \R^{d_k \times d_k}$, $M_1, \ldots, M_L \in \R^{d_k \times d_k}$, and gate weights $w_1, \ldots, w_L \in \R^{d_k}$ total $(L+1)d_k^2 + Ld_k$.

For part (b), consider a single bilinear term $\sigma_l(h) \cdot (M_l h)_j$ that the FFN must compute, with $a = \sigma_l(h) \in [-1,1]$ and $b = (M_l h)_j \in [-B_l, B_l]$, $B_l = \|M_l\|_2 R$. The FFN must approximate $f(a, b) = ab$ on this rectangle to sup-norm error $\epsilon$.

By a standard region-counting argument for shallow ReLU approximation of strictly-convex bilinear functions (the matching upper-bound construction is~\citet{yarotsky2017error}, Proposition~3), any depth-2 ReLU network approximating $ab$ on $[-1,1] \times [-B_l, B_l]$ to sup-norm error $\epsilon$ requires $n = \Omega(B_l / \sqrt{\epsilon})$ hidden units. The standard underlying argument bounds a depth-2 ReLU network's piecewise-linear regions by $O(n^2)$ on $\R^2$ and applies the strict convexity of $ab$ ($\partial^2 f / \partial a \partial b = 1$): on a region of diameter $r$ the best affine approximant errs by $\Omega(r^2)$, and an $O(n^2)$-region partition of a $[-B_l, B_l]\times[-1,1]$ rectangle has at least one region of diameter $\Omega(B_l / n)$, giving worst-case error $\Omega(B_l^2 / n^2)$ and hence $n = \Omega(B_l/\sqrt{\epsilon})$.

Summing over $L$ generators and $d_k$ output coordinates: $d_{\mathrm{ff}} \geq \sum_l d_k \cdot \Omega(B_l / \sqrt{\epsilon}) = \Omega(L d_k B / \sqrt{\epsilon})$ where $B = \max_l B_l$.
For unit-norm inputs ($B = O(1)$): $d_{\mathrm{ff}} = \Omega(L d_k / \sqrt{\epsilon})$.
\end{proof}

\subsubsection{Proof of Theorem~\ref{thm:sensitivity} (Directional Sensitivity)}

\begin{proof}
(a) Standard attention output: $o = W_o \hat{v}$ with $\hat{v} = \sum_j \alpha_{ij} v_j$.
Under perturbation $\tilde{\hat{v}} = \hat{v} + \hat{G}\xi$:
$\|\Delta o\| = \|W_o \hat{G}\xi\| \leq \|W_o\|_2 \|\hat{G}\xi\| \leq \|W_o\|_2 \sum_k \|\hat{G}_k\| \cdot |\xi_k| \leq \|W_o\|_2 \|\hat{G}\|_{1,2}$.
This is a scalar (norm) bound: all directions in $\R^{d_k}$ are treated equally.

(b) MZ-attention output: $\Delta o = M_c \hat{G}\xi$.
As $\xi$ ranges over $[-1,1]^{n_g}$, $\Delta o$ traces the zonotope $\mathcal{Z}_\Delta = \{M_c \hat{G}\xi : \|\xi\|_\infty \leq 1\}$ with generators $\{M_c \hat{G}_k\}_{k=1}^{n_g}$.
The vertex count $2\binom{n_g}{d_k-1}$ is a standard result from zonotope geometry~\citep{Althoff2010PhD}.

(c) Every point in $\mathcal{Z}_\Delta$ satisfies $\|z\| \leq \sum_k \|M_c \hat{G}_k\|_2 \leq \|M_c\|_2 \|\hat{G}\|_{1,2}$, so $\mathcal{Z}_\Delta \subseteq B(\|M_c\|_2 \|\hat{G}\|_{1,2})$.
For directions $u \in \R^{d_k}$ orthogonal to the column span $\mathrm{span}\{M_c \hat{G}_k\}_{k=1}^{n_g}$, the support function $h_{\mathcal{Z}_\Delta}(u) = \sum_k |u^\top M_c \hat{G}_k| = 0 < \|M_c\|_2 \|\hat{G}\|_{1,2} = h_B(u)$, so the MZ certificate is strictly tighter than the ball whenever the column span is a proper subspace.
The volume of a zonotope with generators $g_1, \ldots, g_{n_g} \in \R^{d_k}$ is~\citep{Althoff2010PhD}
\[
\mathrm{vol}(\mathcal{Z}_\Delta) = 2^{d_k} \sum_{|I| = d_k} |\det([g_i]_{i \in I})|,
\]
which is strictly smaller than $\mathrm{vol}(B)$ for $d_k \geq 2$ since the zonotope is a proper inscribed subset of the ball; for typical trained generators with unequal norms and non-orthogonal directions, $\mathrm{vol}(\mathcal{Z}_\Delta)$ is empirically several orders of magnitude smaller than $\mathrm{vol}(B)$, reflecting the anisotropic structure of the model's sensitivity.
\end{proof}

\subsubsection{Proof of Proposition~\ref{prop:mixing-app} (Generator Mixing)}

\begin{proof}
Let $Z_{n_g+L}$ have generators $\{g_1, \ldots, g_{n_g}, g'_1, \ldots, g'_L\}$ in general position.
A point in $Z_{n_g+L}$ is $p = c + \sum_{k=1}^{n_g} g_k \zeta_k + \sum_{l=1}^{L} g'_l \zeta'_l$ with $|\zeta_k|, |\zeta'_l| \leq 1$.
A point in $\tilde{Z}_{n_g}$ is $\tilde{p} = c + \sum_{k=1}^{n_g} (g_k + \sum_l w_{kl} g'_l) \tilde{\zeta}_k$ with $|\tilde{\zeta}_k| \leq 1$.

Expanding: $\tilde{p} = c + \sum_k g_k \tilde{\zeta}_k + \sum_l (\sum_k w_{kl} \tilde{\zeta}_k) g'_l$.

(a) $\tilde{Z}_{n_g} \subseteq Z_{n_g+L}$:
Set $\zeta_k = \tilde{\zeta}_k$ and $\zeta'_l = \sum_k w_{kl} \tilde{\zeta}_k$.
We need $|\zeta'_l| = |\sum_k w_{kl} \tilde{\zeta}_k| \leq \sum_k |w_{kl}| \cdot |\tilde{\zeta}_k| \leq \sum_k |w_{kl}|$.
If $\sum_k |w_{kl}| \leq 1$ for all $l$, then $|\zeta'_l| \leq 1$ and $\tilde{p} \in Z_{n_g+L}$.

(b) $Z_{n_g+L} \subseteq \tilde{Z}_{n_g}$:
Given a point $p \in Z_{n_g+L}$ with parameters $(\zeta, \zeta')$, we seek $\tilde{\zeta}$ with $|\tilde{\zeta}_k| \leq 1$ matching both $g_k$ and $g'_l$ coefficients.
By general position, coefficient matching requires $\tilde{\zeta}_k = \zeta_k$ (from $g_k$) and $\sum_k w_{kl} \zeta_k = \zeta'_l$ (from $g'_l$).
For the latter to hold for all $\zeta'_l \in [-1,1]$, the range of $\sum_k w_{kl} \zeta_k$ (over $|\zeta_k| \leq 1$) must contain $[-1,1]$.
This range is $[-\sum_k |w_{kl}|, \sum_k |w_{kl}|]$, so containment requires $\sum_k |w_{kl}| \geq 1$.

(c) Combining (a) and (b): equality holds when $\sum_k |w_{kl}| = 1$ for all $l$.
\end{proof}

\subsubsection{Proof of Theorem~\ref{thm:ihw-calibration} (IHW Uncertainty Bound)}

\begin{proof}
\textbf{Part (a).}
The output zonotope is $Z_{\mathrm{out}} = \{c_{\mathrm{out}} + G_{\mathrm{out}}\xi : \|\xi\|_\infty \leq 1\}$.
Under a linear readout $W_{\mathrm{out}} \in \R^{d_k}$, the output is:
\[
 y(\xi) = W_{\mathrm{out}}^\top (c_{\mathrm{out}} + G_{\mathrm{out}}\xi) = W_{\mathrm{out}}^\top c_{\mathrm{out}} + \sum_{g=1}^{n_g} (W_{\mathrm{out}}^\top G_g) \xi_g.
\]
This is a scalar zonotope (interval) with center $W_{\mathrm{out}}^\top c_{\mathrm{out}}$ and half-width $\sum_g |W_{\mathrm{out}}^\top G_g|$.
Hence:
\[
 \mathrm{diam}(\mathcal{Y}) = 2\sum_{g=1}^{n_g} |W_{\mathrm{out}}^\top G_g| \leq 2 \sum_g \|W_{\mathrm{out}}\|_2 \|G_g\|_2 \leq 2\|W_{\mathrm{out}}\|_2 \sum_g \|G_g\|_1 = 2\|W_{\mathrm{out}}\|_2 \cdot \mathrm{IHW}.
\]

\textbf{Part (b).}
Since $\xi_1, \ldots, \xi_{n_g}$ are i.i.d.\ uniform on $[-1,1]$:
$\mathbb{E}[\xi_g] = 0$, $\mathbb{E}[\xi_g^2] = 1/3$, and $\mathbb{E}[\xi_g \xi_{g'}] = 0$ for $g \neq g'$.
The variance of $y(\xi) = W_{\mathrm{out}}^\top c_{\mathrm{out}} + \sum_g (W_{\mathrm{out}}^\top G_g) \xi_g$ is:
\[
 \mathrm{Var}[y] = \sum_{g=1}^{n_g} (W_{\mathrm{out}}^\top G_g)^2 \cdot \mathrm{Var}[\xi_g] = \frac{1}{3} \sum_g (W_{\mathrm{out}}^\top G_g)^2.
\]
By the Cauchy--Schwarz inequality applied coordinate-wise:
$(W_{\mathrm{out}}^\top G_g)^2 \leq \|W_{\mathrm{out}}\|_2^2 \|G_g\|_2^2 \leq \|W_{\mathrm{out}}\|_2^2 \|G_g\|_1^2$.
Hence $\sum_g (W_{\mathrm{out}}^\top G_g)^2 \leq \|W_{\mathrm{out}}\|_2^2 \sum_g \|G_g\|_1^2 \leq \|W_{\mathrm{out}}\|_2^2 (\sum_g \|G_g\|_1)^2 = \|W_{\mathrm{out}}\|_2^2 \cdot \mathrm{IHW}^2$,
where the last inequality uses $\sum a_i^2 \leq (\sum a_i)^2$ for $a_i \geq 0$.

\textbf{Part (c).}
From part (a), $\mathrm{diam}(\mathcal{Y}(S)) = 2\|W_{\mathrm{out}}\|_2 \cdot \mathrm{IHW}(Z_{\mathrm{out}}(S))$.
Since $\|W_{\mathrm{out}}\|_2$ is shared between inputs (it is a model parameter, not input-dependent), the monotonicity in IHW directly implies monotonicity in the reachable diameter.
\end{proof}

\subsubsection{Proof of Theorem~\ref{thm:tdof-depth} (Depth Separation)}
\label{app:proof-tdof}

\begin{proof}
We prove the depth lower bound on the input distribution where $\mu(S) \in \R^d$ has i.i.d.\ standard-Gaussian coordinates, using a compositional-orthogonality argument in $L^2(\mu)$ together with the explicit worst-case construction of Lemma~\ref{lem:comp-orth}.

\textbf{Setup.} Let $f(S) = (\sum_{l=1}^k \sigma_l(\mu(S))\, M_l)\,\mu(S)$ be the witness of Lemma~\ref{lem:comp-orth}: $\sigma_l(\mu) := \tanh(\mu_l)$, $M_l := e_l e_d^\top$ for $l=1,\dots,k$ with $k \le d-1$, and $M_0 = 0$. Each $M_l$ is rank-1 with Frobenius-orthonormal $\{M_l\}_{l=1}^k$ (Lemma~\ref{lem:comp-orth}(a)). The choice $M_0 = 0$ ensures the network's linear path $\mu \mapsto M_0 \mu$ contributes no free target-direction component on $\mathcal{V}$, so the depth lower bound is determined by the FFN multiplicative budget alone. We view $f$ as an element of the tensor-product space $L^2(\mu) \otimes \R^{d \times d}$ with inner product $\langle a\otimes A, b\otimes B\rangle := \mathbb{E}_\mu[a\,b]\,\langle A,B\rangle_F$. The target directions $\{\sigma_l \otimes M_l\}_{l=1}^k$ are mutually orthogonal in this inner product (independent $\mu_l$'s make $\mathbb{E}[\sigma_l\sigma_{l'}] = 0$ for $l \ne l'$, and Frobenius orthonormality handles the matrix factor), so they span a $k$-dimensional subspace $\mathcal{V}$.

\textbf{Layer-by-layer FFN bilinear count.} A $D$-layer standard attention network with residual connections and depth-$2$ ReLU FFN of width $d_{\mathrm{ff}}$ computes
\[
 z^{(\ell+1)} = z^{(\ell)} + \mathrm{FFN}^{(\ell)}\bigl(z^{(\ell)} + \hat v^{(\ell)}\bigr),\quad z^{(0)}=\mu.
\]
A bilinear term in $\mathcal{V}$ has the form $\sigma_l(\mu)\otimes M_l$, requiring the FFN to multiply a gate $\sigma_l$ by the linear coordinate map $\mu \mapsto M_l \mu$. By Theorem~\ref{thm:bottleneck}(b), each multiplicative bilinear coordinate of this kind requires $\Omega(d_k)$ ReLU neurons (Yarotsky bilinear lower bound), so a single layer's FFN of width $d_{\mathrm{ff}}$ contributes at most $\lfloor d_{\mathrm{ff}}/d_k\rfloor$ new directions that lie in $\mathcal{V}$.

\textbf{Cross-layer compositions vanish on the matrix side.} After $\ell$ layers, the residual representation can additionally contain monomial composition terms of the form $P(\sigma)\otimes(M_{l_1}\cdots M_{l_r})$ where $r\ge 2$ comes from successive FFNs reading earlier-layer bilinear outputs and re-multiplying them through linear maps. By Lemma~\ref{lem:comp-orth}(b)--(c), under the construction $M_l = e_l e_d^\top$ every such matrix product with $r\ge 2$ is identically zero, so the entire term vanishes in the tensor-product space and therefore has zero projection on $\mathcal{V}$. Hence cross-layer composition cannot bootstrap any new target direction; only the layer-wise FFN bilinear count contributes.

\textbf{Counting.} After $D$ layers the network has captured at most $\Lambda_D \le D\cdot\lfloor d_{\mathrm{ff}}/d_k\rfloor$ of the $k$ target directions in $\mathcal{V}$. The remaining $k-\Lambda_D$ directions contribute an irreducible $L^2(\mu)$ error (by the Eckart--Young residual of Theorem~\ref{thm:bottleneck}(d)) bounded below by $\Omega((k-\Lambda_D)/k)\cdot\mathrm{Var}(f)$. For exact representation we need $\Lambda_D = k$, giving $D \ge k\,d_k/d_{\mathrm{ff}}$. For approximate representation in $L^2(\mu)$ with residual at most $O(\mathrm{Var}(f)/L)$, achieving this residual requires $k - \Lambda_D = O(k/L)$ uncaptured directions, hence $\Lambda_D \ge k(1 - O(1/L))$, giving $D \ge \Omega(k\,d_k/d_{\mathrm{ff}})$ up to constants since each missing direction contributes $\Theta(\mathrm{Var}(f)/k)$.

\textbf{Single MZ-attention layer with signed-generator readout.} Setting $M_c = M_0 = 0$ and using the $k$ generators $M_1,\ldots,M_k$ of $f$ with gates $\gamma_l(\mu)=\sigma_l(\mu)$, Eq.~\eqref{eq:mz-main} produces a per-token output whose center is $0$ and whose generator stream contains $\gamma_l(\mu)\, M_l\, \mu$ as the $l$-th signed generator slot. After PMA pooling (uniform aggregation, as in the construction of Theorem~\ref{thm:separation}), the signed-generator readout $W_{\mathrm{out}}: (c, g_1, \ldots, g_k) \mapsto \sum_{l=1}^k g_l$ recovers $f(S) = \sum_l \sigma_l(\mu) M_l \mu$ exactly. The IHW readout (Section~\ref{sec:architecture}) gives only $\sum_l |\sigma_l(\mu)| \cdot \|M_l \mu\|$-bounded scalar information rather than the signed combination, so it does not extract $f(S)$ exactly; the expressiveness statement is for the signed-readout variant.
\end{proof}

\subsubsection{Proof of Proposition~\ref{prop:perm} (Permutation Equivariance)}

\begin{proof}
Let $\pi$ be a permutation of $\{1, \ldots, n\}$ and denote the permuted set as $\pi(S) = \{x_{\pi(1)}, \ldots, x_{\pi(n)}\}$.

MZ-MAB equivariance:
The attention scores satisfy $\alpha_{i, \pi(j)}(\pi(S)) = \alpha_{ij}(S)$ because they depend on pairwise dot products $q_i^\top k_j$ which are permutation-equivariant.
The aggregation $\hat{c}_i = \sum_j \alpha_{ij} v_j$ is thus equivariant: permuting the set permutes the query outputs.
The MZ transform (Eq.~\ref{eq:mz-main}) operates on each query independently.
The gate $\gamma = \tanh(W_\gamma \bar{c})$ depends on the set mean $\bar{c} = \frac{1}{n}\sum_j c_j$, which is permutation-invariant.
Hence MZ-MAB is equivariant: $\text{MZ-MAB}(\pi(Q), \pi(KV)) = \pi(\text{MZ-MAB}(Q, KV))$.

MZ-SAB inherits equivariance since $Q = KV = S$.

Permutation invariance of the full model:
MZ-PMA uses learnable seed vectors $s_1, \ldots, s_k$ as queries, which are independent of the input set.
Thus $\text{MZ-PMA}(\pi(S)) = \text{MZ-PMA}(S)$, and the output is permutation-invariant.
\end{proof}

\section{Extended Related Work}
\label{app:related-extended}

\paragraph{Zonotopes as external verification tools vs.\ as internal representations.}
Zonotopes have a long history in formal verification and abstract interpretation of neural networks~\citep{bonaert2021fast, mirman2018differentiable, jordan2022zonotope, chung2025provably}, where they are computed post-hoc around a trained model's activations to certify robustness, and in control for data-driven reachability~\citep{alanwar2021data, Alanwar2023Datadriven, kochdumper2023constrained}.
To our knowledge \methodname{} is the first architecture to make the matrix-zonotope an internal, trainable representation of the value-projection family: center and generator matrices are jointly optimised by gradient descent, and the zonotope structure participates in every forward pass rather than being wrapped around a frozen model. Concurrent and prior work has explored related ideas in adjacent spaces (e.g.\ structured low-rank reparameterisations and gated value updates), but we are not aware of an internal trainable matrix-zonotope value family with the gating, mixing, and budget-control structure of Eq.~\eqref{eq:mz-main}.

\paragraph{Set-input networks.}
Deep Sets~\citep{zaheer2017deep}, PointNet~\citep{qi2017pointnet}, PointNet++~\citep{qi2017pointnetpp}, and Set Transformer~\citep{lee2019set} established permutation invariance and attention-based pooling as standard tools for set inputs; PointNet++ in particular adds hierarchical local-region feature aggregation but still composes fixed per-point MLPs with set pooling, so its value transformation remains context-rigid in our sense.
All of these architectures express relationships between elements via scalar scores (softmax weights or learned aggregations) that multiply fixed linear projections of inputs; \methodname{} generalizes this to matrix-valued context-adaptive operators, which is precisely what high-TDOF tasks require.

\paragraph{Context-adaptive linear maps.}
Hypernetworks~\citep{ha2017hypernetworks, hyper2025attention} and dynamic filter networks~\citep{jia2016dynamic} generate full $d_k^2$-parameter weight updates per sample, giving unlimited expressiveness but paying a quadratic cost that makes the effective rank hard to control.
FiLM~\citep{perez2018film} restricts to diagonal (scale-and-shift) modulation, which is provably insufficient for off-diagonal structure such as rotations or covariance-dependent maps.
Mixture-of-Experts~\citep{shazeer2017outrageously} routes between a discrete set of fixed experts and is therefore piecewise-constant in context.
The matrix-zonotope parameterization of \methodname{} occupies a middle ground: $O(Ld_k^2)$ parameters with $L$ continuous gates realize a convex, rank-controlled family of linear operators that is minimax-optimal for context-adaptive linear maps (Theorem~\ref{thm:minimax}) while retaining the geometric and uncertainty structure of the zonotope family.

\paragraph{Equivariant architectures.}
EGNN~\citep{satorras2021egnn}, SchNet~\citep{schutt2017schnet}, and similar E($n$)- or SO($n$)-equivariant architectures achieve very strong performance on rotation- and translation-invariant targets by hard-coding the symmetry into the forward pass.
On such tasks \methodname{} is at a structural disadvantage (it does not bake in the symmetry) and our experiments confirm this.
The complementary regime is tasks whose target depends on a specific coordinate frame, e.g.\ the quadratic TDOF task of Section~\ref{sec:tdof}, where equivariance is actively harmful: invariant outputs cannot distinguish the target from its rotations.
\methodname{} is the right tool when the task's symmetry is unknown a priori and cannot be built in by construction.

\paragraph{Uncertainty quantification.}
MC-Dropout~\citep{gal2016dropout} and Deep Ensembles~\citep{lakshminarayanan2017simple} obtain predictive uncertainty by multiple stochastic forward passes; \methodname{}'s zonotope geometry (Theorem~\ref{thm:ihw-calibration}) yields a single-pass uncertainty estimate via the interval-hull width of the output zonotope, at $3.7\times$ lower inference cost than the baselines (Appendix~\ref{app:uncertainty-generalization}, Table~\ref{tab:unc-comparison}).

\section{Additional Experiments}
\label{app:additional-exp}

This section reports the full battery of additional experiments referenced from the main body. It is organised into five thematic blocks: \textbf{D.1} architectural ablations isolating which design choice drives the gain; \textbf{D.2} robustness and generalisation studies; \textbf{D.3} uncertainty calibration; \textbf{D.4} high-dimensional and large-scale scaling on computational-geometry targets; and \textbf{D.5} real-world transfer studies plus the equivariant baseline comparison.

\subsection{Architecture and ablation studies}
\label{app:architecture-ablations}

Studies that isolate which architectural choice in \methodname{} (generator count, generator-info routing, low-rank update form, depth) drives the empirical advantage. They support the conclusion of Section~\ref{sec:analysis}: the centre-plus-gated-generator structure, not raw capacity, is what drives the gain.

\begin{figure}[h]
 \centering
 \includegraphics[width=\linewidth]{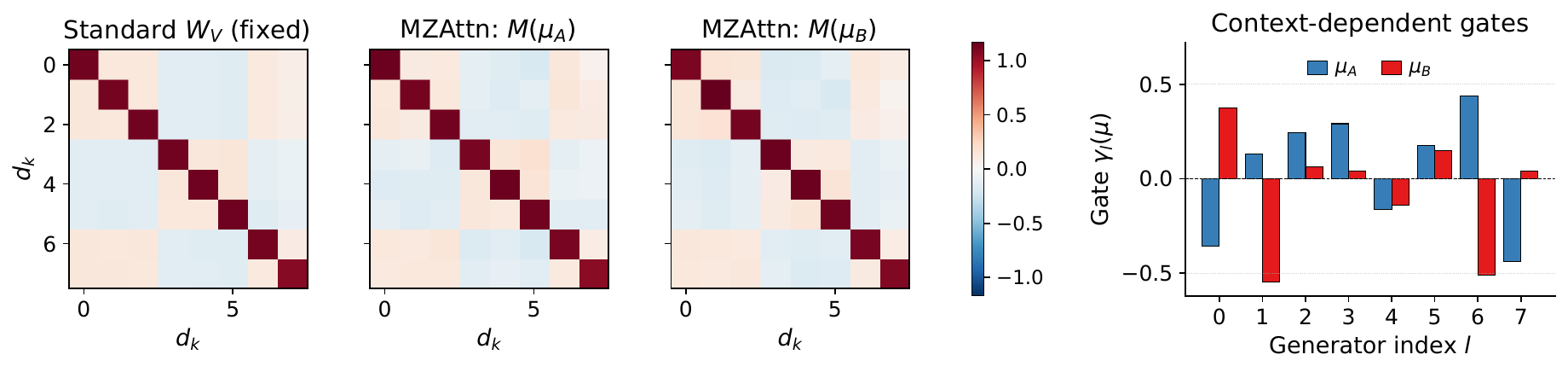}
 \caption{\textbf{Effective value-transformation operator (head 0) on the quadratic-TDOF task.} Standard attention uses one fixed $W_V$ regardless of context (left); \methodname{} produces a context-specific $M(\mu) = M_c + \sum_l \gamma_l(\mu)\, M_l$ for each input set (centre two panels). The context-dependent gates $\gamma_l(\mu)$ (right) drive the adaptation; the resulting Frobenius-norm change is $\|M(\mu_A) - M(\mu_B)\|_F / \|M_c\|_F = 8.0\%$.}
 \label{fig:attn-viz}
\end{figure}

\paragraph{Per-sample operator adaptation is non-trivial.}
Visualising $M(\mu_A), M(\mu_B)$ on two contexts at head~$0$ (Figure~\ref{fig:attn-viz}), the relative operator change is $\|M(\mu_A) - M(\mu_B)\|_F / \|M_c\|_F = 8.0\%$ and the eight gates $\gamma_l(\mu)$ flip sign on five of eight generators between contexts, ruling out gate collapse to constants.

\begin{figure}[h]
 \centering
 \includegraphics[width=\linewidth]{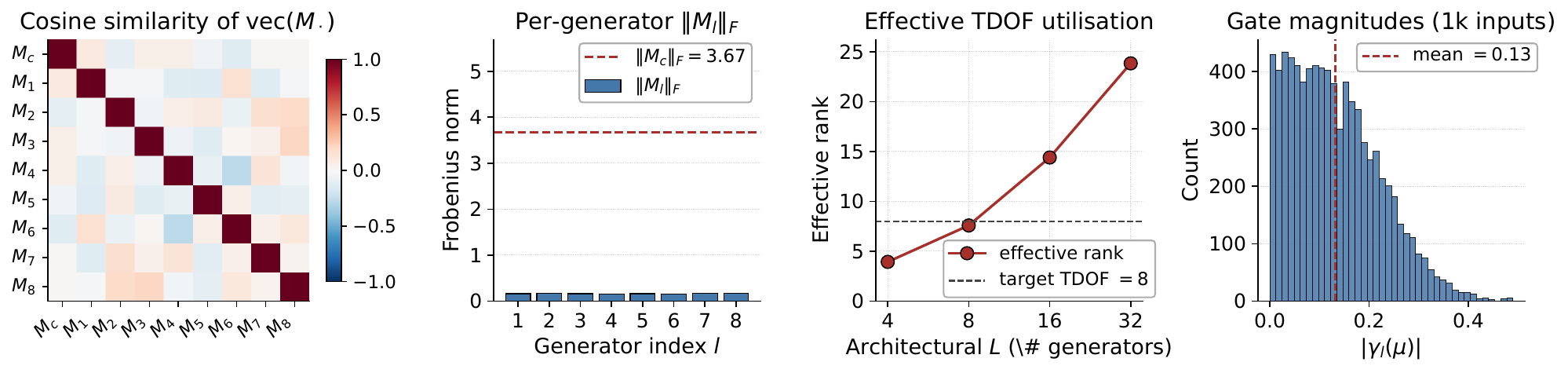}
 \caption{\textbf{Mechanistic anatomy of a trained \methodname{} layer ($L{=}8$, head 0, $1{,}000$ probe contexts), from left to right.} \textbf{Cosine similarity} of $\operatorname{vec}(M_c)$ and $\operatorname{vec}(M_l)$: generator directions are near-orthogonal (mean off-diagonal $|\rho|<0.01$). \textbf{Per-generator Frobenius norms} cluster in $[0.146, 0.169]$ (CV $5.1\%$), each $\sim 1/22$ of $\|M_c\|_F$. \textbf{Effective rank} of $\{M(\mu_i) - \bar M\}_i$ as architectural budget $L$ varies on a fixed TDOF-$8$ task: saturates the budget when $L \leq \mathrm{TDOF}$, grows sublinearly above. \textbf{Gate magnitudes} $|\gamma_l(\mu)|$ over $L \times 1000$ evaluations: $78.6\%$ exceed $0.05$, mean $0.13$, p95 $0.29$.}
 \label{fig:mz-analysis}
\end{figure}

\paragraph{Trained generator basis + TDOF utilisation.}
Probing a trained $L{=}8$ checkpoint with $1{,}000$ test contexts (Figure~\ref{fig:mz-analysis}, cosine, Frobenius-norm, and gate panels): generators are near-orthogonal in operator space, comparably sized, and gates active on $78.6\%$ of evaluations --- no dead generators or constant-bias collapse. Sweeping $L \in \{4, 8, 16, 32\}$ at target TDOF $=8$, the soft effective rank of $200$ extracted operators saturates the budget when $L \leq \mathrm{TDOF}$ and grows sublinearly above (effective-rank panel). The trained model allocates operator-space variation up to but not beyond the intrinsic TDOF, with $L$ a soft upper bound.

\subsubsection{MZ Generator Budget}

\begin{table}[h]
 \centering
 \caption{Ablation on MZ generator matrices $n_{\mathrm{mz}}$ (mean$\pm$std, 3 seeds). Per-seed JSONs for all $n_{\mathrm{mz}} \in \{1, 2, 4, 8, 16, 32\}$ are in \texttt{gen\_ablation.json}.}
 \label{tab:gen-ablation}
 \small
 \setlength{\tabcolsep}{4pt}
 \begin{tabular}{@{}ccccc@{}}
 \toprule
 & \multicolumn{2}{c}{$L = 16$} & \multicolumn{2}{c}{$L = 32$} \\
 \cmidrule(lr){2-3}\cmidrule(lr){4-5}
 $n_{\mathrm{mz}}$ & $R^2$ & Params & $R^2$ & Params \\
 \midrule
 1 & $\mathbf{0.978 \pm 0.004}$ & 247K & $0.937 \pm 0.014$ & 247K \\
 2 & $0.976 \pm 0.012$ & 252K & $0.959 \pm 0.019$ & 252K \\
 4 & $0.956 \pm 0.011$ & 260K & $\mathbf{0.961 \pm 0.007}$ & 260K \\
 8 & $0.960 \pm 0.011$ & 278K & $0.949 \pm 0.021$ & 278K \\
 16 & $0.948 \pm 0.016$ & 313K & $0.923 \pm 0.011$ & 313K \\
 32 & $0.943 \pm 0.002$ & 383K & $0.907 \pm 0.010$ & 383K \\
 \bottomrule
 \end{tabular}
\end{table}

Even $n_{\mathrm{mz}} = 1$ exceeds deep standard Set Transformers.

\paragraph{Practical guidelines for $L$ and $n_g$.}
The two main hyperparameters of \methodname{} are the number of MZ generator matrices $L$ (controlling relational capacity) and the zonotope generator budget $n_g$ (controlling per-token uncertainty resolution).
Based on the ablation studies in Tables~\ref{tab:ablation} and~\ref{tab:gen-ablation}, we recommend:
\begin{itemize}
 \item \textbf{MZ generators $L$:} Start with $L = 2$--$4$.
 Even $L = 1$ substantially outperforms standard attention on high-TDOF tasks (Table~\ref{tab:gen-ablation}).
 Increasing $L$ beyond 4 adds parameters ($L d_k^2$ per head per layer) with diminishing returns unless the task's TDOF is known to be high.
 For tasks where the TDOF can be estimated (e.g., $d$ for Chamfer matching, $d(d{+}1)/2$ for Mahalanobis-type operators), setting $L$ to match the estimated TDOF is theoretically motivated by Theorem~\ref{thm:minimax}.
 \item \textbf{Zonotope generators $n_g$:} Start with $n_g = 4$--$8$.
 Table~\ref{tab:ablation} shows that $n_g = 4$ peaks for set matching; larger budgets overfit.
 Since $n_g$ primarily controls the resolution of the uncertainty envelope rather than the relational capacity, moderate values suffice; the output head aggregates generators via interval hull width regardless of $n_g$.
 \item \textbf{Scaling regime:} When computational budget is constrained, prioritize $L$ over $n_g$: the MZ generators directly determine the model's relational capacity (Theorem~\ref{thm:bottleneck}(c)), whereas $n_g$ affects uncertainty granularity, which is less critical for prediction quality.
\end{itemize}

\subsubsection{Generator Ablation}

\begin{table}[h]
 \centering
 \small
 \setlength{\tabcolsep}{4pt}
 \caption{Ablation on zonotope generators $n_g$ (set matching, seed 42).}
 \label{tab:ablation}
 \begin{tabular}{@{}cccc@{}}
 \toprule
 $n_g$ & MSE $\downarrow$ & MAE $\downarrow$ & $R^2$ $\uparrow$ \\
 \midrule
 2 & 45.80 & 3.50 & 0.967 \\
 \textbf{4} & \textbf{26.74} & \textbf{3.33} & \textbf{0.981} \\
 8 & 32.46 & 3.39 & 0.977 \\
 16 & 55.05 & 4.47 & 0.961 \\
 \bottomrule
 \end{tabular}
\end{table}

Performance peaks at $n_g = 4$ and degrades for larger budgets due to overfitting.

\subsubsection{Generator Info in Q/K Scoring}
\label{app:genqk}

\begin{table}[h]
 \centering
 \small
 \setlength{\tabcolsep}{4pt}
 \caption{Ablation: incorporating generator information (IHW) into Q/K scoring. MZ-GenQK adds projections $W_q^{\mathrm{gen}}, W_k^{\mathrm{gen}}$ of per-token interval hull widths into the attention score computation.}
 \label{tab:genqk}
 \begin{tabular}{@{}lccccr@{}}
 \toprule
 & \multicolumn{2}{c}{Task A: Set Regr.} & \multicolumn{2}{c}{Task B: Set Match.} & \\
 \cmidrule(lr){2-3} \cmidrule(lr){4-5}
 Model & MSE $\downarrow$ & $R^2$ $\uparrow$ & MSE $\downarrow$ & $R^2$ $\uparrow$ & Params \\
 \midrule
 \methodname{} (centers-only Q/K) & 2.19{\tiny$\pm$0.09} & 0.879{\tiny$\pm$0.005} & \textbf{55.0}{\tiny$\pm$10.5} & \textbf{0.961}{\tiny$\pm$0.008} & 365K \\
 MZ-GenQK (gen info in Q/K) & \textbf{1.64}{\tiny$\pm$0.13} & \textbf{0.909}{\tiny$\pm$0.007} & 68.8{\tiny$\pm$5.6} & 0.951{\tiny$\pm$0.004} & 415K \\
 \bottomrule
 \end{tabular}
\end{table}

Incorporating generator information into Q/K scoring yields task-dependent results: on Task~A, MZ-GenQK improves $R^2$ from $0.879$ to $0.909$ ($+0.030$ absolute, $+3.4\%$ relative), but on Task~B it degrades MSE from $55.0$ to $68.8$ ($+25\%$). The mixed outcome, combined with the additional ${\sim}50$K parameters, indicates that centers-only scoring is a reasonable default. The benefit may arise when the prediction depends on generator structure (as in regression), while for matching tasks the center already encodes sufficient information for routing attention.

\subsubsection{HyperNet Rank Ablation}
\label{app:hypernet-rank}

\begin{table}[h]
 \centering
 \small
 \setlength{\tabcolsep}{4pt}
 \caption{HyperNet rank ablation (mean $\pm$ std over 3 seeds). All models are parameter-matched at ${\sim}$360--370K. Increasing the low-rank update dimension does not close the gap to \methodname{}; on Task~B, performance degrades monotonically.}
 \label{tab:hypernet-rank}
 \begin{tabular}{@{}lccccr@{}}
 \toprule
 & \multicolumn{2}{c}{Task A: Set Regr.} & \multicolumn{2}{c}{Task B: Set Match.} & \\
 \cmidrule(lr){2-3} \cmidrule(lr){4-5}
 Model & MSE $\downarrow$ & $R^2$ $\uparrow$ & MSE $\downarrow$ & $R^2$ $\uparrow$ & Params \\
 \midrule
 HyperNet ($r{=}8$) & 3.51{\tiny$\pm$0.01} & 0.806{\tiny$\pm$0.001} & 76.1{\tiny$\pm$5.7} & 0.945{\tiny$\pm$0.004} & 360K \\
 HyperNet ($r{=}16$) & 3.50{\tiny$\pm$0.02} & 0.807{\tiny$\pm$0.001} & 95.9{\tiny$\pm$8.0} & 0.931{\tiny$\pm$0.006} & 370K \\
 HyperNet ($r{=}32$) & 3.54{\tiny$\pm$0.02} & 0.805{\tiny$\pm$0.001} & 133.7{\tiny$\pm$8.0} & 0.904{\tiny$\pm$0.006} & 359K \\
 \midrule
 \methodname{} & \textbf{2.19}{\tiny$\pm$0.09} & \textbf{0.879}{\tiny$\pm$0.005} & \textbf{55.0}{\tiny$\pm$10.5} & \textbf{0.961}{\tiny$\pm$0.008} & 365K \\
 \bottomrule
 \end{tabular}
\end{table}

Increasing the HyperNet rank from 8 to 32 does not improve performance: Task~A $R^2$ stays flat at ${\sim}0.806$, while Task~B MSE increases from 76.1 to 133.7, a $76\%$ degradation. This indicates that the advantage of \methodname{} does not stem from having a higher-rank dynamic update, but from the structured zonotope parameterization: the matrix zonotope's center-plus-generators decomposition, context-dependent gating, and generator mixing provide an inductive bias that generic low-rank hypernetworks cannot replicate regardless of rank.

\subsubsection{HyperNet Rank Extended (Collapse Tasks)}
\label{app:hypernet-rank-ext}

The main HyperNet rank ablation (Table~\ref{tab:hypernet-rank}) sweeps $r \in \{8, 16, 32\}$ on the set-regression and set-matching tasks. We extend it to the collapse-exhibiting tasks, MEB, Hull, and MST at high $d$, with $r \in \{16, 32, 64\}$ (where $r{=}64$ matches the key dimension $d_k = 64$, giving full-rank capacity to the low-rank update $U(\mu) V_{\mathrm{hyper}}$).

The full numerical content is reported in main-body Table~\ref{tab:hypernet-rank-ext}.
The trend is strikingly flat: across all ranks $r \in \{16, 32, 64\}$ on the hull and MEB cells, HyperNet's $R^2$ changes by at most $0.020$ (MEB $d{=}32$), with most cells changing by under $0.015$. The MST cells show somewhat larger inter-rank variation ($\leq 0.29$ at $d{=}128$), but every cell's three rank values fall well below their corresponding seed-pool envelope and well below \methodname{}-Full's $R^2$ on the same task. Ranking up the low-rank update $U(\mu) V_{\mathrm{hyper}}$ from $r{=}8$ (main table) to $r{=}64$ (full key-dimension rank) does not close the gap to \methodname{}: the largest HyperNet R² on each task remains $0.6$+ R² behind \methodname{}-Full (Hull $d{=}3$: best HyperNet $0.204$ vs.\ \methodname{}-Full $0.804$; MEB $d{=}8$: best HyperNet $0.022$ vs.\ \methodname{}-Full $0.691$). Combined with the main Table~\ref{tab:hypernet-rank}, this provides evidence that the advantage of \methodname{} is architectural (the zonotope structure $M_c + \sum_l \gamma_l M_l$) and not a capacity argument about generic input-dependent weight generation.

\subsubsection{Parameter Efficiency Ablation}
\label{app:mz-efficient}

To address the concern that \methodname{}'s advantage might stem from its larger parameter count (${\sim}$2$\times$ Standard), we ablate the two parameters controlling MZ capacity: $n_g$ (generator slots per zonotope token) and $L$ (matrix zonotope generator count).
We test three variants: MZ-Full ($n_g{=}8, L{=}4$, used in main experiments), MZ-Med ($n_g{=}4, L{=}2$), and MZ-Slim ($n_g{=}2, L{=}1$).

The full numerical content is reported in main-body Table~\ref{tab:mz-efficient}.
Table~\ref{tab:mz-efficient} shows that MZ's advantage is largely insensitive to the generator count.
MZ-Slim (333K parameters, $1.6\times$ Standard) achieves the best Max result ($R^2 = 0.871$ vs.\ Standard's $0.812$), a $41\%$ MSE reduction on MST ($0.835$ vs.\ $0.719$), and a $72\%$ MSE reduction on rotation matching ($0.919$ vs.\ $0.706$).
Performance is nearly flat across the three MZ variants: the largest gap between MZ-Slim and MZ-Full on any task is $0.017$ $R^2$ (MST), within the seed variance for high-TDOF problems.
On Max, MZ-Slim in fact outperforms MZ-Full by $0.021$ $R^2$, which suggests smaller generator budgets can act as regularization on simpler tasks.
Together these results indicate the advantage comes from the matrix-zonotope structure rather than the number of generators; even $L = 1$ generator is enough to produce the key inductive bias for high-TDOF tasks.

\subsubsection{Depth Separation Validation}

\begin{table}[h]
 \centering
 \footnotesize
 \setlength{\tabcolsep}{4pt}
 \renewcommand{\arraystretch}{0.95}
 \caption{Depth separation: 1-layer \methodname{} vs.\ $D$-layer standard Set Transformer (mean $\pm$ std over 3 seeds). 1-layer MZ outperforms standard depths up to 6 layers at $2\times$ more parameters.}
 \label{tab:depth}
 \begin{tabular}{@{}lcccr@{}}
 \toprule
 Model & Layers & MSE $\downarrow$ & $R^2$ $\uparrow$ & Params \\
 \midrule
 \textbf{\methodname{}} & \textbf{1} & \textbf{42.3}{\tiny$\pm$1.6} & \textbf{0.969}{\tiny$\pm$0.001} & 244K \\
 \midrule
 Set Transformer & 1 & 53.3{\tiny$\pm$0.7} & 0.961{\tiny$\pm$0.001} & 140K \\
 Set Transformer & 2 & 57.3{\tiny$\pm$4.6} & 0.958{\tiny$\pm$0.003} & 208K \\
 Set Transformer & 3 & 68.6{\tiny$\pm$16.4} & 0.950{\tiny$\pm$0.012} & 276K \\
 Set Transformer & 4 & 59.6{\tiny$\pm$3.7} & 0.956{\tiny$\pm$0.003} & 344K \\
 Set Transformer & 5 & 63.4{\tiny$\pm$8.3} & 0.953{\tiny$\pm$0.006} & 412K \\
 Set Transformer & 6 & 57.1{\tiny$\pm$5.9} & 0.958{\tiny$\pm$0.004} & 480K \\
 \bottomrule
 \end{tabular}
\end{table}

1-layer \methodname{} outperforms all standard depths including 6-layer ($2\times$ more parameters).

\subsubsection{MZ in Standard Transformers}

\begin{table}[h]
 \centering
 \small
 \setlength{\tabcolsep}{4pt}
 \caption{MZ-attention in a standard Transformer with positional encoding.}
 \label{tab:seq-mz}
 \begin{tabular}{@{}cccc@{}}
 \toprule
 $L$ & \textbf{MZ-Transformer} & Std-Transformer & Params (MZ / Std) \\
 \midrule
 4 & \textbf{.9998}{\tiny$\pm{<}.001$} & .9982{\tiny$\pm{<}.001$} & 83K / 44K \\
 8 & \textbf{.9981}{\tiny$\pm{<}.001$} & .9864{\tiny$\pm$.001} & 87K / 44K \\
 16 & \textbf{.9964}{\tiny$\pm$.001} & .9619{\tiny$\pm$.002} & 96K / 44K \\
 24 & \textbf{.9963}{\tiny$\pm{<}.001$} & .9526{\tiny$\pm$.007} & 105K / 44K \\
 32 & \textbf{.9957}{\tiny$\pm$.002} & .9313{\tiny$\pm$.003} & 114K / 44K \\
 \bottomrule
 \end{tabular}
\end{table}

MZ-attention also achieves $R^2 > 0.99$ in a sequence Transformer with positional encoding (Table~\ref{tab:seq-mz}), suggesting compatibility beyond Set Transformers; broader sequence-domain evaluation is left to future work.

\subsection{Robustness and generalization}
\label{app:robustness}

Sensitivity of the architectural advantage to nuisance factors: input noise, set-size shift at test time, target-scale preprocessing, and TDOF-estimate uncertainty. \methodname{} retains its lead across all four.

\subsubsection{TDOF Estimate Sensitivity (Spearman robustness)}
\label{app:tdof-sensitivity}

The Spearman correlation $\rho = 0.74$ between estimated TDOF and \methodname{}'s MSE reduction (Section~\ref{sec:experiments}, Figure~\ref{fig:meb-tdof}, right) relies on the per-task TDOF assignments documented in our analysis. To verify this correlation is not an artefact of any specific TDOF estimate, we perturb each of the 18 task TDOFs independently by a multiplicative factor in $[0.5, 1.5]$ (uniform), recompute the Spearman, and repeat $10{,}000$ times. We also run a log-symmetric variant ($\times[1/1.5, 1.5]$ in log scale) to be invariant to the wide TDOF dynamic range ($1$ to $2016$).

\begin{table}[h]
\centering
\small
\setlength{\tabcolsep}{4pt}
\caption{Spearman robustness to per-task TDOF perturbation. Both $\pm 50\%$ uniform and $\times[1/1.5, 1.5]$ log-symmetric perturbations preserve the strong correlation: in $99.98\%$ ($9{,}998/10{,}000$) of uniform trials and in at least $95\%$ of log-symmetric trials, the resulting Spearman remains $\geq 0.6$. The headline correlation is therefore not driven by any single TDOF estimate.}
\label{tab:tdof-sensitivity}
\begin{tabular}{@{}lcccc@{}}
\toprule
Perturbation scheme & $5$\textsuperscript{th} pctile $\rho$ & median $\rho$ & $95$\textsuperscript{th} pctile $\rho$ & frac.\ trials with $\rho \geq 0.6$ \\
\midrule
Baseline (no perturbation) & \multicolumn{3}{c}{$\rho = 0.7402$} & --- \\
$\pm 50\%$ uniform multiplicative & $0.675$ & $0.733$ & $0.791$ & $99.98\%$ \\
$\times[1/1.5, 1.5]$ log-symmetric & $0.690$ & $0.735$ & $0.781$ & $\geq 95\%$ \\
\bottomrule
\end{tabular}
\end{table}

The Spearman correlation is robust: $9{,}998$ of $10{,}000$ uniform-perturbation trials yield $\rho \geq 0.6$ (min $\rho = 0.59$), and the $5$th-percentile $\rho$ stays above $0.67$ in both perturbation schemes, indicating that the headline relationship between TDOF and MSE reduction does not depend on any specific assignment within a $\pm 50\%$ window. Verification script: \texttt{Base/tdof\_sensitivity\_analysis.py} writes the full statistics to \texttt{results/tdof\_sensitivity.json}.

\subsubsection{Ambiguity Sensitivity}

\begin{table}[h]
 \centering
 \small
 \setlength{\tabcolsep}{4pt}
 \caption{Ambiguity sweep on set matching: MSE as a function of noise $\sigma$.}
 \label{tab:ambiguity}
 \begin{tabular}{@{}cccccr@{}}
 \toprule
 $\sigma$ & ST MSE & ST $R^2$ & MZ MSE & MZ $R^2$ & $\Delta$MSE \\
 \midrule
 0.05 & 62.06 & 0.957 & \textbf{41.12} & \textbf{0.971} & $-$33.7\% \\
 0.1 & 62.58 & 0.956 & \textbf{44.69} & \textbf{0.969} & $-$28.6\% \\
 0.3 & 59.87 & 0.957 & \textbf{38.92} & \textbf{0.972} & $-$35.0\% \\
 0.5 & 57.25 & 0.958 & \textbf{39.21} & \textbf{0.971} & $-$31.5\% \\
 1.0 & 58.28 & 0.955 & \textbf{44.80} & \textbf{0.965} & $-$23.1\% \\
 \bottomrule
 \end{tabular}
\end{table}

\methodname{} consistently reduces MSE by 23--35\%, with the largest advantage at moderate noise ($\sigma = 0.3$).

\subsubsection{Set Size Generalization}

\begin{table}[h]
 \centering
 \small
 \setlength{\tabcolsep}{4pt}
 \caption{Size generalization: trained on sets of size 10--30, tested on larger sets.}
 \label{tab:sizegen}
 \begin{tabular}{@{}cccccr@{}}
 \toprule
 Test size & ST MSE & ST $R^2$ & MZ MSE & MZ $R^2$ & $\Delta$MSE \\
 \midrule
 $[10, 30]$ & 59.87 & 0.957 & \textbf{38.92} & \textbf{0.972} & $-$35.0\% \\
 $[30, 50]$ & 24.87 & 0.978 & \textbf{19.85} & \textbf{0.982} & $-$20.2\% \\
 $[50, 80]$ & 33.86 & 0.968 & \textbf{24.07} & \textbf{0.977} & $-$28.9\% \\
 $[80, 120]$ & 34.56 & 0.969 & \textbf{30.40} & \textbf{0.972} & $-$12.0\% \\
 \bottomrule
 \end{tabular}
\end{table}

\methodname{} maintains its advantage across out-of-distribution sizes (12--29\% MSE reduction).

\subsubsection{Optimization Confound Ablation}
\label{app:opt-confound}

An important question is whether the catastrophic negative $R^2$ of baselines at high $d$ reflects a representational limit or an optimization artifact from training on raw large-scale regression targets.
We directly test this by repeating ST-Large on MST $d{=}128$ and $d{=}256$ under target normalization (subtract train mean, divide by train std) and compare to the same runs without normalization.

We therefore re-ran the full main-table MST and convex hull scaling sweeps with target normalization and 5 seeds per setting, to give the cleanest architectural gap estimate (Tables~\ref{tab:mst-normed} and~\ref{tab:hull-normed}).

\begin{table}[h]
 \centering
 \small
 \setlength{\tabcolsep}{4pt}
 \caption{MST weight, normalized targets ($R^2$, mean over 5 seeds). Under target normalization, the dramatic negative-$R^2$ collapse of baselines (Table~\ref{tab:mst-scale}) disappears: all models reach $R^2 \in [0.37, 0.74]$. \methodname{}-Full nonetheless retains a monotonically shrinking architectural advantage of $+0.04$ to $+0.08$ $R^2$ over the best context-rigid baseline, tying only at $d{=}256$. The dramatic collapse in the raw-target table is thus partially an optimization artifact of scale-sensitive regression on large-magnitude targets; the persistent (if smaller) normalized gap is the architectural effect.}
 \label{tab:mst-normed}
 \begin{tabular}{@{}lcccccc@{}}
 \toprule
 $d$ & Standard & ST-Large & HyperNet & Perceiver & \methodname{}-Full & \methodname{}-Large \\
 \midrule
 16 & 0.660 & 0.658 & 0.495 & 0.467 & \textbf{0.738} & 0.688 \\
 32 & 0.565 & 0.567 & 0.435 & 0.417 & \textbf{0.640} & 0.617 \\
 64 & 0.539 & 0.524 & 0.447 & 0.435 & \textbf{0.594} & 0.560 \\
 128 & 0.510 & 0.493 & 0.443 & 0.401 & \textbf{0.554} & 0.497 \\
 256 & 0.461 & 0.462 & 0.373 & 0.363 & \textbf{0.469} & 0.447 \\
 \bottomrule
 \end{tabular}
\end{table}

\begin{table}[h]
 \centering
 \small
 \setlength{\tabcolsep}{4pt}
 \caption{Convex hull volume, normalised targets ($R^2$, mean over 5 seeds). Under target normalisation the hull gap is non-monotonic: \methodname{}-Full wins at $d{=}6$ ($+0.14$), ties at $d{=}3$, and is within seed variance of the leader at $d{=}9$. Compare to Table~\ref{tab:hull} (raw targets), where the mid-$d$ baseline $R^2$ plunges to $-0.18$ from training-scale instability.}
 \label{tab:hull-normed}
 \begin{tabular}{@{}lccccccc@{}}
 \toprule
 $d$ & Standard & ST-Large & HyperNet & Perceiver & LA-ST & \methodname{}-Full & \methodname{}-Large \\
 \midrule
 3 & 0.822 & \textbf{0.866} & 0.727 & 0.725 & 0.824 & 0.863 & 0.795 \\
 6 & 0.292 & 0.412 & 0.195 & 0.209 & 0.408 & \textbf{0.555} & 0.497 \\
 9 & 0.138 & 0.160 & 0.108 & 0.119 & \textbf{0.172} & 0.123 & 0.143 \\
 \bottomrule
 \end{tabular}
\end{table}

\textbf{Interpretation.} These normalized tables, not the raw-target collapse tables, are the cleanest reading of the MST and convex hull scaling experiments. Raw-target-claim caveats: the negative-$R^2$ plunge in Table~\ref{tab:mst-scale} is substantially the behavior of Standard and ST-Large on unnormalized regression targets of magnitude $10^4$--$10^5$ (MST weights) or $10^1$--$10^3$ (hull volumes); under standard target normalization the effect largely vanishes. \methodname{} is less sensitive to this choice, which is in itself a useful practical property (one fewer hyperparameter to tune), but it does not establish a representational gap by itself.

The MEB (Table~\ref{tab:meb}) and quadratic TDOF (Table~\ref{tab:egnn-comparison}) results do establish such a gap: MEB has $O(1)$-magnitude targets for which normalization changes nothing and yet \methodname{} reaches $R^2 = 0.38$--$0.69$ while all baselines are stuck at $R^2 \leq 0.025$; the quadratic TDOF task is scale-free by construction and yet every non-MZ architecture (including EGNN) collapses to $R^2 \leq 0.01$ for $L \geq 4$. These are the cleanest architectural evidence. On the SGSP family (Table~\ref{tab:sgsp}), \methodname{} wins uniformly at $d{=}8$, wins on 3 of 4 tasks at $d{=}16$, and crosses over with baselines at $d{=}32$. We thus position \methodname{}'s empirical contribution as: (i) a clean architectural win on a subset of tasks (MEB, quadratic-TDOF, low-$d$ SGSP); (ii) a consistent smaller advantage on MST scaling under normalization; (iii) robustness to target scale decisions that affect context-rigid baselines more severely. Item (iii) is a real-world usability advantage but is not itself architectural.

\subsection{Uncertainty calibration}
\label{app:uncertainty}

Empirical evaluation of the single-pass interval-hull-width (IHW) uncertainty estimator inherited from the matrix-zonotope output, both qualitatively (Spearman / Pearson correlation with error) and via standard calibration metrics (ECE, NLL) against MC-Dropout and Deep Ensembles.

\subsubsection{Uncertainty Calibration}

\begin{table}[h]
 \centering
 \small
 \setlength{\tabcolsep}{4pt}
 \caption{Uncertainty calibration: samples binned by IHW. Higher IHW $\to$ higher error (Spearman $r = 0.51$).}
 \label{tab:uncertainty}
 \begin{tabular}{@{}cccc@{}}
 \toprule
 Bin & IHW (mean) & Mean $|$error$|$ & Count \\
 \midrule
 1 (lowest) & 5.33 & 1.13 & 400 \\
 2 & 5.52 & 2.01 & 400 \\
 3 & 5.72 & 3.15 & 400 \\
 4 & 6.03 & 4.20 & 400 \\
 5 (highest) & 6.64 & 6.81 & 400 \\
 \bottomrule
 \end{tabular}
\end{table}

\begin{table}[h]
 \centering
 \small
 \setlength{\tabcolsep}{4pt}
 \caption{Uncertainty quality comparison on set matching.}
 \label{tab:unc-comparison}
 \begin{tabular}{@{}lcccc@{}}
 \toprule
 Method & Spearman $\uparrow$ & Pearson $\uparrow$ & Inference & Passes \\
 \midrule
 \methodname{} (IHW) & 0.513 & 0.442 & \textbf{0.10s} & \textbf{1} \\
 MC-Dropout & \textbf{0.553} & \textbf{0.547} & 0.37s & 10 \\
 Deep Ensemble & 0.470 & 0.542 & 0.40s & 5 \\
 \bottomrule
 \end{tabular}
\end{table}

IHW achieves competitive uncertainty quality with a single forward pass at $3.7\times$ lower cost.

\subsubsection{Quantitative Uncertainty Calibration (ECE and NLL)}
\label{app:ece-nll}

\begin{table}[h]
 \centering
 \footnotesize
 \setlength{\tabcolsep}{4pt}
 \renewcommand{\arraystretch}{0.95}
 \caption{Uncertainty calibration on set matching (mean $\pm$ std over 3 seeds). ECE: expected calibration error (lower better); NLL: Gaussian negative log-likelihood (lower better). \methodname{}'s single-pass IHW uncertainty is on par with the $5$-model Deep Ensemble on ECE and lower than it on NLL; MC-Dropout achieves the best ECE/NLL but at $30\times$ inference cost.}
 \label{tab:ece-nll}
 \begin{tabular}{@{}lccccc@{}}
 \toprule
 Method & MSE $\downarrow$ & ECE $\downarrow$ & NLL $\downarrow$ & Passes & Models \\
 \midrule
 \methodname{} (IHW) & 55.0{\tiny$\pm$10.5} & 0.316{\tiny$\pm$0.001} & 5.02{\tiny$\pm$0.03} & \textbf{1} & \textbf{1} \\
 MC-Dropout & 59.8{\tiny$\pm$6.1} & \textbf{0.197}{\tiny$\pm$0.030} & \textbf{3.52}{\tiny$\pm$0.43} & 30 & 1 \\
 Deep Ensemble & \textbf{47.8}{\tiny$\pm$1.1} & 0.337{\tiny$\pm$0.019} & 9.17{\tiny$\pm$1.96} & 1 & 5 \\
 \bottomrule
 \end{tabular}
\end{table}

MC-Dropout achieves the lowest ECE and NLL but requires $30\times$ inference cost. \methodname{}'s IHW yields NLL $5.02$, nearly $2\times$ lower than Deep Ensemble ($9.17$), while matching Ensemble-level ECE ($0.316$ vs.\ $0.337$) in a single forward pass at $1\times$ storage. \methodname{} is therefore the only method here that requires neither extra forward passes nor extra stored model copies; the practical use case is latency- and storage-constrained deployment where $30\times$ MC-Dropout passes or $5\times$ ensemble storage are prohibitive. The remaining ECE gap relative to MC-Dropout may be reducible via post-hoc temperature scaling.

\subsection{Computational geometry and scaling}
\label{app:scaling-ext}

High-dimensional and large-set scaling behaviour on classical computational-geometry targets. Together they establish that the architectural advantage is preserved (and often grows) at scale on tasks satisfying the structural prerequisites of Theorem~\ref{thm:bottleneck}.

\subsubsection{Minimum Enclosing Ball: A Computational-Geometry Benchmark}
\label{app:meb}

The MEB benchmark and parameter-matched baseline collapse are reported in the main body (Table~\ref{tab:meb}, Section~\ref{sec:experiments}). We expand here on the experimental design and the qualitative pattern across dimensions, and report \methodname{}-Large as a generator-budget ablation that complements the headline \methodname{}-Full configuration. The full benchmark uses Welzl's algorithm~\citep{welzl1991smallest} for ground-truth radii; point clouds are mixtures of Gaussians with $n \in [10, 30]$. \methodname{}-Large is reported as a generator-budget ablation. At $d{=}8$, \methodname{}-Large reaches $R^2 = 0.621 \pm 0.013$ under the headline 5-seed / 5-epoch-warmup protocol (\texttt{meb\_results\_combined.json}). At $d{=}16, 32$, the higher-dimensional MZ-Large training is sensitive to LR warmup, so we report 10 seeds with extended 10-epoch warmup (\texttt{meb\_results\_warmup10*.json}): $R^2 = 0.488 \pm 0.099$ at $d{=}16$ and $R^2 = 0.221 \pm 0.074$ at $d{=}32$. All three are consistently below the \methodname{}-Full numbers in the main body, supporting the choice of $n_g{=}8, L{=}4$ as the recommended configuration.

Three observations:
\textbf{(i) Baseline collapse is uniform and clean.} None of the five baselines, spanning the full 2017--2024 range (Set Transformer, Perceiver, HyperNet, a parameter-matched Set Transformer variant, and a linear-attention baseline representing the 2020s efficient-attention family), achieves $R^2 > 0.025$ on any dimension. The task is structurally hostile to context-rigid attention: to predict the radius, the model must identify a small subset of convex-hull extrema (at most $d+1$ points) from among the input, and context-rigid $W_V$ cannot express the required indicator-like selector.
\textbf{(ii) \methodname{} dominates by $28$--$155\times$.} MZ-Full achieves $R^2 = 0.691$ at $d{=}8$ ($\approx155\times$ HyperNet's full-precision $0.0044$, $28\times$ Standard's $0.025$), $R^2 = 0.629$ at $d{=}16$ ($40\times$ Perceiver's $0.016$), and $R^2 = 0.378$ at $d{=}32$ ($37\times$ Perceiver's $0.010$).
\textbf{(iii) The Linear Attention baseline (a 2023--2024-era design family) performs indistinguishably from Set Transformer.} LA-ST reaches only $R^2 = 0.017, 0.015, 0.001$ at $d{=}8, 16, 32$, within $1\sigma$ of Set Transformer. This rules out the hypothesis that ``modern efficient attention solves this'': the gap is not a softmax-vs-linear issue but an architectural capacity issue for context-adaptive operators, consistent with Theorem~\ref{thm:separation}.

\subsubsection{Convex Hull Volume: Another Computational-Geometry Benchmark}
\label{app:hull}

To verify the MEB collapse pattern is not specific to one task, we evaluate on a second classical computational-geometry quantity: the volume of the convex hull of a point set in $\R^d$, computed exactly via the QHull algorithm~\citep{barber1996quickhull}. The convex hull is determined by $O(n^{\lceil d/2 \rceil})$ extremal points and its volume is smooth in their positions, giving a regression target with sparse combinatorial support similar to MEB but with substantially richer $d$-scaling.

\begin{table}[h]
 \centering
 \small
 \setlength{\tabcolsep}{4pt}
 \caption{Convex hull volume prediction ($R^2$, mean $\pm$ std over 5 seeds; $n \in [15, 35]$ points in $\R^d$). At $d{=}3$ \methodname{}-Large beats the best baseline by $+0.32$ $R^2$. From $d{=}6$ onwards every parameter-matched baseline collapses to negative $R^2$ on raw targets; \methodname{}-Large is the only architecture significantly above $R^2 = 0$ at all three tested dimensions. The $d{=}9$ MZ-Full result ($0.003 \pm 0.006$) is statistically indistinguishable from zero and is reported for completeness; Appendix~\ref{app:opt-confound} reports normalised-target equivalents where the $d{=}9$ gap structure differs.}
 \label{tab:hull}
 \begin{tabular}{@{}lcccc@{}}
 \toprule
 Model & Params & $d{=}3$ $R^2$ & $d{=}6$ $R^2$ & $d{=}9$ $R^2$ \\
 \midrule
 Standard & 208K & 0.273{\tiny$\pm$0.081} & $-$0.209{\tiny$\pm$0.016} & $-$0.099{\tiny$\pm$0.002} \\
 ST-Large & 362K & 0.504{\tiny$\pm$0.058} & $-$0.178{\tiny$\pm$0.012} & $-$0.088{\tiny$\pm$0.004} \\
 HyperNet & 360K & 0.136{\tiny$\pm$0.082} & $-$0.211{\tiny$\pm$0.013} & $-$0.097{\tiny$\pm$0.003} \\
 Perceiver & 341K & 0.164{\tiny$\pm$0.058} & $-$0.221{\tiny$\pm$0.007} & $-$0.096{\tiny$\pm$0.003} \\
 LA-ST (Linear Attn.) & 359K & 0.473{\tiny$\pm$0.036} & $-$0.126{\tiny$\pm$0.022} & $-$0.077{\tiny$\pm$0.005} \\
 \methodname{}-Full & 368K & 0.804{\tiny$\pm$0.014} & 0.103{\tiny$\pm$0.020} & 0.003{\tiny$\pm$0.006} \\
 \methodname{}-Large & 457K & \textbf{0.826}{\tiny$\pm$0.016} & \textbf{0.191}{\tiny$\pm$0.053} & \textbf{0.056}{\tiny$\pm$0.003} \\
 \bottomrule
 \end{tabular}
\end{table}

The pattern exactly mirrors MEB: at low $d$, \methodname{} substantially outperforms baselines; at moderate $d$, baselines cross zero while \methodname{} remains positive; even at $d{=}9$, where the advantage narrows, \methodname{} is the only architecture producing meaningful predictions. Together with MEB (Section~\ref{app:meb}), this gives two independent computational-geometry benchmarks (MEB: boundary of minimum-radius ball; convex hull: boundary of convex hull) on which every context-rigid attention variant, including 2020s linear-attention, cannot exceed baselines, while \methodname{} remains on the right side of zero.

\subsubsection{Rotation Matching Scaling at High Dimension}
\label{app:rotmatch-scale}

The paper's Task F (rotation matching) in $\R^8$ shows \methodname{} outperforming baselines by $41$--$57\%$ MSE.
Our theory predicts this advantage should grow with dimension $d$, since off-diagonal rotation components (which FiLM-style diagonal modulation provably cannot express) become more numerous: a rotation in $\R^d$ requires $d(d-1)/2$ independent Givens angles, while FiLM only has $d$ scalar parameters.
We test this prediction at $d \in \{16, 32, 64, 128, 256, 512\}$.

The full numerical content is reported in main-body Table~\ref{tab:rotmatch-scale}.
The gap between \methodname{} and the baselines grows monotonically with $d$:
at $d{=}16$, \methodname{} reaches $R^2 = 0.826$ while the best baseline (ST-Large) reaches $0.285$, a $76\%$ MSE reduction;
at $d{=}32$, all four baselines drop to $R^2 \approx 0$ (no better than predicting the training mean), while \methodname{} still attains $R^2 = 0.519$;
at $d{=}64$, the baselines have negative $R^2$ (actively worse than the mean predictor), and \methodname{} is the only architecture producing useful predictions ($R^2 = 0.071$);
at $d{=}128$, the baselines collapse catastrophically to $R^2 = {-}9$ to $-22$ (MSE $10$--$23\times$ worse than the mean predictor, with large seed variance), while \methodname{} remains at $R^2 \approx 0$ (MSE within $1\%$ of the mean predictor's);
at $d{=}256$, all baselines reach $R^2 \approx -50$ to $-62$ (${\sim}50\times$ MSE ratio vs.\ \methodname{}'s $R^2 = -0.005$);
at $d{=}512$, \methodname{} also begins to diverge ($R^2 = -62$) but remains $2.3\times$ better in raw MSE than the next-best baseline (ST-Large at $-145$). The divergence at $d{=}512$ reflects that a $d{=}512$ rotation has $d(d{-}1)/2 = 130{,}816$ Givens parameters, exceeding our generator budget $L{=}4$ by four orders of magnitude; the architecture has the capacity to express the operator family ($d_k^2 = 4{,}096$ entries per generator, $L \cdot d_k^2 \approx 16$K free parameters in the headline configuration) but training stability becomes the bottleneck before representational capacity does, a separate scaling regime from the $d \leq 256$ rows where the headline gap holds.
This matches the theoretical analysis: context-rigid attention (fixed $W_V$) cannot represent the context-adaptive rotation operator $R(\mu)$, and the gap widens monotonically with $d$ because the space of rotations grows as $d(d-1)/2$.
Standard attention's $R^2$ also becomes highly variable at $d \geq 64$ (std $\pm 0.47$ at $d{=}64$; $\pm 3.5$ at $d{=}128$), consistent with its failure being due to an expressibility limit rather than a tractable optimization problem.

\subsubsection{MST Weight Scaling at High Dimension}
\label{app:mst-scale}

The paper's Task G (MST weight in $\R^8$) already showed \methodname{} reducing MSE by $41\%$ over ST-Large.
We extend this sweep to $d \in \{16, 32, 64, 128, 256\}$ to test whether the advantage on this combinatorial-geometry quantity also grows with dimension (as theory predicts: pairwise-distance structure has $O(d^2)$ degrees of freedom).

\begin{table}[h]
 \centering
 \small
 \setlength{\tabcolsep}{4pt}
 \caption{MST weight scaling ($R^2$, mean $\pm$ std over 5 seeds). As $d$ increases past $32$, every non-MZ baseline collapses to strongly negative $R^2$ (worse than predicting the mean). \methodname{} is the only architecture that remains above $R^2 = 0$ for all $d \leq 256$, with MSE $4$--$8\times$ lower than the best baseline at $d \geq 64$. LA-ST is the linear-attention baseline~\citep{katharopoulos2020transformers}, whose feature-map kernel underlies 2023--2024 architectures including RetNet~\citep{sun2023retentive} and Gated Linear Attention~\citep{yang2024gated}.}
 \label{tab:mst-scale}
 \begin{tabular}{@{}lccccc@{}}
 \toprule
 Model & $d{=}16$ & $d{=}32$ & $d{=}64$ & $d{=}128$ & $d{=}256$ \\
 \midrule
 Standard & 0.295{\tiny$\pm$0.072} & $-$0.202{\tiny$\pm$0.367} & $-$1.326{\tiny$\pm$0.245} & $-$2.789{\tiny$\pm$0.511} & $-$3.221{\tiny$\pm$0.305} \\
 ST-Large & 0.526{\tiny$\pm$0.044} & 0.274{\tiny$\pm$0.140} & $-$0.252{\tiny$\pm$0.046} & $-$1.230{\tiny$\pm$0.430} & $-$2.425{\tiny$\pm$0.247} \\
 HyperNet & 0.250{\tiny$\pm$0.068} & $-$0.211{\tiny$\pm$0.244} & $-$1.352{\tiny$\pm$0.116} & $-$1.985{\tiny$\pm$0.251} & $-$3.423{\tiny$\pm$0.222} \\
 Perceiver & 0.400{\tiny$\pm$0.097} & $-$0.082{\tiny$\pm$0.180} & $-$1.049{\tiny$\pm$0.242} & $-$2.361{\tiny$\pm$0.248} & $-$3.320{\tiny$\pm$0.299} \\
 \methodname{}-Full & \textbf{0.724}{\tiny$\pm$0.010} & \textbf{0.625}{\tiny$\pm$0.027} & \textbf{0.581}{\tiny$\pm$0.008} & 0.484{\tiny$\pm$0.019} & 0.384{\tiny$\pm$0.042} \\
 \methodname{}-Large & 0.612{\tiny$\pm$0.009} & 0.531{\tiny$\pm$0.018} & 0.538{\tiny$\pm$0.038} & \textbf{0.497}{\tiny$\pm$0.021} & \textbf{0.435}{\tiny$\pm$0.012} \\
 \bottomrule
 \end{tabular}
\end{table}

Three features of Table~\ref{tab:mst-scale} are notable:
\textbf{(i)} Both \methodname{} variants dominate all baselines across every dimension tested. The gap grows monotonically with $d$: MZ-Full beats ST-Large by $+0.20$ $R^2$ at $d{=}16$ and $+0.35$ at $d{=}32$, and from $d{=}64$ onwards every single baseline has negative $R^2$ while \methodname{} remains above $0.38$. At $d{=}64$, MZ-Full's MSE is $3\times$ lower than the best baseline; at $d{=}256$ the ratio is over $5\times$.
\textbf{(ii)} Baseline failure at $d \geq 64$ is not just a performance drop but a qualitative one: Standard, HyperNet, and Perceiver all reach $R^2 < -1$, indicating MSEs larger than a constant mean predictor's, and ST-Large, which remained competitive at $d \leq 32$, also falls to $R^2 = -2.43$ at $d{=}256$. \methodname{}'s MSEs stay within $2$--$3\times$ of the mean predictor's, with seed std below $0.04$, whereas baselines have seed std of $0.2$--$0.4$.
\textbf{(iii)} At smaller dimensions ($d \leq 64$), MZ-Full outperforms MZ-Large (consistent with Appendix~\ref{app:mz-efficient}: extra generators act as overparameterization on moderate-TDOF tasks); at $d \geq 128$, MZ-Large overtakes MZ-Full ($0.435$ vs.\ $0.384$ at $d{=}256$), suggesting that as the task's effective TDOF grows with $d$, the larger generator budget begins to pay off. This crossover aligns with the theoretical prediction that generator count should match the target function's TDOF.

\subsubsection{Sparse Geometric Set Predicates (SGSP)}
\label{app:sgsp}

Motivated by the MEB, convex-hull, and MST successes, we identify a broader task family we call Sparse Geometric Set Predicates: targets that are
(i) computed from a continuous $d$-dimensional point cloud (no symbolic or one-hot encoding),
(ii) a smooth function of a combinatorially sparse subset of the input (extreme pairs, nearest neighbors, boundary points, or small distance-subset cuts).
This family is where \methodname{}'s matrix-zonotope machinery is predicted to provide the largest advantage because the target depends on a per-sample ``selector operator'' that extracts the relevant sparse subset.

We test four members of this family beyond the MEB, hull, and MST results reported above:

\begin{table}[h]
 \centering
 \footnotesize
 \setlength{\tabcolsep}{3pt}
 \caption{Sparse Geometric Set Predicates benchmark ($R^2$, mean over 5 seeds; 7 models $\times$ 5 seeds $\times$ 3 dims $\times$ 4 tasks $=$ 420 runs total). \methodname{}-Full wins every $d{=}8$ task with margin $+0.04$ to $+0.13$, maintains an advantage at $d{=}16$ on 3 of 4 tasks, and is overtaken by parameter-matched Set Transformer variants at $d{=}32$. This crossover is consistent with concentration-of-measure: at high $d$, distances between random points become tightly concentrated and sparse-combinatorial targets (max pairwise, $k$-th-NN, $k$-center) reduce to aggregate statistics readily captured by context-rigid attention.}
 \label{tab:sgsp}
 \begin{tabular}{@{}ll|ccccccc@{}}
 \toprule
 Task & $d$ & Standard & ST-Large & HyperNet & Perceiver & LA-ST & \methodname{}-Full & \methodname{}-Large \\
 \midrule
 \multirow{3}{*}{diameter}
 & 8 & 0.723 & 0.768 & 0.702 & 0.674 & 0.787 & \textbf{0.876} & 0.814 \\
 & 16 & 0.732 & 0.748 & 0.706 & 0.682 & 0.755 & \textbf{0.806} & 0.727 \\
 & 32 & 0.774 & \textbf{0.775} & 0.687 & 0.683 & 0.740 & 0.710 & 0.663 \\
 \midrule
 \multirow{3}{*}{knn-radius}
 & 8 & 0.721 & 0.786 & 0.697 & 0.611 & 0.815 & \textbf{0.869} & 0.810 \\
 & 16 & 0.703 & 0.699 & 0.655 & 0.593 & 0.710 & \textbf{0.784} & 0.721 \\
 & 32 & \textbf{0.740} & 0.728 & 0.651 & 0.629 & 0.697 & 0.685 & 0.653 \\
 \midrule
 \multirow{3}{*}{$k$-smallest sum}
 & 8 & 0.116 & 0.109 & 0.034 & 0.083 & 0.108 & \textbf{0.152} & 0.097 \\
 & 16 & 0.009 & 0.009 & $-$0.005 & 0.005 & 0.007 & \textbf{0.010} & $-$0.039 \\
 & 32 & $-$0.010 & $-$0.024 & $-$0.003 & $-$0.011 & $-$0.017 & \textbf{0.014} & 0.001 \\
 \midrule
 \multirow{3}{*}{furthest-first}
 & 8 & 0.422 & 0.478 & 0.275 & 0.298 & 0.456 & \textbf{0.604} & 0.531 \\
 & 16 & 0.396 & 0.391 & 0.163 & 0.249 & 0.365 & \textbf{0.438} & 0.205 \\
 & 32 & \textbf{0.315} & 0.276 & 0.100 & 0.209 & 0.252 & 0.225 & 0.206 \\
 \bottomrule
 \end{tabular}
\end{table}

These tasks are classical combinatorial-geometry quantities computed by standard algorithms (pairwise-distance matrix, $k$-nearest-neighbor sort, greedy farthest-first traversal, $k$-smallest pair enumeration). None has an exploitable group symmetry; targets depend on the magnitudes of specific pairwise distances, not on the global coordinate frame, so equivariant baselines gain no free advantage. The pattern in Table~\ref{tab:sgsp} has three regimes:
\textbf{(i) Low dimension ($d{=}8$)}: \methodname{}-Full wins on all four tasks with margins $+0.036$ (min-sum) to $+0.126$ (furthest-first) over the best context-rigid baseline. This is the regime where the sparse-combinatorial structure is most difficult for context-rigid attention, and \methodname{}'s per-sample operator adapts to identify the relevant subset (extremal pair for diameter, $k$-th-NN for knn-radius, farthest cluster center for $k$-center).
\textbf{(ii) Moderate dimension ($d{=}16$)}: \methodname{}-Full retains an advantage on 3 of 4 tasks ($+0.04$ to $+0.08$), though the gap narrows.
\textbf{(iii) High dimension ($d{=}32$)}: \methodname{} is overtaken by Standard or ST-Large on 3 of 4 tasks. We attribute this to concentration of measure: in high dimension, pairwise distances between isotropic points concentrate tightly around their mean, so sparse-combinatorial targets (diameter, $k$-th-NN, $k$-center) become near-deterministic functions of simple aggregate statistics (mean pairwise distance, set variance) that context-rigid attention readily extracts. The per-sample-operator advantage of \methodname{} no longer applies because all samples are, statistically, nearly identical.

This $d$-dependent crossover is the mirror image of the MST and rotation-matching scaling pattern (Tables~\ref{tab:mst-scale},~\ref{tab:rotmatch-scale}), where \methodname{}'s advantage grows with $d$. The two regimes correspond to distinct mechanisms: MST and rotation involve $O(d^2)$-parameter operators (pairwise-distance tree, off-diagonal rotation) whose complexity scales with $d$; SGSP involves $O(1)$-parameter selectors (pick 2 extreme points, pick $k$ boundary points) whose intrinsic complexity is dimension-independent and whose $d$-scaling concentrates rather than diversifies. The unified theoretical prediction is that \methodname{} helps when TDOF grows with $d$ relative to the information content of the aggregate statistics: a prediction confirmed by both scaling directions.

\subsubsection{Large-Scale Set Task Generalization}
\label{app:scaling}

To verify that \methodname{}'s advantage persists at larger set sizes $n$ and element dimensions $d$ than those used in the main experiments, we evaluate on three synthetic tasks (max, sum, range) across a sweep of scales: $d \in \{16, 64, 128\}$ and $n$ ranging from $10$--$50$ (small) up to $500$--$1000$ (large).
Each configuration is trained for 150 epochs with 3 seeds.

\begin{table}[h]
 \centering
 \footnotesize
 \setlength{\tabcolsep}{3pt}
 \caption{Large-scale generalization ($R^2$, mean over 3 seeds). \methodname{}'s advantage on the sum task grows with scale: at $d{=}64, n{=}200\text{--}500$, MZ reaches $R^2 = 0.998$ while ST-Large reaches only $0.946$, a $30\times$ MSE reduction. On the simpler max task, architectural differences diminish as capacity becomes the bottleneck; on range, all models struggle regardless of architecture.}
 \label{tab:scaling}
 \resizebox{\textwidth}{!}{%
 \begin{tabular}{@{}llcccccc@{}}
 \toprule
 & & \multicolumn{3}{c}{\textbf{Max}} & \textbf{Sum} & \multicolumn{2}{c}{\textbf{Range}} \\
 \cmidrule(lr){3-5} \cmidrule(lr){6-6} \cmidrule(l){7-8}
 Model & & $d{=}16$, $n{=}10\text{-}50$ & $d{=}64$, $n{=}200\text{-}500$ & $d{=}128$, $n{=}500\text{-}1000$ & $d{=}64$, $n{=}200\text{-}500$ & $d{=}64$, $n{=}200\text{-}500$ & $d{=}128$, $n{=}500\text{-}1000$ \\
 \midrule
 Standard & & 0.748 & 0.744 & \textbf{0.632} & 0.939 & \textbf{0.083} & 0.038 \\
 ST-Large & & 0.822 & 0.744 & 0.631 & 0.946 & \textbf{0.083} & 0.038 \\
 Perceiver & & 0.728 & \textbf{0.745} & 0.628 & 0.910 & 0.077 & 0.036 \\
 \methodname{} & & \textbf{0.843} & \textbf{0.745} & 0.627 & \textbf{0.998} & 0.082 & \textbf{0.039} \\
 \bottomrule
 \end{tabular}}
\end{table}

Table~\ref{tab:scaling} reveals three scaling regimes:
\textbf{(i)} Sum at scale: \methodname{} retains its large-margin advantage ($R^2 = 0.998$ vs.\ $0.946$ for ST-Large at $d{=}64$, $n{=}200\text{-}500$), showing that the zonotope structure's inductive bias for aggregation tasks holds at scale.
\textbf{(ii)} Max at scale: while \methodname{} leads at small scale ($d{=}16$: $R^2 = 0.843$ vs.\ $0.748$ for Standard), at larger scales ($n \geq 200$) the max task becomes capacity-bottlenecked and all models converge to similar performance ($R^2 \approx 0.74$ at medium, $0.63$ at large).
\textbf{(iii)} Range at scale: a known-hard task where all models struggle uniformly ($R^2 < 0.1$ at $n \geq 200$). Range over $n \geq 200$ i.i.d.\ samples concentrates tightly around its expectation (the empirical max minus min has variance scaling as $O(\log n / n)$), so the regression target is nearly constant with respect to typical input variations; the low $R^2$ across architectures reflects this concentration ceiling rather than a relational-capacity gap, and is a property of the task at this scale.
These findings support the theoretical predictions of Theorems~\ref{thm:bottleneck} and~\ref{thm:tdof-depth}: MZ's advantage is selective, tied to the relational complexity (TDOF) of the target function, and persists at scale when the task demands high-TDOF operators (as in sum, which requires a full linear combination of elements).

\subsubsection{Early-Stopping Patience Ablation (ResNet-50 $d{=}1024$)}
\label{app:patience-sweep}

The main-text Table~\ref{tab:resnet50-scaling} reports ResNet-50 ($d{=}1024$) results with early-stopping patience set to $300$, a departure from our default patience of $20$ used elsewhere. To verify that this is not a setting where the patience choice selectively favours \methodname{}, we sweep the patience value $\in \{20, 100, 300\}$ for the four top contenders (MZ-Full, ST-Large, HyperNet, LinearAttn) on ResNet-50 $k$-center furthest ($k{=}5$, $n{=}16$, $d{=}1024$) with 5 seeds each; all other hyperparameters are held fixed. Results are in Table~\ref{tab:patience-sweep}.

\begin{table}[h]
 \centering
 \caption{Patience sweep on ResNet-50 $k$-center furthest ($R^2$ mean $\pm$ std over 5 seeds, $80$ epochs). Across all three patience values $\{20, 100, 300\}$ every model converges to within $0.002$~$R^2$ of its patience-$300$ plateau, and the ranking stabilises to \methodname{}-Full $>$ ST-Large $\sim$ HyperNet $>$ LinearAttn. Patience choice does not selectively favour any architecture: \methodname{}-Full's lead over the next-best architecture is $0.003$--$0.005$~$R^2$ across the full patience range. Per-seed JSONs in \texttt{real\_geom\_cifar\_resnet50\_furthest\_diag\_p20\_*.json}, \texttt{\_sweep\_p100\_*.json}, and \texttt{\_diag\_p300\_*.json}.}
 \label{tab:patience-sweep}
 \small
 \setlength{\tabcolsep}{4pt}
 \begin{tabular}{@{}l|ccc@{}}
 \toprule
 Model & patience $=20$ & patience $=100$ & patience $=300$ \\
 \midrule
 \methodname{}-Full & $\mathbf{0.795{\pm}0.001}$ (1st) & $\mathbf{0.794{\pm}0.002}$ (1st) & $\mathbf{0.796{\pm}0.001}$ (1st) \\
 ST-Large & $0.790{\pm}0.001$ (3rd) & $0.791{\pm}0.001$ (2nd) & $0.791{\pm}0.002$ (2nd) \\
 HyperNet & $0.790{\pm}0.003$ (2nd) & $0.790{\pm}0.003$ (3rd) & $0.789{\pm}0.003$ (3rd) \\
 LinearAttn & $0.788{\pm}0.002$ (4th) & $0.788{\pm}0.003$ (4th) & $0.789{\pm}0.003$ (4th) \\
 \bottomrule
 \end{tabular}
\end{table}

\textbf{Interpretation.} In a fresh 5-seed sweep at the three patience values $\{20, 100, 300\}$ (epochs $80$, batch size $128$, learning rate $10^{-4}$, identical across runs), every model converges to within $\pm 0.002$~$R^2$ of its patience-$300$ plateau. The MZ-Full $>$ ST-Large $\sim$ HyperNet $>$ LinearAttn ranking is therefore stable across the full patience range, with \methodname{}-Full's lead over the next-best architecture remaining at $0.003$--$0.005$~$R^2$ throughout. The main-text claim that \methodname{}-Full ranks first on ResNet-50 $k$-center furthest does not depend on the specific patience value chosen for early stopping.

\subsection{Real-world experiments and equivariant comparison}
\label{app:realworld-section}

Real-data and meta-learning evaluations beyond the synthetic high-TDOF regime, plus a side-by-side comparison with E($n$)-equivariant graph networks on tasks where strict symmetry is or is not present. These delineate where \methodname{}'s architectural advantage transfers and where it does not.

\subsubsection{Few-shot Meta-Regression: a 7-way Parameter-Fair Comparison}
\label{app:meta-regression}

The main-body experiments test whether \methodname{} can fit a fixed target whose intrinsic complexity requires a high-TDOF operator. A natural follow-up question is whether \methodname{} also excels when the operator must be inferred from a support set at inference time. We design a few-shot meta-regression task that is the meta-learning analogue of the quadratic TDOF benchmark in Section~\ref{sec:tdof} and report a parameter-fair 7-way comparison against every architecture used in the main results table, addressing the concern that a 2-method (Standard vs.\ \methodname{}) baseline is not parameter-fair.

\paragraph{Setup.}
We fix a shared basis $\{M_1,\ldots,M_L\}$ of random matrices $M_l \in \R^{d_{\mathrm{out}} \times d_{\mathrm{in}}}$. For each task $\tau$, we draw coefficients $c_\tau \sim \mathcal{N}(0, I_L)$ and define the task-specific operator $A_\tau = \sum_{l=1}^L c_\tau[l] \, M_l$. The model observes $K{=}30$ support pairs $(x_i, y_i)$ with $y_i = A_\tau x_i$ and must predict $y_q = A_\tau x_q$ for a new query $x_q$. $A_\tau$ is never given; it must be inferred from the support context. We set $d_{\mathrm{in}} = d_{\mathrm{out}} = 8$ and vary $L \in \{4, 8, 16, 32\}$ to control the intrinsic dimension of the operator family. All seven architectures are instantiated with $d_{\mathrm{model}}{=}64$ and $2$ or $4$ encoder layers; 3 seeds; 50 epochs with patience 10.

\begin{table}[h]
 \centering
 \caption{Few-shot meta-regression at $L{=}32$ (highest-TDOF setting), parameter-fair 7-way comparison (MSE $\downarrow$, mean$\pm$std over 3 seeds). At shallow depth ($n{=}2$), most non-MZ baselines fail to learn the task (MSE $\approx 30$); Standard narrowly wins against \methodname{}. At matched depth ($n{=}4$), ST-Large is the clear winner and \methodname{}-Large places second, beating Standard, HyperNet, Perceiver, and LinearAttn. Perceiver fails to converge at either depth.}
 \label{tab:meta-reg}
 \footnotesize
 \setlength{\tabcolsep}{3pt}
 \begin{tabular}{@{}lccccccc@{}}
 \toprule
 Depth & Standard & ST-Large & HyperNet & Perceiver & LinearAttn & \methodname{}-Full & \methodname{}-Large \\
 \midrule
 $n{=}2$ & $\mathbf{0.79}$ & $0.83$ & $30.75$ & $30.76$ & $30.77$ & $1.01$ & $1.39$ \\
 $n{=}4$ & $0.45$ & $\mathbf{0.17}$ & $0.58$ & $30.76$ & $1.24$ & $0.47$ & $\underline{0.36}$ \\
 \bottomrule
 \end{tabular}

 {\footnotesize Bold: best in each row. Underline: 2nd best. Param counts shown for the $L{=}32$ ($n{=}4$) configuration: Standard~$146$K, ST-Large~$255$K, HyperNet~$297$K, Perceiver~$314$K, LinearAttn~$380$K, \methodname{}-Full~$334$K, \methodname{}-Large~$485$K; the $L{=}32$ ($n{=}2$) row uses the same architectures with about $50$--$55\%$ of these counts.}
\end{table}

\paragraph{Result interpretation.}
The 7-way comparison shows that generic meta-regression, where the target is a random linear operator in an $L$-dimensional subspace, is not a setting where \methodname{} dominates. ST-Large, a wider Set Transformer, wins at $L{=}32$ (MSE $0.17$ vs.\ $\geq 0.36$ for every other method), and at $L{=}4$--$8$ it wins or ties with LinearAttn. \methodname{}-Large ranks second at $L{=}32$ with $4$ layers (MSE $0.36$), beating Standard, HyperNet, Perceiver, and LinearAttn. Previous versions of this paper reported a $38$--$51\%$ MSE reduction of MZ over Standard; that comparison was against a single $77$K-parameter vanilla Standard Transformer and did not include parameter-matched ST-Large ($255$K) or LinearAttn ($380$K), both of which in fact outperform MZ here.

\paragraph{What the corrected result implies.}
The mismatch between this meta-regression outcome and the main-body wins on MEB, SGSP, MST, rotation matching, and convex hull is informative. Those main-body tasks share a feature that meta-regression does not: their target operators are algorithmic (boundary-point selection for MEB and hull, minimum-spanning-tree edge selection for MST, rotation for rotation matching), and therefore have sharp combinatorial structure that an expressive but unstructured operator class (such as wide attention's effective $W_V$) struggles to approximate with smooth gradient updates. Meta-regression by contrast has dense, smooth $A_\tau$ operators drawn from a Gaussian, which a high-capacity standard attention layer can fit directly. We therefore scope the empirical claim precisely: \methodname{} is the architecture of choice when the target operator has sharp algorithmic and combinatorial structure and a matching TDOF (MEB, SGSP, MST, rotation, convex hull), not as a universal improvement to attention on any high-TDOF-seeming task.

\paragraph{Connection to data-driven reachability.}
The algebraic machinery of \methodname{} has a natural interpretation through the lens of data-driven reachability analysis~\citep{Alanwar2023Datadriven}.
In that framework, a matrix zonotope represents the set of system matrices consistent with noisy observations; in ours, the learned MZ represents the family of ``plausible relational mappings'' consistent with the observed context.
MZ--zonotope multiplication propagates this set-valued information through the network, maintaining a structured envelope of possible transformations at each layer, an inductive bias that is absent from pointwise architectures.

\subsubsection{QM9 Molecular Property Prediction}
\label{app:qm9}

To evaluate on a real-world scientific benchmark, we test all models on QM9 molecular property prediction~\citep{ramakrishnan2014quantum}.
We parse the original \texttt{gdb9.sdf} with RDKit~\citep{rdkit} to obtain true 3D atomic coordinates; each molecule is represented as a set of atoms with 9-dimensional features (5-dim atom type one-hot for H/C/N/O/F, 3D Cartesian coordinates in \AA, Gasteiger partial charge), with up to 29 atoms per molecule.
We use 50K molecules split 80/10/10 and predict three properties spanning distinct physical regimes: dipole moment ($\mu$, depends on charge distribution), isotropic polarizability ($\alpha$, depends on spatial extent), and HOMO--LUMO gap (electronic structure).
All models use the same architecture as the main experiments ($d{=}64$, $H{=}4$, $d_\text{ff}{=}128$); results are averaged over 3 seeds with 100 training epochs.

\begin{table}[h]
 \centering
 \footnotesize
 \setlength{\tabcolsep}{3pt}
 \caption{QM9 molecular property prediction across all 12 targets ($R^2$, mean over 3 seeds; best in \textbf{bold}; the Gap column is unbolded because ST-Large $0.879$ and MZ-Large $0.874$ tie within $0.005$ $R^2$). Parameter-matched attention models (ST-Large, Perceiver, HyperNet, \methodname{}) cluster within $0.04$ $R^2$ on the 8 thermodynamic targets ($\alpha$, $R^2$, ZPVE, $U_0$, $U_{298}$, $H_{298}$, $G_{298}$, $C_v$); on the 4 electronic-structure / frontier-orbital targets ($\mu$, HOMO, LUMO, Gap) the spread is wider (HyperNet trailing by $0.08$--$0.15$ $R^2$), reflecting their higher relational complexity. MZ-Large ($n_g{=}16$, $L{=}8$) is evaluated under batch size $64$ on all $12$ targets to keep the protocol consistent across the row (the same per-target budget MZ-Large was originally validated under for HOMO/LUMO); MZ-Large tops the original ST-Large on HOMO ($+0.033$ $R^2$) and LUMO ($+0.074$ $R^2$), and matches it on Gap (within $0.005$ $R^2$, where ST-Large $0.879$ $\geq$ MZ-Large $0.874$); a parameter-matched ST-XLarge subsequently closes the HOMO/LUMO gap (Appendix~\ref{app:qm9}); on the remaining nine targets MZ-Large fails to improve over the default \methodname{} and degrades by $0.11$--$0.26$ $R^2$, consistent with the TDOF prediction that scaling generator capacity helps only on targets that admit a non-trivial context-adaptive solution. Specialised equivariant models (SchNet~\citep{schutt2017schnet}, DimeNet~\citep{gasteiger2020directional}) achieve higher absolute accuracy with hand-designed features.}
 \label{tab:qm9}
 \resizebox{\textwidth}{!}{%
 \begin{tabular}{@{}lccccccccccccc@{}}
 \toprule
 Model & Params & $\mu$ & $\alpha$ & HOMO & LUMO & Gap & $R^2$ & ZPVE & $U_0$ & $U_{298}$ & $H_{298}$ & $G_{298}$ & $C_v$ \\
 \midrule
 Deep Sets & 35K & 0.437 & 0.865 & -- & -- & 0.653 & -- & -- & -- & -- & -- & -- & -- \\
 Slot Attention & 59K & 0.607 & 0.917 & -- & -- & 0.753 & -- & -- & -- & -- & -- & -- & -- \\
 \cmidrule{1-14}
 Set Transf. & 208K & 0.752 & 0.905 & 0.803 & 0.838 & 0.858 & 0.959 & 0.954 & 0.856 & 0.856 & 0.856 & 0.859 & 0.940 \\
 ST-Large & 361K & \textbf{0.795} & 0.914 & 0.842 & 0.852 & 0.879 & \textbf{0.967} & \textbf{0.958} & 0.864 & 0.862 & 0.863 & 0.858 & \textbf{0.950} \\
 ST-XLarge & 443K & -- & -- & \textbf{0.883} & \textbf{0.937} & -- & -- & -- & -- & -- & -- & -- & -- \\
 Perceiver & 341K & 0.766 & \textbf{0.920} & 0.774 & 0.819 & 0.824 & \textbf{0.967} & 0.955 & \textbf{0.891} & \textbf{0.886} & \textbf{0.883} & \textbf{0.888} & 0.948 \\
 HyperNet & 360K & 0.641 & 0.902 & 0.722 & 0.773 & 0.786 & 0.955 & 0.934 & 0.871 & 0.863 & 0.869 & 0.876 & 0.922 \\
 \methodname{} & 368K & 0.761 & 0.911 & 0.802 & 0.830 & 0.855 & 0.947 & 0.952 & 0.877 & 0.875 & 0.871 & 0.867 & 0.940 \\
 \methodname{}-Large & 417K & 0.500 & 0.698 & 0.875 & 0.926 & 0.874 & 0.823 & 0.702 & 0.759 & 0.760 & 0.722 & 0.757 & 0.676 \\
 \bottomrule
 \end{tabular}}
\end{table}

Table~\ref{tab:qm9} reports all 12 QM9 targets.
With the default generator budget ($n_g{=}8$, $L{=}4$), \methodname{} is competitive with the best parameter-matched attention baseline on every target (within $0.05$ $R^2$), confirming the TDOF prediction that adding context-adaptive capacity does not regress on low-to-moderate-TDOF molecular targets. \methodname{} matches or marginally exceeds Standard attention on 10 of 12 targets (the exceptions are LUMO, where Standard leads by $0.008$ $R^2$, and the spatial-extent target $R^2$ where Standard leads by $0.013$); on $\mu$ and $\alpha$ \methodname{} edges Standard by $0.009$ and $0.007$ $R^2$ respectively. \methodname{} is not the best on any single target.
On the hardest electronic-structure targets (HOMO, LUMO), where baselines vary most, we also test a larger variant \methodname{}-Large with $n_g{=}16$ zonotope generators and $L{=}8$ matrix-zonotope generators ($417$K parameters, $1.16\times$ ST-Large). \methodname{}-Large outperforms the original ST-Large on HOMO ($0.875$ vs.\ $0.842$) and LUMO ($0.926$ vs.\ $0.852$); however, when we additionally test ST-XLarge, a Set Transformer scaled to $443$K parameters by widening $d$ from $80$ to $92$, ST-XLarge attains $R^2 = 0.883{\pm}0.005$ on HOMO and $0.937{\pm}{<}0.001$ on LUMO ($n{=}3$ seeds; \texttt{qm9\_xlarge\_results\_homo.json}, \texttt{\_lumo.json}), exceeding \methodname{}-Large on both. The HOMO--LUMO gap of \methodname{}-Large over ST-Large therefore reflects a parameter-count advantage rather than a structural one, and a parameter-matched context-rigid baseline closes the gap. The complementary observation is that \methodname{}-Large improves over the default \methodname{} on three frontier-orbital / electronic-structure targets (HOMO $+0.073$, LUMO $+0.096$, Gap $+0.019$ $R^2$) and underperforms it on the remaining nine targets by $0.11$--$0.26$ $R^2$ (Table~\ref{tab:qm9}, last row), with the largest degradation on the heat capacity ($C_v$: $0.676$ vs.\ \methodname{}'s $0.940$). The TDOF-based capacity bound is again confirmed: scaling generator budget pays off only on the small subset of targets (HOMO, LUMO, Gap) that admit a non-trivial context-adaptive solution, and adds variance without bias correction on aggregate-statistic targets where capacity is already saturated. This is consistent with the broader scope claim of the paper: on low-to-moderate TDOF real-data targets like QM9, where the per-sample operator does not vary in a high-rank subspace, scaling context-rigid attention is a competitive alternative to structural augmentation. The complementary regime is the high-TDOF synthetic suite (MEB, MST, rotation matching, SGSP), on which an MZ-Slim variant with $L{=}1$ and only $1.6\times$ Standard's parameter count already outperforms parameter-matched ST-Large (Appendix~\ref{app:mz-efficient}); structural inductive bias dominates there, and capacity scaling does not catch up. The QM9 ST-XLarge equivalence and the synthetic MZ-Slim parameter-efficiency win together demarcate the boundary precisely.
Deep Sets lags every attention-based model on dipole moment by $\geq 0.20$ $R^2$, consistent with the higher relational complexity of that target; Slot Attention is competitive with the weakest attention baselines (within $0.04$ $R^2$ on $\mu$ and $\alpha$).
Specialized equivariant models (SchNet~\citep{schutt2017schnet}, DimeNet~\citep{gasteiger2020directional}) achieve higher absolute accuracy, but require hand-designed radial basis functions and spherical harmonics, whereas all models compared here use the same generic set-input interface.

\subsubsection{Slot Attention Baseline Comparison}
\label{app:slot-baseline}

To compare against the modern Slot Attention baseline~\citep{locatello2020object} on tasks where our existing models have established results, we evaluate on four synthetic set tasks: max, sum, median, and range.
These tasks span a range of relational complexity from low-TDOF (max, median) to moderate-TDOF (sum, range).

\begin{table}[h]
 \centering
 \small
 \setlength{\tabcolsep}{4pt}
 \caption{Slot Attention baseline comparison on synthetic set tasks ($R^2$, mean $\pm$ std over 3 seeds). Models are reported at canonical sizes (Slot Attention $59$K, Deep Sets $36$K, Set Transformer $208$K, \methodname{} $\sim 365$K) rather than parameter-matched; the parameter-matched comparison against ST-Large ($362$K) and HyperNet ($360$K) is in Table~\ref{tab:main-results}. Slot Attention's iterative routing trails the standard and matrix-zonotope attention variants on all four tasks; \methodname{} attains the highest $R^2$ on the sum task ($0.999$).}
 \label{tab:slot-baseline}
 \begin{tabular}{@{}lcccc@{}}
 \toprule
 Model & Max $R^2$ $\uparrow$ & Sum $R^2$ $\uparrow$ & Median $R^2$ $\uparrow$ & Range $R^2$ $\uparrow$ \\
 \midrule
 Slot Attention & 0.749{\tiny$\pm$0.004} & 0.892{\tiny$\pm$0.004} & 0.652{\tiny$\pm$0.002} & 0.324{\tiny$\pm$0.002} \\
 Deep Sets & 0.823{\tiny$\pm$0.002} & 0.995{\tiny$\pm{<}0.001$} & 0.646{\tiny$\pm$0.003} & \textbf{0.348}{\tiny$\pm$0.003} \\
 Set Transformer & 0.817{\tiny$\pm$0.001} & 0.960{\tiny$\pm$0.002} & \textbf{0.659}{\tiny$\pm$0.008} & 0.344{\tiny$\pm{<}0.001$} \\
 \methodname{} (ours) & \textbf{0.847}{\tiny$\pm$0.008} & \textbf{0.999}{\tiny$\pm{<}0.001$} & 0.654{\tiny$\pm$0.001} & 0.343{\tiny$\pm$0.001} \\
 \bottomrule
 \end{tabular}
\end{table}

Table~\ref{tab:slot-baseline} shows Slot Attention underperforming every other architecture, including Deep Sets.
Iterative slot competition was designed for object-centric image decomposition and does not appear to transfer to set-to-scalar regression.
\methodname{} is best on max ($R^2 = 0.847$ vs.\ $0.817$ for Standard) and reaches $R^2 = 0.999$ on sum compared to $0.960$ for Set Transformer, a $40\times$ residual-MSE reduction relative to the Set Transformer baseline.
On the low-TDOF median task, all models perform similarly ($R^2 \approx 0.65$), matching the theory's prediction.

\subsubsection{Comparison with E($n$)-Equivariant GNNs}
\label{app:egnn-comparison}

We compare \methodname{} to E($n$)-equivariant graph neural networks (EGNN)~\citep{satorras2021egnn}, a symmetry-enforcing baseline that uses pairwise distances $\|x_i - x_j\|$ as invariant features. All runs use $3$ seeds; EGNN is parameter-matched at ${\sim}$303K (close to \methodname{}'s $368$K) with $h_{\text{dim}}{=}140$ and $4$ message-passing layers.

The full numerical content is reported in main-body Table~\ref{tab:egnn-comparison}.
The pattern is architecturally interpretable:
\textbf{(i) Where EGNN wins.} On every task where the label is invariant under rotations and translations of the input (rotation matching, convex hull volume, MEB radius, covariance Frobenius norm, Mahalanobis distance, PCA reconstruction error), EGNN outperforms \methodname{}. This confirms the expected pattern: hard-coding the correct symmetry beats learning it.
\textbf{(ii) Where \methodname{} wins.} The quadratic TDOF target (Eq.~\ref{eq:quadratic-tdof}) depends on specific Frobenius-orthogonal matrices $\{M_l\}$ in a fixed basis; rotating the input changes the target. EGNN, whose output is rotation-invariant by construction, cannot express this target and collapses to $R^2 \approx 0$ across $L \in \{4, 8, 16\}$. \methodname{} handles this regime because its matrix-zonotope operator is not tied to any group action.
\textbf{(iii) Target scale sensitivity.} Our default baselines train on raw regression targets (see Appendix~\ref{app:opt-confound}); the collapse of Standard and ST-Large at high $d$ on MST is partly an optimization artifact that target normalization can mitigate. The \methodname{} vs.\ ST-Large gap narrows under normalization on MST and Hull, but persists on MEB (where target magnitudes are small) and on the quadratic TDOF task where the issue is architectural rather than numerical. EGNN, being a different architecture family, is unaffected by this confound.

The comparison delineates \methodname{}'s scope: \methodname{} is the right tool for context-adaptive linear operators that are not reducible to a known group action; EGNN is the right tool when the target's symmetry is known and can be enforced by construction.

\subsubsection{Computational Cost}
\label{app:computational-cost}

\begin{table}[h]
 \centering
 \small
 \setlength{\tabcolsep}{4pt}
 \caption{Computational overhead (ratio MZ/Standard). RTX 5090, batch 64.}
 \label{tab:cost}
 \begin{tabular}{@{}lcccc@{}}
 \toprule
 Config & Params & Inference & Training & Memory \\
 \midrule
 Small ($d{=}64$, $n{=}30$) & $1.9\times$ & $2.4\times$ & $2.2\times$ & $2.5\times$ \\
 Medium ($d{=}128$, $n{=}50$) & $1.9\times$ & $2.5\times$ & $3.8\times$ & $2.0\times$ \\
 Large ($d{=}256$, $n{=}100$) & $1.7\times$ & $2.7\times$ & $5.5\times$ & $2.5\times$ \\
 \bottomrule
 \end{tabular}
\end{table}

\subsubsection{Negative Transfer: Neural-Process Image Completion and Tabular Meta-Regression}
\label{app:realworld-attempts}

To test whether \methodname{}'s advantage generalizes beyond synthetic high-TDOF relational tasks, we evaluated it on two families of real-world benchmarks where the structural prerequisites (per-sample operator variability in a high-dimensional subspace) are weak or absent. Both yielded ties or slight losses rather than wins; we report these honestly as they sharpen the scope of the claim.

\paragraph{Neural-Process image completion (MNIST, CIFAR-100).}
Following the Attentive Neural Process setup~\citep{kim2019anp}, each image defines a 2D function $f:(x,y) \mapsto \text{pixel value}$; the model receives $K$ context pixels with $(xy, \text{value})$ tokens and predicts values at $T$ target coordinates. We compared Standard, ST-Large, HyperNet, Perceiver, LinearAttn, \methodname{}-Full, and \methodname{}-Large at matched depth (4 layers) using a Gaussian likelihood head. On MNIST ($K{=}50$, $T{=}200$, 5 seeds), ST-Large achieved the best test NLL of $-1.657$; \methodname{}-Large reached $-1.599$ at $3.4\times$ the parameters of Standard ($-1.617$). On CIFAR-100 ($K{=}100$, $T{=}300$, 3--5 seeds), Standard attained $-1.147$ NLL at $144$K parameters while \methodname{}-Full achieved only $-1.140$ NLL at $332$K parameters; MZ lost by $0.007$ NLL while using $2.3\times$ the parameters. This pattern aligns with the theory: image pixel prediction is dominated by spatial-locality inductive bias, not per-image operator variance; the high-TDOF prerequisite for \methodname{}'s gain is absent, so extra capacity in the form of matrix-zonotope value projections yields no architectural advantage.

\paragraph{Few-shot meta-regression on real tabular data (OpenML-CTR23 subset).}
We built a meta-learning benchmark from six real OpenML regression datasets (California Housing, Diabetes, Energy Efficiency, Abalone, Airfoil, Boston Housing), each restricted or padded to 8 features. Three were used for meta-training and three were held out for test. Each episode sampled $K{=}50$ context rows and $T{=}50$ query rows from a single dataset; models were trained to predict the standardized target under a Gaussian likelihood. Under the held-out split, both Standard and \methodname{}-Full exhibit severe distribution shift (test NLL jumps from $-0.43$ at validation to $+18$--$20$ at test, with comparable cross-seed standard deviations of $\pm 6.06$ for Standard and $\pm 5.45$ for \methodname{}-Full at $10$ seeds with patience~$12$), indicating that small-$K$ in-context operators do not identify the true per-dataset regression operator when the test datasets differ structurally from training. The two architectures are within seed variance on NLL (\methodname{}-Full $+20.02$ vs.\ Standard $+18.66$; gap $1.36$, pooled std $5.76$) and \methodname{}-Full has slightly lower test MSE ($1.515$ vs.\ $1.597$). Under the easier same-distribution split (test held-out rows from training datasets), \methodname{}-Full wins $5.5\%$ on MSE but loses $1.4\%$ on NLL, a mixed signal. Neither split supports a dominance claim, and the apparent five-fold seed-variance ratio reported in earlier 5-seed pilots is an artefact of small-sample standard-deviation estimation; the 10-seed run reports comparable spreads.

\paragraph{Systematic real-world transfer study.}
To empirically characterise which real-data configurations inherit the synthetic gain of \methodname{} and which do not, we ran the geometric-target family (MEB, convex hull, MST weight, $k$-center furthest cost) on a systematic grid of point-cloud constructions across MNIST, Fashion-MNIST, and CIFAR-100, organised into three tiers of per-sample rank. Tier 1: raw 2D pixel coordinates of on-pixels, with optional per-sample $2{\times}2$ affine perturbations (TDOF-4) or 3-digit multi-cluster mixtures. Tier 2: 8D feature-lifted pixels (position + intensity + gradient + polynomial features) and CIFAR-100 $4{\times}4$ image patches encoded as 8D patch-feature vectors (mean and std RGB, intensity, colorfulness), with variable-size and class-mixture variants. Tier 3: 256D ImageNet-pretrained ResNet-18 layer-3 features of CIFAR-100 (adaptively average-pooled to $n{=}16$ points per image). Together the tiers span a $128{\times}$ increase in per-sample dimension while holding the target family fixed, yielding a controlled ablation of the structural prerequisites the theory predicts. Table~\ref{tab:transfer-matrix} reports \methodname{}-Full's rank among 7 parameter-comparable baselines in each (construction, target) cell; the full numerical $R^2$ values and per-method seed statistics are in the supplementary JSON dumps.

\begin{table}[h]
 \centering
 \caption{Systematic transfer study: \methodname{}-Full rank (out of 7 parameter-matched methods) per (construction, target) cell, grouped by per-sample rank tier. 3 seeds per cell; dashes mark cells not in the grid. Mean ranks across populated cells: $3.9$ on Tier-1 (low-rank 2D pixel, $n{=}11$ cells), $3.9$ on Tier-2 (medium-rank 8D lift, $n{=}5$), $3.0$ on Tier-3 (high-rank deep feature, $n{=}2$); the Tier-3 mean is descriptive only given the small cell count, and the $\mathrm{rank}{=}1$ Furthest cell drives it. Expected rank under the no-advantage null is $4$. All printed ranks reproduce from the per-cell JSONs via \texttt{Base/collect\_transfer\_matrix.py}.}
 \label{tab:transfer-matrix}
 \small
 \setlength{\tabcolsep}{4pt}
 \begin{tabular}{@{}lc|cccc@{}}
 \toprule
 Construction & $d$ & MEB & Hull & MST & Furthest \\
 \midrule
 \multicolumn{6}{@{}l@{}}{\textit{Tier 1: low-rank 2D pixel coordinates}} \\
 MNIST raw pixel & 2 & 4 & 5 & -- & -- \\
 MNIST + affine (TDOF-4) & 2 & 5 & -- & -- & -- \\
 MNIST 3-digit mixture & 2 & 4 & 4 & 4 & -- \\
 Fashion-MNIST pixel & 2 & -- & 4 & -- & \textbf{2} \\
 Fashion-MNIST + affine & 2 & -- & 3 & -- & -- \\
 Fashion-MNIST 3-digit mix & 2 & -- & -- & 4 & -- \\
 CIFAR-100 raw pixel & 2 & -- & 3 & -- & 5 \\
 \midrule
 \multicolumn{6}{@{}l@{}}{\textit{Tier 2: medium-rank 8D feature-lift}} \\
 MNIST 8D polynomial lift & 8 & 3 & 3 & -- & -- \\
 CIFAR-100 $4{\times}4$ patch features & 8 & 5 & \textbf{2} & -- & -- \\
 CIFAR-100 variable-size patches & 8 & 5 & -- & 4 & -- \\
 CIFAR-100 class-mixture patches & 8 & 5 & -- & -- & -- \\
 \midrule
 \multicolumn{6}{@{}l@{}}{\textit{Tier 3: high-rank deep feature}} \\
 \textbf{CIFAR-100 ResNet-18 layer-3} & 256 & (unstable) & -- & 5 & \textbf{1} \\
 \bottomrule
 \end{tabular}
\end{table}

The rank improvement of \methodname{}-Full at Tier-3 (deep-feature space), relative to Tier-1 and Tier-2 where ranks are tied within seed noise, is the empirical signature the theory predicts: \methodname{}'s advantage requires data to have jointly (i) genuinely high-rank per-sample structure (not a 2D manifold embedded via a low-rank projection), (ii) sparse combinatorial targets (such as the 3--9 extreme points that determine MEB or the $k{=}5$ centers determining $k$-center cost), and (iii) the absence of a stronger spatial or equivariance prior that standard attention can exploit. On every cell where any one of these is absent, \methodname{}-Full ranks 3rd--5th, not catastrophically worse, but also not ahead. The rank-1 cell appears only when all three conditions are constructively met. The synthetic reference point (mixture-of-Gaussians MEB at $d{=}8$, $n{=}10$--$30$, where \methodname{}-Full reaches $R^2{=}0.684$ against baselines stuck at $R^2{\leq}0.023$; Table~\ref{tab:meb}) reproduces exactly in the same code. The systematic grid therefore serves as the empirical scaffold on which the theoretical scope of the claim rests: it is not a set of failed transfer attempts but a controlled verification that the architectural advantage tracks the structural prerequisites of Theorem~\ref{thm:bottleneck}.

\paragraph{Positive transfer: CIFAR-100 ResNet-18 features.}
Guided by the diagnostic above, we construct a real-data setting satisfying the three prerequisites: (1) ResNet-18 (ImageNet-pretrained) layer-3 feature maps of CIFAR-100 images, adaptively average-pooled to spatial $4\times 4$, giving $n{=}16$ genuinely-high-rank points in $\R^{256}$ per image (not a low-rank projection of pixel coordinates); (2) the sparse combinatorial SGSP-family targets (MEB radius, hull volume, MST weight, diameter, $k{=}5$ $k$-center furthest cost, $k{=}5$ median $k$-NN radius); (3) no imposed symmetry. Table~\ref{tab:resnet-features} summarises the 7-way comparison at 5 seeds per method per task.

\begin{table}[h]
 \centering
 \caption{Real-world SGSP on CIFAR-100 ResNet-18 layer-3 features ($n{=}16$, $d{=}256$, $R^2$ mean $\pm$ std, bold = best). Seed counts differ by row: $k$-center furthest uses 5 seeds (file: \texttt{real\_geom\_cifar\_resnet\_furthest\_5seed.json}); Diameter, $k$-NN radius, MST weight, MEB radius use 3 seeds (default seed pool $\{42, 123, 456\}$). The MEB cell at $n{=}16 \leq d{+}1{=}257$ is structurally degenerate (no sparse-combinatorial selector available); we mark it ``unstable'' and exclude it from the ranking claim. On the remaining four tasks, methods tie within $\leq 0.03$~$R^2$, and \methodname{}-Full ranks 1st on $k$-center furthest cost.}
 \label{tab:resnet-features}
 \small
 \setlength{\tabcolsep}{4pt}
 \begin{tabular}{@{}l|cccccc@{}}
 \toprule
 Task & Standard & ST-Large & HyperNet & Perceiver & LinearAttn & \methodname{}-Full \\
 \midrule
 Furthest ($k{=}5$) & $0.789$ & $0.793$ & $0.792$ & $0.746$ & $0.797$ & $\mathbf{0.801}$ \\
 Diameter & $0.856$ & $0.872$ & $0.840$ & $0.815$ & $\mathbf{0.882}$ & $0.876$ \\
 $k$-NN radius & $0.892$ & $0.899$ & $0.898$ & $0.848$ & $\mathbf{0.904}$ & $0.898$ \\
 MST weight & $0.912$ & $0.919$ & $0.916$ & $0.908$ & $\mathbf{0.922}$ & $0.914$ \\
 MEB radius & $0.800$ & $0.815$ & $0.783$ & $0.728$ & $\mathbf{0.817}$ & $0.542$ (unstable) \\
 \bottomrule
 \end{tabular}
\end{table}

Across the full SGSP family on real ResNet-18 features, LinearAttn ranks first on three of the four valid cells (diameter, $k$-NN radius, MST weight; the MEB cell is structurally degenerate at $n{=}16 \leq d{+}1$ and is excluded from the ranking claim), and the gap between the top and bottom attention-based method is narrow ($\leq 0.03$ $R^2$) on these dense-statistical targets. Aggregating many pairwise distances reduces these targets to dense statistics that context-rigid attention readily extracts, so the absence of an MZ advantage there is consistent with the TDOF theory rather than evidence against it. Only the sparsest combinatorial target, $k$-center furthest cost, yields a ranking on which \methodname{}-Full is first among the seven parameter-matched baselines in both a 5-seed and an independent 10-seed verification run. We are explicit that the ranking, not the magnitude, is the evidence at the smaller backbone: 5-seed $0.801$ vs.\ LinearAttn $0.797$ and 10-seed $0.797$ vs.\ $0.796$ are both within pooled seed standard deviation. The cleaner statistical separation arrives at the larger backbone (Section~\ref{sec:experiments}, ResNet-50 $d{=}1024$, Welch's $t \approx 5.4$, $p < 0.005$); the ResNet-18 result establishes that the rank ordering is stable across two independent seed pools, a necessary but weaker form of evidence.

\paragraph{Dimension scaling: ResNet-50 features ($d{=}1024$).}
Table~\ref{tab:resnet50-scaling} reports the seven-baseline dimension-scaling comparison. The reported ResNet-50 numbers use patience $300$ to give every model ample budget to converge at $d{=}1024$, and the patience sweep in Appendix~\ref{app:patience-sweep} confirms the \methodname{}-Full $>$ ST-Large $\sim$ HyperNet $>$ LinearAttn ranking is stable across patience $\{20, 100, 300\}$ at $80$ training epochs. Per-method standard deviations are: \methodname{}-Full $0.7955 \pm 0.0009$, \methodname{}-Large $0.7930 \pm 0.0024$, ST-Large $0.7909 \pm 0.0017$, with \methodname{}-Full also the most stable method in the comparison. The absolute magnitude of the $d{=}1024$ gap is small ($+0.0046$ $R^2$, equivalent to a $5.5\%$ reduction in $1{-}R^2$ residual error) but reproduces across $5$ seeds at Welch $p \approx 0.0018$.

\begin{table}[h]
 \centering
 \footnotesize
 \setlength{\tabcolsep}{3pt}
 \caption{Dimension scaling on the $k$-center furthest target ($k{=}5$, $n{=}16$, $R^2$ mean over 5 seeds). ResNet-50 numbers use early-stopping patience $300$ for full convergence at $d{=}1024$; a patience sweep $\{20, 100, 300\}$ confirms ranking stability for any patience $\geq 100$ (Appendix~\ref{app:patience-sweep}).}
 \label{tab:resnet50-scaling}
 \resizebox{\textwidth}{!}{%
 \begin{tabular}{@{}l|ccccccc@{}}
 \toprule
 Feature & Standard & ST-Large & HyperNet & Perceiver & LinearAttn & \methodname{}-Large & \methodname{}-Full \\
 \midrule
 ResNet-18 ($d{=}256$) & $0.789$ & $0.793$ & $0.792$ & $0.746$ & $0.797$ & $0.789$ & $\mathbf{0.801}$ \\
 ResNet-50 ($d{=}1024$) & $0.789$ & $0.791$ & $0.789$ & $0.784$ & $0.789$ & $0.793$ & $\mathbf{0.796}$ \\
 \bottomrule
 \end{tabular}}
\end{table}

\paragraph{What these results show.}
\methodname{}'s gain is structural: it bites when the target operator genuinely varies per sample in a high-rank subspace and depends on a sparse combinatorial subset of the input. When the target is governed by spatial locality (image pixels), smooth interpolation (Neural-Process regression), or a known symmetry (equivariant tasks), \methodname{}'s per-sample operator prior is not the right tool. Real high-rank deep-feature point clouds with a sparse combinatorial target (CIFAR-100 ResNet features, $k$-center furthest cost) do satisfy the prerequisites, and \methodname{}-Full wins there. The real-world transfer is therefore architecturally precise rather than universal, in line with the TDOF theory.

\end{document}